\documentclass[letterpaper]{article} 
\usepackage{submission} 
\usepackage[hyphens]{url}  
\usepackage{graphicx} 
\usepackage{natbib}  
\usepackage{caption} 
\usepackage{algorithm}

\usepackage{newfloat}
\usepackage{listings}
\DeclareCaptionStyle{ruled}{labelfont=normalfont,labelsep=colon,strut=off} 
\floatstyle{ruled}
\newfloat{listing}{tb}{lst}{}
\floatname{listing}{Listing}

\usepackage{booktabs}

\usepackage{subcaption}
\usepackage{multirow}

\usepackage{array}
\usepackage{algpseudocode}
\usepackage{amsmath,amssymb,amsthm}
\usepackage{pifont}
\usepackage{xspace}
\usepackage{tikz}
\usepackage{makecell}
\usepackage{enumitem}
\usepackage{etoc}

\newtheorem{example}{Example}
\newtheorem{proposition}{Proposition}
\newtheorem{lemma}{Lemma}
\newtheorem{theorem}{Theorem}
\newtheorem{definition}{Definition}
\newtheorem{remark}{Remark}
\newtheorem{corollary}{Corollary}

\newcommand{\cmark}{\ding{51}}
\newcommand{\xmark}{\ding{55}}

\newcommand*{\eg}{e.g.,\xspace}
\newcommand*{\ie}{i.e.,\xspace}
\newcommand*{\RelShap}{\texttt{RelShap}\xspace}

\title{RelShap: Relationally Consistent Shapley Explanations}
\author {
    Seungeun Lee\textsuperscript{\rm 1},
    Joao Fonseca\textsuperscript{\rm 1, \rm 2},
    Julia Stoyanovich\textsuperscript{\rm 1}\corresponding
}
\affiliations {
    \textsuperscript{\rm 1} New York University, New York, USA \\
    \textsuperscript{\rm 2} INESC-ID, Lisbon, Portugal
}

\begin{document}

\maketitle

\begin{abstract}

Machine learning pipelines commonly flatten relational data into single-table representations, discarding structural constraints. Widely used Shapley value-based feature attributions then rely on feature independence, evaluating the model on combinations that could never arise in the underlying data, producing misleading explanations. We propose \RelShap, a framework that incorporates relational constraints and data provenance into Shapley value computation, restricting both background data and coalition evaluation to relationally valid configurations. The framework is estimator-agnostic and composes with Kernel SHAP, Monte Carlo, and Leverage SHAP without altering their sampling or weighting properties. Functional dependencies further induce equivalence classes over feature coalitions, which \RelShap exploits to reduce runtime without changing Shapley values; we provide a combinatorial characterization of the expected speedup. Experiments across multiple datasets, models, and estimators show that \RelShap produces explanations that are more faithful to the data-generating process, correctly identifying the dominant feature in controlled settings where existing methods, including Conditional SHAP and ManifoldShap, do not. Our code is available at: \url{https://github.com/duneag2/relshap}.

\end{abstract}

\section{Introduction}
\label{sec:intro}

Most machine learning (ML) pipelines flatten relational data into single-table representations for model training~\cite{pmlr-v235-fey24a}. Once a predictive model is trained on such a flattened table, practitioners routinely seek to understand its predictions at the level of individual instances. Shapley value-based attributions~\cite{shapley1953value,dattaSZ16,lundbergL2017unified} have become a dominant approach for this purpose, assigning each input feature
a score based on its contribution to the prediction. 
Notably, relational ML is a well-established setting~\cite{robinson2024relbench,dwivedi2025relational}, yet approaches that train directly on multi-table data must ultimately attribute predictions to individual features and thus face the same explanation-time flattening.

In practice, this computation requires two conceptual choices: \emph{background data selection}, which determines how feature absence is simulated using reference data, and \emph{coalition selection}, which specifies which feature subsets are evaluated. Existing methods instantiate different sides of the tension between being \emph{true to the model} and \emph{true to the data}~\cite{chen2020true,fryeshapley,janzing2020feature}. SHAP~\cite{lundbergL2017unified}, the most widely used implementation, stays \emph{true to the model} by assuming feature independence but potentially considers feature combinations that never arise in the data. \citet{aas2021explaining} instead remain \emph{true to the data} by estimating conditional distribution, but only as true as the estimate itself, and the estimator's errors propagate directly into the attributions.

Both premises are problematic on data drawn from a relational database. The flattening step can discard database-level information, such as functional dependencies (FDs) that imply that some features are fully determined by others, and domain constraints that restrict admissible values. Feature independence is violated by these very constraints, and conditional distribution estimation, while potentially avoiding that violation, introduces its own modeling assumptions. When attribution methods ignore relational constraints, they may evaluate the model on inputs that could never arise from the underlying relational data, thereby fundamentally altering the resulting feature attributions.

\paragraph{Core contribution:} We present \RelShap---\emph{Relationally Consistent Shapley Explanations}---the first framework to integrate relational database constraints into Shapley value-based feature attributions. \RelShap changes the admissible space of the Shapley explanation itself: instead of relying on feature independence assumptions or additional distributional modeling, it restricts both background data and coalitions to relationally valid configurations through constraints derived from the database schema, query, and data. 
Importantly, in controlled settings with known ground truth, \RelShap correctly identifies the dominant feature where Kernel SHAP, Conditional SHAP, and ManifoldShap do not (Section~\ref{sec:exp:intuition}). 
For instance, in a loan approval scenario (Example~\ref{ex:motivating}), the explanation changes substantially depending on whether relational constraints are respected (Figure~\ref{fig:kernelshap_relshap}).  The divergence is not a sampling artifact but a consequence of which feature combinations the explanation method permits.

\begin{example}\label{ex:motivating}
Consider two tables,
\texttt{Applicants}(\underline{\texttt{a\_id}}, \texttt{age}, \texttt{life\_stage}, \texttt{empl})
and
\texttt{Transactions}(\underline{\texttt{t\_id}}, \texttt{a\_id}, \texttt{amount}),
linked by applicant identifier \texttt{a\_id}. Six representative
applicants and their transactions are shown in
Tables~\ref{tab:running_applicants} and~\ref{tab:running_transactions}. A data analyst issues the query:

\begin{lstlisting}[language=SQL,basicstyle=\fontsize{8.7pt}{9.7pt}\selectfont\ttfamily,xleftmargin=0.5em,numbers=none]
SELECT a.a_id, a.age, a.life_stage, a.empl,
       SUM(t.amount) AS total_amt
FROM   Applicants a, Transactions t
WHERE  a.a_id = t.a_id
GROUP BY a.a_id, a.age, a.life_stage, a.empl;
\end{lstlisting}

\noindent
to obtain an applicant-level view, then drops the identifier \texttt{a\_id} and adds a target \texttt{loan\_approved} to yield the ML-ready dataset in Table~\ref{tab:running_flattened}.

\end{example}

ML practitioners computing explanations typically receive only this final flattened representation, with no visibility into the upstream structure.
Applicant \texttt{a27} applied for a loan, was rejected, and sought an explanation. Kernel SHAP~\cite{lundbergL2017unified}, a widely used method, reports $\texttt{life\_stage}$ as the top driver of the prediction with a positive score (Figure~\ref{fig:kernelshap_example}). This attribution, however, is an artifact of relationally invalid completions from Kernel SHAP. For instance, fixing $\texttt{age}=35$ and filling in the rest from applicant \texttt{a29} yields the combination $(\texttt{age}=35, \texttt{life\_stage}=\texttt{older})$, which is impossible under the FD $\texttt{age}\rightarrow\texttt{life\_stage}$ (first row of Table~\ref{tab:running_completions}).

\RelShap instead operates over a space of relationally valid completions, automatically extracting constraints from the data's relational structure. In this example, the data reveals a \emph{functional dependency} (FD) $\texttt{age}\rightarrow\texttt{life\_stage}$: each age value maps to exactly one life stage. Similarly, because the query aggregates all of an applicant's transactions into a single row, each applicant identifier determines exactly one \texttt{total\_amt}. \RelShap enforces these constraints during Shapley computation: fixing $\texttt{age}=35$ forces $\texttt{life\_stage}=\texttt{middle}$ (second row of Table~\ref{tab:running_completions}); the remaining attributes are
unconstrained and filled from a chosen background data point (\texttt{a29}). Additionally, \RelShap supports a \emph{provenance-aware} mode that traces each row of the flattened table back to its source tuples in the original tables. Fixing $\texttt{age}=35$ narrows the provenance to \texttt{a27}, the sole applicant with that age, at which point the schema determines $\texttt{empl}=\texttt{self\_emp}$ and $\texttt{total\_amt}=119$ (third row of Table~\ref{tab:running_completions}). Provenance-based recovery is not always available when it does not resolve to a single tuple: $\texttt{age}=63$, for instance, is shared by \texttt{a12} and \texttt{a29}. See Appendix~\ref{app:relational_primer} for a primer on relational concepts and a detailed description of how \RelShap extracts constraints from the database schema, query structure, and data.

Notably, \emph{\RelShap computes different Shapley values even without approximation.} Under both \RelShap variants, \texttt{life\_stage} receives \emph{exactly zero} attribution (which does not necessarily generalize), down from $+0.2299$ under Kernel SHAP (Figure~\ref{fig:kernelshap_relshap}), because the discovered FD $\texttt{age}\rightarrow\texttt{life\_stage}$ makes \texttt{life\_stage} redundant when \texttt{age} is present in a coalition, once coalitions are projected to relationally valid completions (Section~\ref{sec:relshap}). Kernel SHAP ignores this redundancy and distributes substantial attribution to \texttt{life\_stage}. The top-ranked feature consequently changes from \texttt{life\_stage} to \texttt{age}. In all three cases, feature attribution values are computed \emph{exactly}, over all background points and all coalitions; the divergence is not a sampling artifact but a consequence of which feature combinations the explanation method permits.

\noindent\textbf{Our contributions are as follows:}

\begin{table}[t]
\centering
\caption{Running example (excerpt)}
\label{tab:running_example}

\begin{subtable}[t]{0.48\columnwidth}
\centering
\caption{\texttt{Applicants} (excerpt)}
\small
\setlength{\tabcolsep}{1pt}
\begin{tabular}{cccc}
\toprule
a\_id & age & life\_stage & empl \\
\midrule
a12 & 63 & older & unemp \\
a14 & 26 & young  & unemp \\
a24 & 54 & middle & emp \\
a25 & 22 & young  & self\_emp \\
\textbf{a27} & \textbf{35} & \textbf{middle} & \textbf{self\_emp} \\
a29 & 63 & older & unemp \\
\bottomrule
\end{tabular}
\label{tab:running_applicants}
\end{subtable}
\hfill
\begin{subtable}[t]{0.48\columnwidth}
\centering
\caption{\texttt{Transactions} (excerpt)}
\small
\setlength{\tabcolsep}{4pt}
\begin{tabular}{ccc}
\toprule
t\_id & a\_id & amount \\
\midrule
t62 & a25 & 47 \\
\textbf{t67} & \textbf{a27} & \textbf{43} \\
\textbf{t68} & \textbf{a27} & \textbf{59} \\
\textbf{t69} & \textbf{a27} & \textbf{17} \\
t73 & a29 & 50 \\
\multicolumn{3}{c}{$\vdots$} \\
\bottomrule
\end{tabular}
\label{tab:running_transactions}
\end{subtable}

\vspace{0.1em}

\begin{subtable}[t]{\columnwidth}
\centering
\caption{Flattened, ML-ready dataset (excerpt): \texttt{a27} is highlighted.}
\small
\setlength{\tabcolsep}{4pt}
\begin{tabular}{ccccc}
\toprule
age & life\_stage & empl & total\_amt & loan\_approved \\
\midrule
63 & older & unemp        & 264 & 1 \\
26 & young  & unemp        &  94 & 0 \\
54 & middle & emp          & 235 & 1 \\
22 & young  & self\_emp    & 100 & 0 \\
\textbf{35} & \textbf{middle} & \textbf{self\_emp} & \textbf{119} & \textbf{0} \\
63 & older & unemp        & 114 & 0 \\
\bottomrule
\end{tabular}
\label{tab:running_flattened}
\end{subtable}

\vspace{0.1em}

\begin{subtable}[t]{\columnwidth}
\centering
\caption{Completions for $\texttt{age}=35$ under increasing relational awareness. The first row violates the FD $\texttt{age}\rightarrow\texttt{life\_stage}$.}
\small
\setlength{\tabcolsep}{1.3pt}
\begin{tabular}{lcccc}
\toprule
 & age & life\_stage & empl & total\_amt \\
\midrule
Kernel SHAP (unaware)        & 35 & older  & unemp      & 114 \\
\RelShap (FDs only)          & 35 & middle & unemp      & 114 \\
\RelShap (FDs + provenance)  & 35 & middle & self\_emp  & 119 \\
\bottomrule
\end{tabular}
\label{tab:running_completions}
\end{subtable}

\end{table}

\begin{figure*}[t]
  \centering
  \begin{subfigure}[t]{0.32\textwidth}
    \centering
    \includegraphics[width=\textwidth]{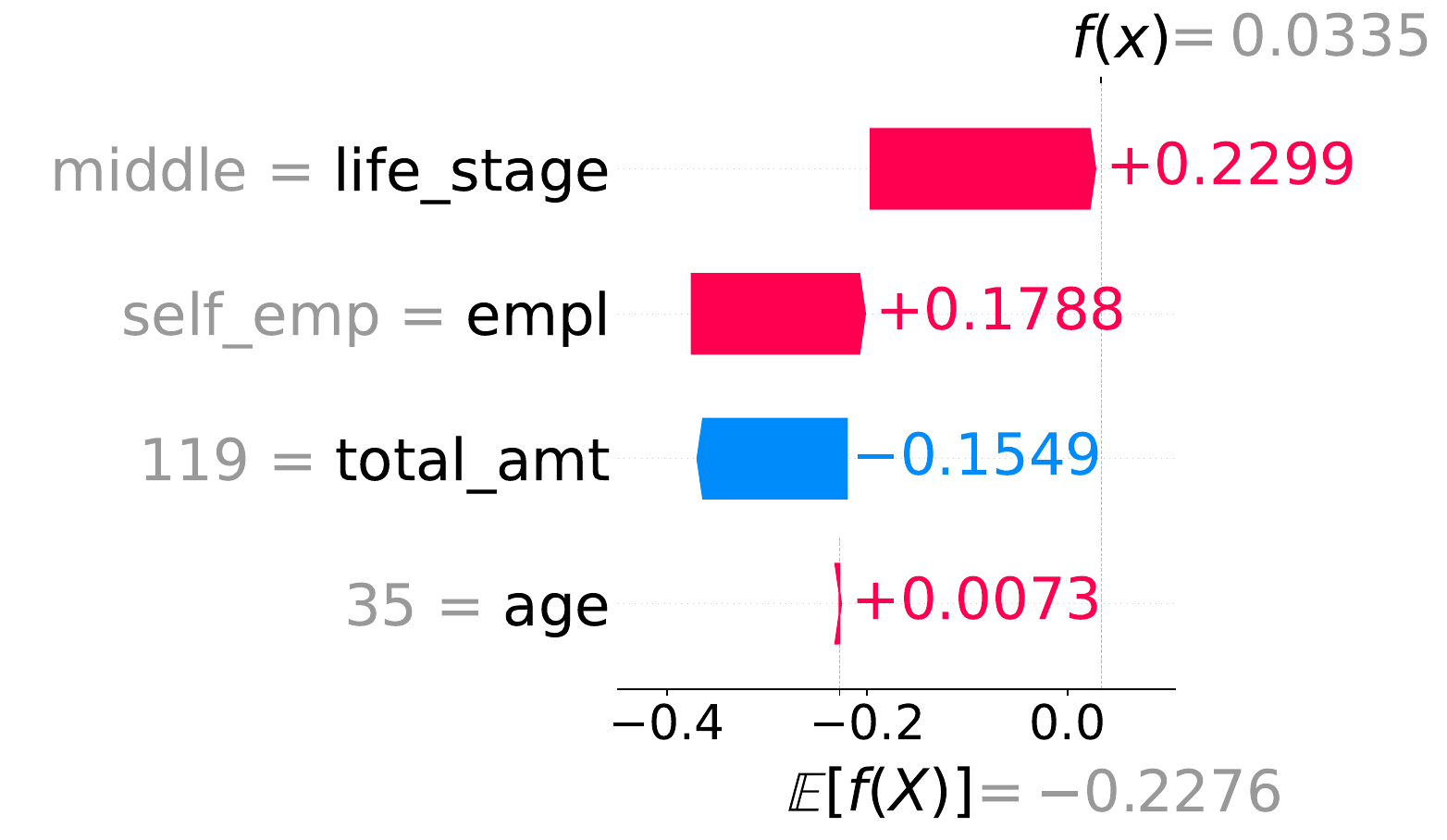}
    \caption{Kernel SHAP.}
    \label{fig:kernelshap_example}
  \end{subfigure}
  \hfill
  \begin{subfigure}[t]{0.32\textwidth}
    \centering
    \includegraphics[width=\textwidth]{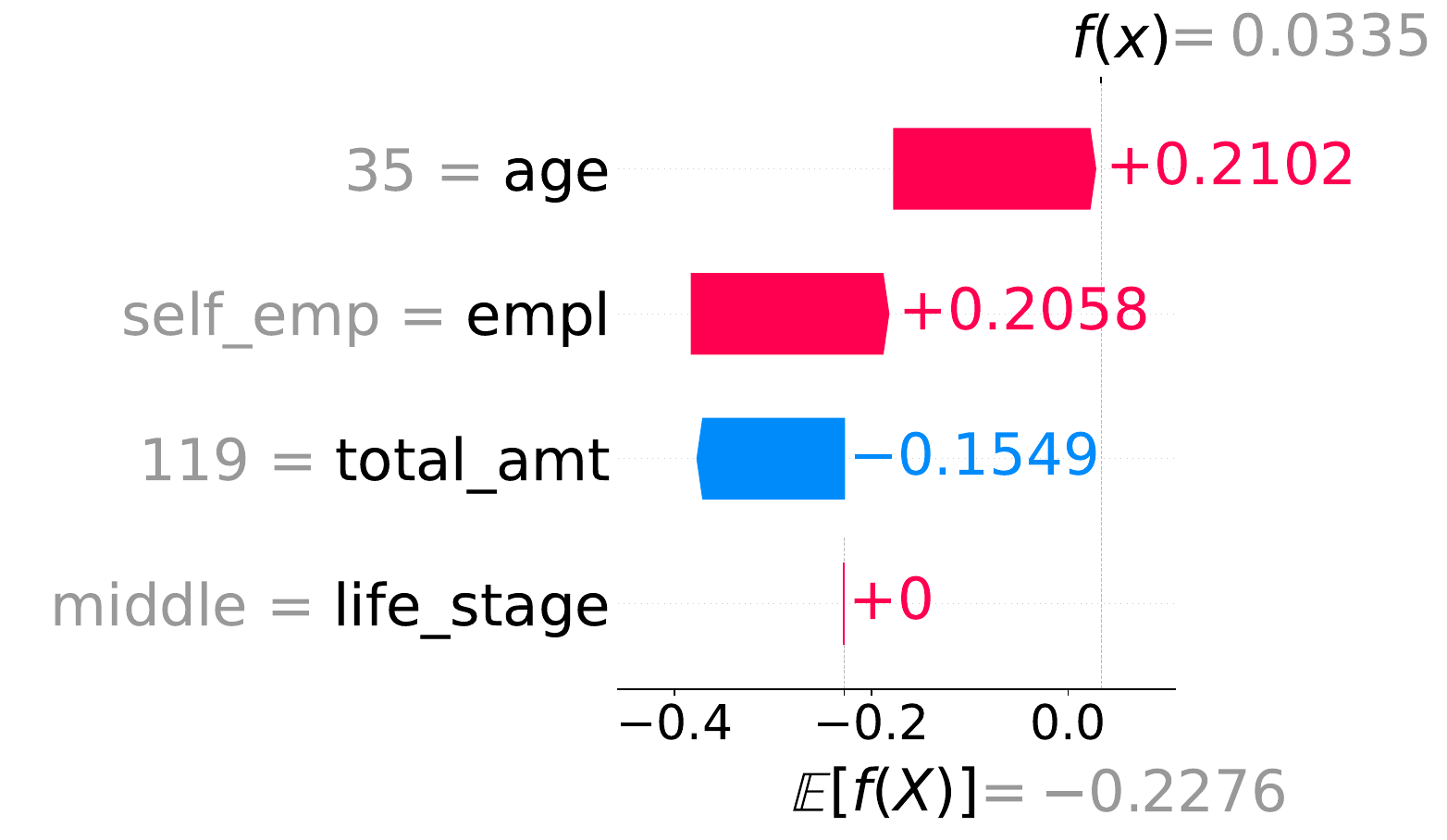}
    \caption{\RelShap (FDs only).}
    \label{fig:relshap_example}
  \end{subfigure}
  \hfill
  \begin{subfigure}[t]{0.32\textwidth}
    \centering
    \includegraphics[width=\textwidth]{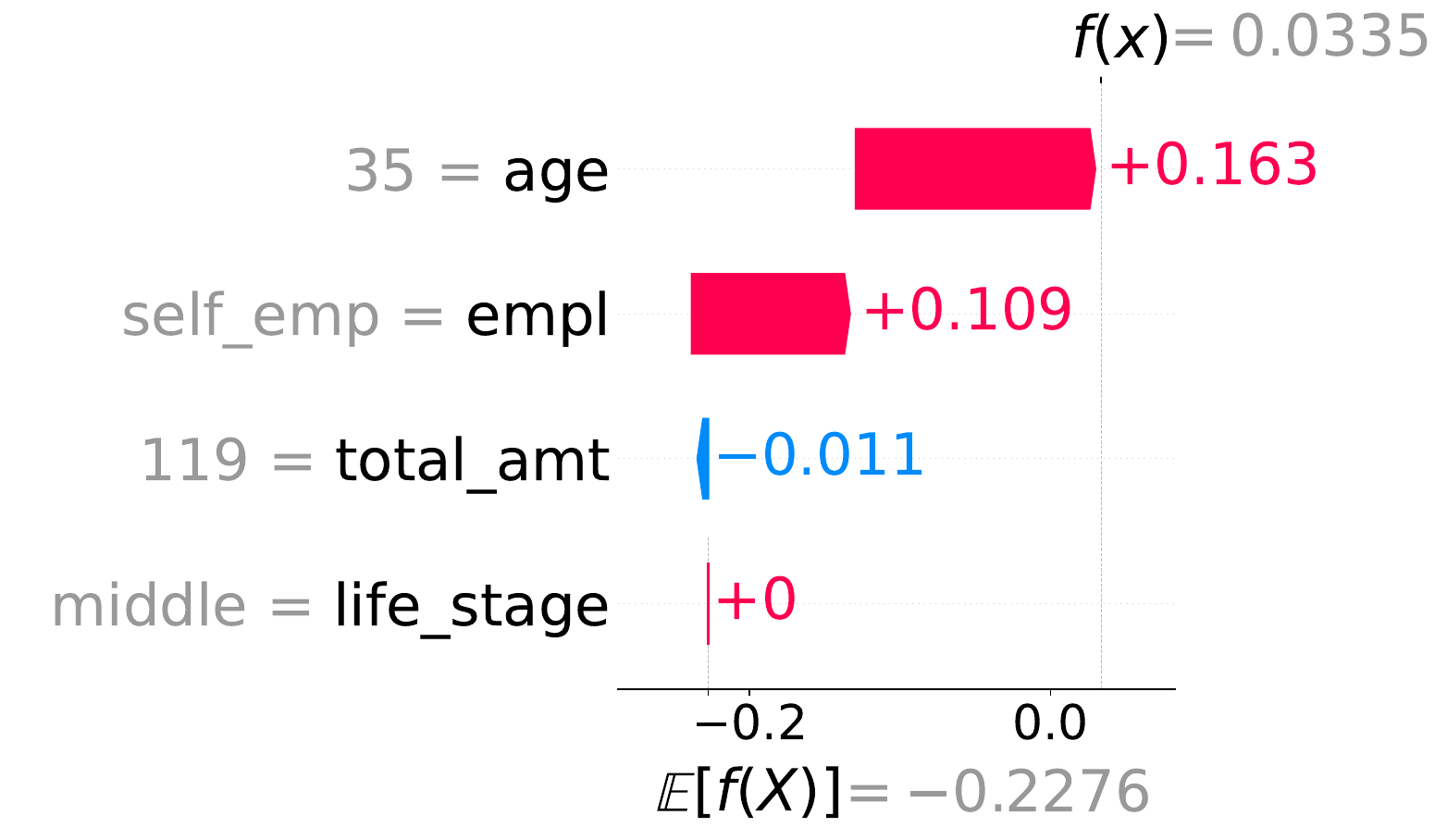}
    \caption{\RelShap (FDs and provenance).}
    \label{fig:relshap_prov_example}
  \end{subfigure}
  \caption{Kernel SHAP vs. \RelShap feature attributions for applicant \texttt{a27} in Example~\ref{ex:motivating} under an XGBoost classifier.}
  \label{fig:kernelshap_relshap}
\end{figure*}

\begin{enumerate}
    \item \RelShap alters both background data and coalition selection under relational constraints, producing attributions that correctly reflect the data-generating structure (Section~\ref{sec:exp:intuition}).
    
    \item \RelShap is agnostic to the choice of background data and coalition estimator: it enforces relational validity as a plug-in restriction layer while preserving the chosen background modeling approach and the base estimator's sampling and weighting scheme.
    
    \item \emph{Computational effects:} \RelShap achieves \emph{orthogonal acceleration} by avoiding redundant coalition evaluations while preserving Shapley values (Prop.~\ref{prop:quotient_invariance}). We combinatorially characterize the resulting speedup factor as a function of the coalition estimator and FD structure (Thm.~\ref{thm:reduction}).
    
    \item We evaluate \RelShap on 9 datasets across 4 models and 3 coalition estimators, empirically validating both the \emph{semantic effects} and the runtime predictions of Thm.~\ref{thm:reduction}.
\end{enumerate}
\section{Preliminaries \& Related Work}
\label{sec:background}

Let $F$ be the set of input features, $f: \mathbb{R}^{|F|} \rightarrow \mathbb{R}$ a predictive model trained on $\mathcal{D}_{\mathrm{train}}$, and $\mathbf{x} \in \mathcal{D}_{\mathrm{test}}$ an instance to be explained. To explain the prediction $f(\mathbf{x})$, for any coalition $S \subseteq F$, $\mathbf{x}_S$ denotes the projection of $\mathbf{x}$ onto $S$. The Shapley value of feature $i \in F$ is defined as
{\scriptsize
\begin{equation}
\phi_i(f)
= \sum_{S \subseteq F \setminus \{i\}}
\frac{|S|!(|F|-|S|-1)!}{|F|!}
\bigl(f(\mathbf{x}_{S \cup \{i\}})-f(\mathbf{x}_S)\bigr).
\label{eq:eq1}
\end{equation}
}

As discussed in Section~\ref{sec:intro}, computing Eq.~\eqref{eq:eq1} in practice requires two design choices:
\emph{background data selection}, which determines how features in $F \setminus S$ are filled in when evaluating $f$, and \emph{coalition selection}, which subsets $S \subseteq F \setminus \{i\}$ are evaluated. Since exhaustive enumeration over reference points and over the $2^{|F|-1}$ coalitions is infeasible in practice, both are typically operationalized via \emph{sampling}. The two dimensions admit a multiplicity of instantiations~\cite{sundararajan2020many} and are typically treated as orthogonal~\cite{zern2023interventional, chen2020true}. \RelShap operates at the level of \emph{selection} rather than \emph{sampling}: it restricts the admissible background and coalition spaces to relationally valid configurations, so it applies whether the values are then computed exactly or approximated by any estimators. Appendices~\ref{app:details_background} and~\ref{app:coalition_samplers} provide additional details.

\subsection{Background Data Selection}
\label{sec:background:data}

Let $\mathcal{P}_{\mathrm{bg}}(\mathbf{z}_{F\setminus S})$ denote a background distribution over the absent features indexed by $F \setminus S$, and $B$ be the number of background samples (possibly all reference data points). With a completed input $\tilde{\mathbf{x}}(S, \mathbf{z}_{F\setminus S}) \in \mathbb{R}^{|F|}$, we define
\[
f(\mathbf{x}_S) =
\mathbb{E}_{\mathbf{z}_{F\setminus S} \sim \mathcal{P}_{\mathrm{bg}}}
\Bigl[
f\bigl(\tilde{\mathbf{x}}(S, \mathbf{z}_{F\setminus S})\bigr)
\Bigr].
\]

Methods differ in their choice of $\mathcal{P}_{\mathrm{bg}}$. \emph{Interventional methods}~\cite{janzing2020feature,zern2023interventional}, including default Kernel SHAP~\cite{lundbergL2017unified}, treat feature absence as an intervention and fill absent features from the empirical marginal distribution, which is assumption-light but can break feature dependencies. \emph{Conditional SHAP}~\cite{aas2021explaining} instead estimates the distribution of absent features given the observed ones, aiming to preserve dependencies but potentially introducing estimator error. \emph{Causal methods}~\cite{heskes2020causal} intervene through a causal graph among the features, and \emph{ManifoldShap}~\cite{taufiq2023manifold} restricts completions to an estimated data manifold. These methods span the \emph{true to the model} vs.\ \emph{true to the data} axis discussed in Section~\ref{sec:intro} (see Appendices~\ref{app:perspectives_bg_selection},~\ref{app:true_data_model} for details). 

Importantly, \RelShap is \emph{orthogonal} to this axis: rather than choosing \emph{what distribution} to put on the absent features, it constrains \emph{which configurations} the chosen distribution is allowed to place mass on under relational constraints. The two design dimensions therefore compose rather than conflict: \RelShap can be layered on top of any background data selection method. By default we use marginal (feature independence) as the base when no relational constraints are available (Def.~\ref{def:relshap:bg}), to comply with the no-distributional-assumptions property~\cite{sundararajan2020many, janzing2020feature} while removing the structurally infeasible inputs that \citet{fryeshapley} and \citet{taufiq2023manifold} warn about.

\subsection{Coalition Selection}
\label{sec:background:coal}

We draw $M$ coalitions from a coalition distribution $\mathcal{P}_{\mathrm{coal}}$ over $S \subseteq F \setminus \{i\}$ and approximate
\[
\phi_i(f) \approx \frac{1}{M} \sum_{m=1}^{M} w\bigl(S^{(m)}\bigr) \Bigl( f(\mathbf{x}_{S^{(m)} \cup \{i\}}) - f(\mathbf{x}_{S^{(m)}}) \Bigr).
\]
Different Shapley value estimators choose different $\mathcal{P}_{\mathrm{coal}}$ and importance weight $w$. \emph{Kernel SHAP}~\cite{lundbergL2017unified} fits a weighted linear surrogate using the Shapley kernel, which concentrates weight on extreme (very small or large) coalition sizes. \emph{Monte Carlo (MC)} estimator samples coalitions from the Shapley-weighted distribution, with coalition sizes sampled approximately uniformly. \emph{Leverage SHAP}~\cite{muscoprovably} uses leverage score-based sampling, sampling the coalition size nearly uniformly and then sampling a coalition uniformly within that size. Recent variants further improve estimation through residual estimation~\cite{witterregression} or Fourier basis reduction~\cite{fumagalli2026odd}.

\RelShap is \emph{estimator-agnostic} and provides \emph{orthogonal acceleration}: it preserves any base estimator's sampling distribution, weights, and accuracy guarantees, canonicalizing sampled coalitions under relational validity (Def.~\ref{def:quotient_projection}). 
\citet{witter2026exact} also uses equivalence classes, but derives them from structural causal models (see Appendix~\ref{app:closure} for a detailed comparison). \citet{maafa2018algorithms} studies lattice-structured coalition spaces to reduce computations, but not for ML predictions.

\subsection{Relational Structure in ML Data}
\label{sec:background:relational}
 
Many ML datasets originate from relational databases, where data is organized into multiple tables linked by keys. Flattening this structure into a single table for model training discards integrity constraints that govern valid data combinations. A \emph{functional dependency} (FD) $A \rightarrow B$ states that the value of attributes in $A$ uniquely determines the value of $B$; for instance, in Example~\ref{ex:motivating}, $\texttt{age} \rightarrow \texttt{life\_stage}$ means each age value maps to exactly one life stage. \emph{Domain constraints} restrict attributes to valid ranges conditioned on other attributes (\eg if $\texttt{region} = \text{`North America'}$ then $\texttt{currency} \in \{\text{`USD'}, \text{`CAD'}, \text{`MXN'}\}$). \emph{Denial constraints} capture structural prohibitions such as mutual exclusivity in one-hot encoded features. Standard Shapley estimators are free to violate all of these, evaluating the model on feature combinations that could never arise in the data-generating process. \RelShap systematically prevents this. A fuller introduction to relational concepts is in Appendix~\ref{app:relational_primer}.
\section{The \RelShap Framework}
\label{sec:relshap}

\subsection{Relational Constraints} \label{sec:extracting_constraints}

\RelShap first prepares relational constraints for downstream modules, either directly from relational inputs or from flattened-only inputs through normalization (Figure~\ref{fig:workflow_revision}).
We denote a relational schema as $\mathcal{S}$, a query as $\mathcal{Q}$, and the resulting flattened dataset as $\mathcal{D}$.
The schema, query, and resulting data together induce a collection of relational constraints, which we denote by
$\Sigma^{\{\mathcal{S},\mathcal{Q},\mathcal{D}\}}
= \Sigma_{\mathrm{FD}} \cup \Sigma_{\mathrm{dom}} \cup \Sigma_{\mathrm{den}}$, where $\Sigma_{\mathrm{FD}}$,  $\Sigma_{\mathrm{dom}}$, and $\Sigma_{\mathrm{den}}$ are the sets of FDs, domain constraints, and denial constraints, respectively.
After extracting all relational constraints, users inspect them and specify which types of constraints to include and to what extent; we denote the resulting set by $\Sigma$. Details are in Appendix~\ref{app:relational_constraints}.

\textbf{Relational schema.}
$\Sigma^S$ contains schema-declared integrity constraints: FDs from primary and unique keys, dependencies induced by foreign keys (applied when the corresponding relations are joined in the query), and conditional domain constraints (\eg \texttt{CHECK} clauses), ignoring unary type constraints already satisfied by the data.

\textbf{Query.}
We parse $\mathcal{Q}$ using SQLGlot~\cite{sqlglot} to extract query-induced FDs and conditional domain rules.
$\Sigma^{\mathcal{Q}}_{\mathrm{FD}}$ identifies \texttt{GROUP BY} clauses, window aggregation, top-1 selection patterns, and self-joins.
$\Sigma^{\mathcal{Q}}_{\mathrm{dom}}$ captures \texttt{WHERE} and \texttt{JOIN} predicates.

\textbf{Flattened data.}
\RelShap discovers exact minimal FDs with bounded left-hand-side size (default $2$) and single-attribute right-hand sides; multi-attribute right-hand sides can be derived via Armstrong's axioms~\cite{armstrong1974dependency}. For any $L \subseteq F$ and attribute $a \notin L$, let $\mathcal{D}/L$ denote the partition of $\mathcal{D}$ induced by equality on $L$; an FD $L \rightarrow a$ holds if, for all equivalence classes $C \in \mathcal{D}/L$, $|\mathrm{DISTINCT}_a(C)| \le 1$. For each discovered FD $A \rightarrow B$ with $|\mathrm{dom}(B)| \le 20$ (default), we derive conditional domain rules $(B = s) \Rightarrow \varphi(A)$: for categorical $A$, $\varphi(A)$ is the set of observed values; for continuous $A$, an interval $[\ell_s, u_s]$ estimated from the data. We also extract denial constraints by identifying cycles in the FD set (\eg $A \rightarrow B$, $B \rightarrow C$, $C \rightarrow A$) and testing for structural patterns common in ML data, such as scaled one-hot encodings and exclusive-or among binary attributes. By default, we use $\mathcal{D}$ for constraint discovery and $\mathcal{D}_{\mathrm{train}}$ as the reference distribution (see Appendix~\ref{app:split}).

\begin{figure}[t]
 \centering
  \includegraphics[width=\linewidth
]{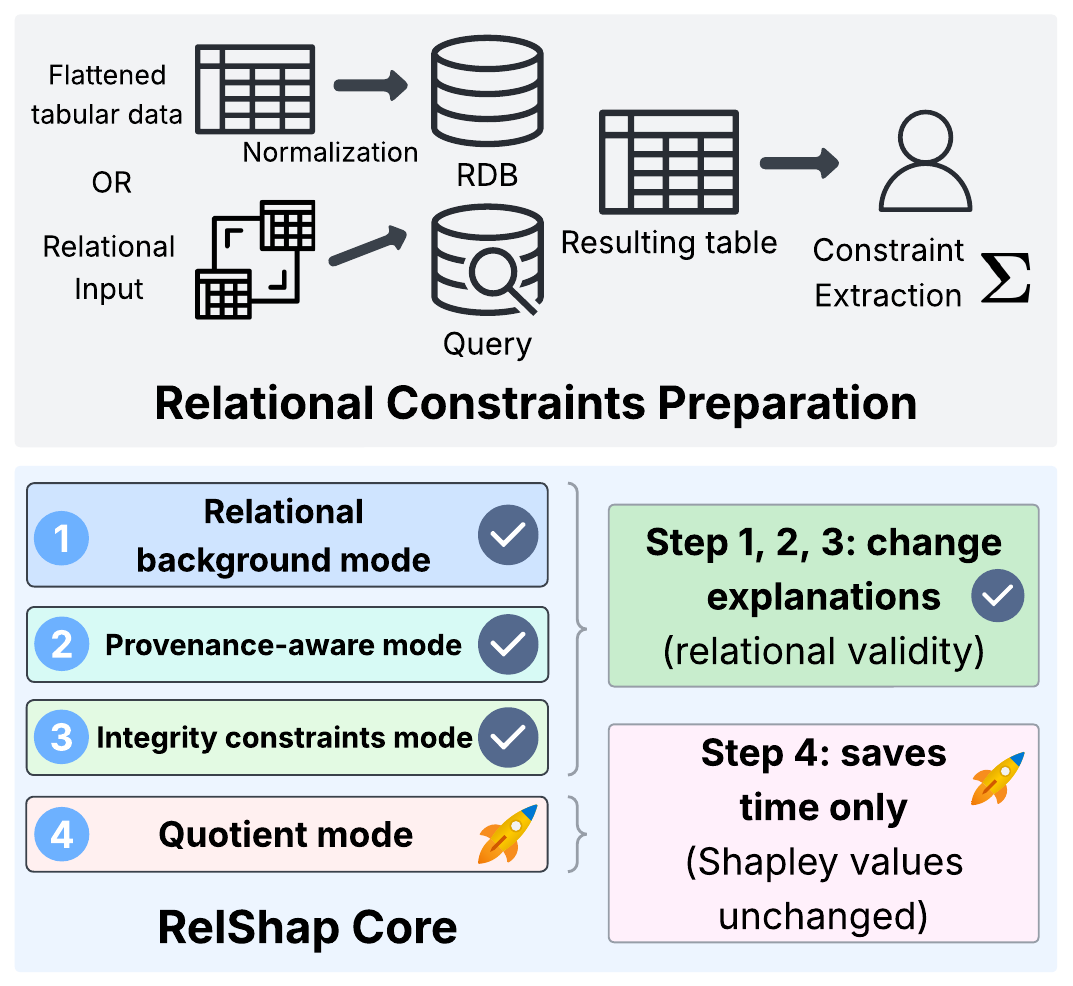}
\caption{\RelShap workflow. The upper panel extracts relational constraints; the core applies relational background ($\Sigma_\mathrm{FD}$), provenance-aware, and integrity constraints modes, which change explanations, and quotient mode, which reduces runtime while leaving Shapley values unchanged.}
\label{fig:workflow_revision}
\end{figure}

\subsection{Background and Coalition Selection}
\label{sec:relshap:bg_coal_main}

\RelShap incorporates two types of relational constraints: \emph{global}, applied to all test instances, and \emph{local}, instantiated per test instance. $\Sigma_{\mathrm{FD}}$ is global; provenance-aware constraints, $\Sigma_{\mathrm{dom}}$, and $\Sigma_{\mathrm{den}}$ are applied locally.

\begin{definition}[Relationally Consistent Background Distribution]
\label{def:relshap:bg}
Given a feature set $F$, a coalition $S \subseteq F$, an instance $\mathbf{x}$, and a set of relational constraints $\Sigma$, we define
\[
\mathcal{P}_{\mathrm{bg}}^{\mathrm{Rel}}
=
\begin{cases}
\mathcal{P}_{\mathrm{bg}}^{\mathrm{Marginal}},
& \text{if } \Sigma = \varnothing, \\
\widehat{P}\bigl(
\mathbf{Z}_{F\setminus S}
\,\big|\,
\tilde{\mathbf{x}}(S,\mathbf{Z}_{F\setminus S}) \models \Sigma
\bigr)
& \text{otherwise.}
\end{cases}
\]
\end{definition}

Under Def.~\ref{def:relshap:bg}, any coalitions $S$ and $S'$ are \emph{$\Sigma$-equivalent}, i.e., $S \sim_{\Sigma} S'$, if they induce the same set of relationally valid completions under $\Sigma$. Accordingly, $[S]_{\Sigma}$ denotes the \emph{equivalence class} of $S$ under $\sim_{\Sigma}$, and the \emph{$\Sigma$-coalition space} $\mathcal{C}_{\Sigma}$ is the \emph{quotient} of the original coalition space by $\sim_{\Sigma}$, collapsing coalitions indistinguishable under $\Sigma$.

\begin{definition}[Quotient space coalition projection] \label{def:quotient_projection}
Given $\mathcal{C}_{\Sigma}$, \emph{quotient space coalition projection} is a post-sampling procedure that maps a coalition $S \subseteq F \setminus \{i\}$ drawn by any sampler to its $\Sigma$-equivalence class $[S]_{\Sigma}$ (i.e., each coalition value is replaced by its canonical representative under $\Sigma$).
\end{definition}

For example, if $\Sigma$ contains an FD $a \rightarrow b$ over $\{a,b,c\}$, then $\{a\} \sim_{\Sigma} \{a,b\}$ and $\{a,c\} \sim_{\Sigma} \{a,b,c\}$, and hence each pair belongs to the same $\Sigma$-equivalence class. Under Def.~\ref{def:quotient_projection}, once a coalition is evaluated, subsequent samples mapping to the same $\Sigma$-equivalence class reuse cached results without additional model calls. Two properties hold, for distinct reasons. First, quotient space coalition projection (quotient mode) leaves Shapley values unchanged (Prop.~\ref{prop:quotient_invariance}) because it only reuses cached evaluations of equivalent coalitions. Second, the base estimator's coalition distribution, weighting scheme, and accuracy guarantees are preserved because all constraints (FDs, provenance, integrity constraints) are applied to each sampled coalition \emph{before} quotient deduplication (Algorithm~\ref{alg:rel_coal_selection}). Together these provide \emph{orthogonal acceleration} under $\Sigma \neq \varnothing$. See Appendix~\ref{app:closure} for details and proof.

\begin{proposition}[Shapley value invariance under quotient space coalition projection]
\label{prop:quotient_invariance}
\emph{Quotient space coalition projection does not change the resulting Shapley value of feature $i$, i.e., using only relational background distribution with or without quotient space coalition projection yields identical Shapley values, and this holds for all coalition estimators.}
\end{proposition}

Moreover, whenever $\sim_{\Sigma}$ is non-trivial, quotient projection strictly reduces the coalition space, $|\mathcal{C}_{\Sigma}| < 2^{|F|-1}$ (see Appendix~\ref{app:proof:coalition_reduction}); Section~\ref{sec:algorithms} quantifies the expected reduction.

\subsubsection{Provenance-aware mode}

\emph{Provenance-aware} mode incorporates \emph{local} constraints: identifier-induced FDs (primary and foreign keys; they are part of global $\Sigma_\mathrm{FD}$ but applied instance by instance) no longer explicit in the flattened table (\eg when identifiers are dropped as in Example~\ref{ex:motivating}). Given a coalition, we check whether an identifier can be inferred from the observed features. In strict mode, dependencies are applied only when the identifier is uniquely determined; in relaxed mode, dependencies are applied to any attribute whose value is constant across the narrowed candidate set. For instance, in Example~\ref{ex:motivating}, observing only the age of \texttt{a29} yields candidates $\{\texttt{a12}, \texttt{a29}\}$; although the exact identifier is ambiguous, \texttt{life\_stage} (which may already be determined by \texttt{age}) and \texttt{empl} take the same value for both and can still be used. Provenance-aware mode is applicable even when no other relational structure exists within the feature tables (\eg discovered FDs), and when identifiers have been dropped during flattening. Standard (non-provenance) \RelShap is typically sufficient in with-key settings where identifiers are retained (\eg recommender systems).

\subsection{Algorithm and Runtime Analysis}
\label{sec:algorithms}

\begin{algorithm}[t]
\small 
\caption{\textsc{RelShap}}
\label{alg:rel_coal_selection}
\begin{algorithmic}[1]
\State \textbf{Input:} $\mathcal{P}_{\mathrm{coal}}$, $\Sigma$, $\mathbf{x}$ (instance), $M$ (coalition budget)
\State \textbf{Output:} canonical coalitions $\mathcal{M}$
\State $\mathrm{Seen} \gets \emptyset$, \quad $\mathcal{M} \gets [\ ]$

\For{$t=1$ to $M$}
\State sample a raw coalition $m \sim \mathcal{P}_{\mathrm{coal}}$

\If{\emph{relational background mode} under $\Sigma_\mathrm{FD}$}
    \State $m \gets \textsc{CanonicalizeCoal}(m)$
    \EndIf

\If{\emph{provenance-aware} mode}
    \State $m \gets m \cup \textsc{ProvExpand}(m,\mathbf{x},\Sigma_{\mathrm{FD}})$
\EndIf

\If{\emph{integrity constraints} mode}
    \State $m \gets m \cup 
    \textsc{ICRepair}(m,\mathbf{x},\Sigma_{\mathrm{dom}}\cup\Sigma_{\mathrm{den}})$
\EndIf

\If{\emph{quotient} mode and $m \in \mathrm{Seen}$}
    \State \textbf{continue} \Comment{reuse cached evaluation}
\EndIf
\State $\mathrm{Seen} \gets \mathrm{Seen} \cup \{m\}$; $\mathcal{M} \gets \mathcal{M} \cup \{m\}$

\EndFor

\State \Return $\mathcal{M}$
\end{algorithmic}
\end{algorithm}

Algorithm~\ref{alg:rel_coal_selection} orchestrates \RelShap: it draws $M$ raw coalitions from a base coalition distribution, maps each to its canonical representative, expands it with local provenance information when available, applies integrity constraints, and deduplicates coalitions. \textsc{CanonicalizeCoal} lets $\Sigma$-equivalent coalitions share one evaluation based on \textsc{FDClosure}; \textsc{ProvExpand} incorporates identifier-induced dependencies using prebuilt inverted and row-set indexes, requiring only identifier-level lookups independent of dataset size, rather than full group-by operations; and \textsc{ICRepair} extends the coalition with features whose values are forced by domain and denial constraints (all subroutines are in Appendix~\ref{app:algo_details}). \RelShap first samples coalitions from the base coalition distribution and then applies all available relational constraints, reusing cached evaluations when possible and evaluating new coalitions otherwise. This procedure preserves the base sampler's draws while enforcing relational consistency and avoiding redundant evaluations ($|\mathcal{M}| \le M$).

\paragraph{Runtime analysis.}\quad The runtime reduction from \emph{quotient space coalition projection} is the decrease from the sampling budget $M$ to the number of distinct projected coalitions $|\mathcal{M}|$. Prior works~\cite{demetrovics1985minimum, demetrovics1992functional} characterize $|\mathcal{C}_{\Sigma}|$ for representative FD structures (Appendix~\ref{app:fd_lattice}), but only under deterministic coalition evaluation. \RelShap also supports \emph{stochastic} sampling from $\mathcal{P}_{\mathrm{coal}}$, where widely used estimators sample with coalition size-dependent probabilities $p_k=\Pr(|S|=k)$ (Section~\ref{sec:background:coal}). We therefore decompose the probability of reaching each class by coalition size and derive the expected reduction combinatorially: speedup arises when the sampler repeatedly hits the same  equivalence class, allowing evaluations to be reused.

\begin{theorem}[Expected reduction under quotient mode]
\label{thm:reduction}
Let $q(C)=\Pr_{S\sim \mathcal{P}_{\mathrm{coal}}}([S]_{\Sigma}=C)$ be the probability that a sampled coalition maps to $\Sigma$-equivalence class $C \in \mathcal{C}_{\Sigma}$. If the draws are i.i.d., the expected number of distinct classes evaluated after projection mode is
\[
\mathbb{E}[K_M]
=
\sum_{C\in\mathcal{C}_{\Sigma}}
\left(1-(1-q(C))^M\right),
\]
where $q(C)=\sum_{k=r(C)}^{|F|-1}
p_k\frac{N_C(k)}{\binom{|F|-1}{k}}$ is the probability of a sampled coalition landing in class $C$. Hence the normalized expected runtime speedup factor relative to the baseline sampler is
$R_M=1-\mathbb{E}[K_M]/{M}.$
\end{theorem}

$N_C(k)$ is the number of size-$k$ coalitions whose projection belongs to $C$, and $r(C)=\min\{|S|\mid [S]_{\Sigma}=C\}$ is the minimum size of any coalition generating $C$ (rank). For the FD $a \rightarrow b$ over $\{a,b,c\}$ from Section~\ref{sec:relshap:bg_coal_main}, quotient mode reduces $|\mathcal{C}_{\Sigma}|$ from $2^{3}=8$ to $6$. Intuitively, FDs with small left-hand sides (LHSs) induce classes with small $r(C)$, so many coalitions collapse to the same class, yielding large $N_C(k)$ at small sizes. The improvement for \emph{Kernel SHAP} is therefore amplified when $\Sigma$ contains many low-arity dependencies, since it concentrates mass on extreme coalition sizes and FDs with small LHSs are common in practice; \emph{Monte Carlo} and \emph{Leverage SHAP} spread mass more evenly across $k$, so their improvement tracks the overall magnitude of $N_C(k)$. We ignore implementation-level costs (\eg lookup or indexing) and assume i.i.d.\ draws; sequential samplers may introduce mild dependence in practice. Appendix~\ref{app:runtime_reduction_projection_proof} provides a proof and Section~\ref{sec:experiments} confirms that observed speedups closely match this analysis.

\section{Experimental Evaluation} 
\label{sec:experiments}

We evaluate \RelShap on 9 datasets: 5 standard ML datasets normalized into relational form~\cite{papenbrock2017data} --- Amazon Employee Access~\cite{openml_amazon_employee_access_43900}, Churn~\cite{openml_churn_40701}, Churn Modelling~\cite{kaggle_churn_modelling}, German Credit~\cite{openml_credit_g_31}, and SpeedDating~\cite{openml_speeddating_40536} --- and 4 relational datasets: TPC-H~\cite{tpch}, UW-CSE~\cite{motl2015ctu, uw_cse_dataset}, MovieLens 20M~\cite{kaggle_movielens_20m}, and a Synthetic dataset extending Example~\ref{ex:motivating} with domain and denial constraints. 

We train logistic regression, XGBoost~\cite{chen2016xgboost}, Random Forest~\cite{liaw2002classification}, and MLP-PLR~\cite{gorishniy2022embeddings}, covering linear, tree-based, and neural models, and compute explanations with Kernel SHAP, MC, and Leverage SHAP (means and standard deviations over three seeds) for $n_{\text{explain}} = \min(200, |\mathcal{D}_{\text{test}}|)$ instances per dataset, using a per-dataset convergence budget $M_{\mathrm{conv}}$ at which Shapley estimates stabilize~\cite{covert2020improving, yuan2022empirical}. 

Relational constraint extraction is a one-time preprocessing step reused across all explanations; it completes within 12 seconds on every dataset except SpeedDating ($482.7$s) and MovieLens 20M ($98.9$s), where data-driven FD discovery dominates, and its cost, typically offset by the resulting speedups, is excluded from runtime comparisons. Setup details are in Appendices~\ref{app:dataset} and~\ref{app:detailed_exp_setups}.

\paragraph{In summary,} our experiments show that standard estimators query impossible worlds constantly; when ground truth is known, \RelShap is the only method that recovers the correct explanation; on real data, the correction is large and systematic; and it comes at reduced, not increased, runtime.

\subsection{Validation with Ground Truth}
\label{sec:exp:intuition}

We demonstrate that \RelShap yields more intuitive explanations than existing methods in a controlled setting, following experiments in~\citet{taufiq2023manifold}. Building on Example~\ref{ex:motivating}, we consider a loan approval scenario restricted to two features (\texttt{age} and \texttt{life\_stage}), with an FD $\texttt{age} \rightarrow \texttt{life\_stage}$. We define a synthetic predictor
$
g(\mathbf{x})
=\mathbf{1}\{\texttt{age}>50\}
+\delta\, r(\texttt{life\_stage})\,
\mathbf{1}\{\mathbf{x}\not\models\Sigma\},
$
where $r(\texttt{life\_stage})\in\{0.3,0.6,0.9\}$ assigns an arbitrary constant to each of the three life stages, and $\delta$ ranges from $0$ to $10$ in increments of $0.5$. Since every data instance satisfies $\Sigma$, the second perturbation term is inactive during training and testing and is triggered only by relationally invalid combinations generated during Shapley computation. Thus, predictions on the valid domain depend \emph{only} on \texttt{age}, and a semantically intuitive explanation should assign greater importance to \texttt{age} than to \texttt{life\_stage} regardless of $\delta$.

Figure~\ref{fig:relshap_intuition} shows how often each feature receives the larger absolute attribution as $\delta$ increases. All three baselines rank \texttt{age} first more often than \texttt{life\_stage}; however, Kernel SHAP increasingly shifts attribution toward \texttt{life\_stage} as invalid perturbations grow. Conditional SHAP and ManifoldShap are less sensitive to increasing $\delta$, but still rank \texttt{life\_stage} first for a substantial fraction of instances. In contrast, \RelShap consistently assigns the larger attribution to \texttt{age} for all $\delta$. This indicates that the explanation change from \RelShap is not a mere difference; it is an \emph{improvement} with respect to relational validity, achieved by eliminating \emph{impossible worlds}.

The mechanism behind these differences is instructive. We measure \emph{violation prevalence}: the fraction of coalitions violating at least one relational constraint. Kernel SHAP fills \texttt{life\_stage} from the marginal distribution, independently of \texttt{age}, so $31.8\%$ of its coalitions violate the FD (\eg $\texttt{age}=35$ with $\texttt{life\_stage}=\texttt{older}$); each violation activates the perturbation term $\delta\,r(\texttt{life\_stage})$, and the resulting spurious attribution to \texttt{life\_stage} grows with $\delta$. Conditional SHAP, which estimates the conditional distribution of \texttt{life\_stage} given \texttt{age}, and ManifoldShap, which restricts completions to an estimated data manifold, suppress most, but not all, violations ($4.8\%$ and $5.1\%$, respectively), explaining their partial yet incomplete robustness. \RelShap enforces the FD exactly: its violation prevalence is $0\%$, the perturbation term never activates, and \texttt{life\_stage} receives no spurious attribution. Note that violation prevalence itself does not depend on $\delta$: increasing $\delta$ makes each violation more costly, not more frequent, which is why misattribution grows with $\delta$ while prevalence stays fixed.

\begin{figure}[t]
  \centering
  \includegraphics[width=\linewidth
]{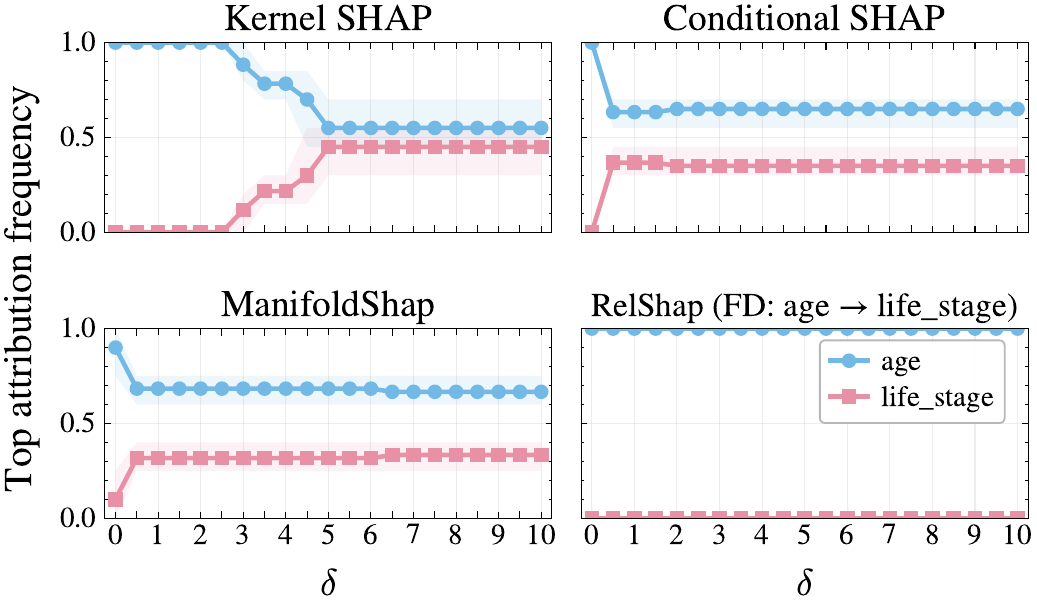}
  \caption{Comparison of top attribution frequency across methods as $\delta$ varies; exact Shapley values computed via full coalition enumeration.}
\label{fig:relshap_intuition}
\end{figure}

\begin{figure*}[t] 
\centering 
\begin{subfigure}[t]{0.33\linewidth} 
\centering 
\includegraphics[width=\linewidth]{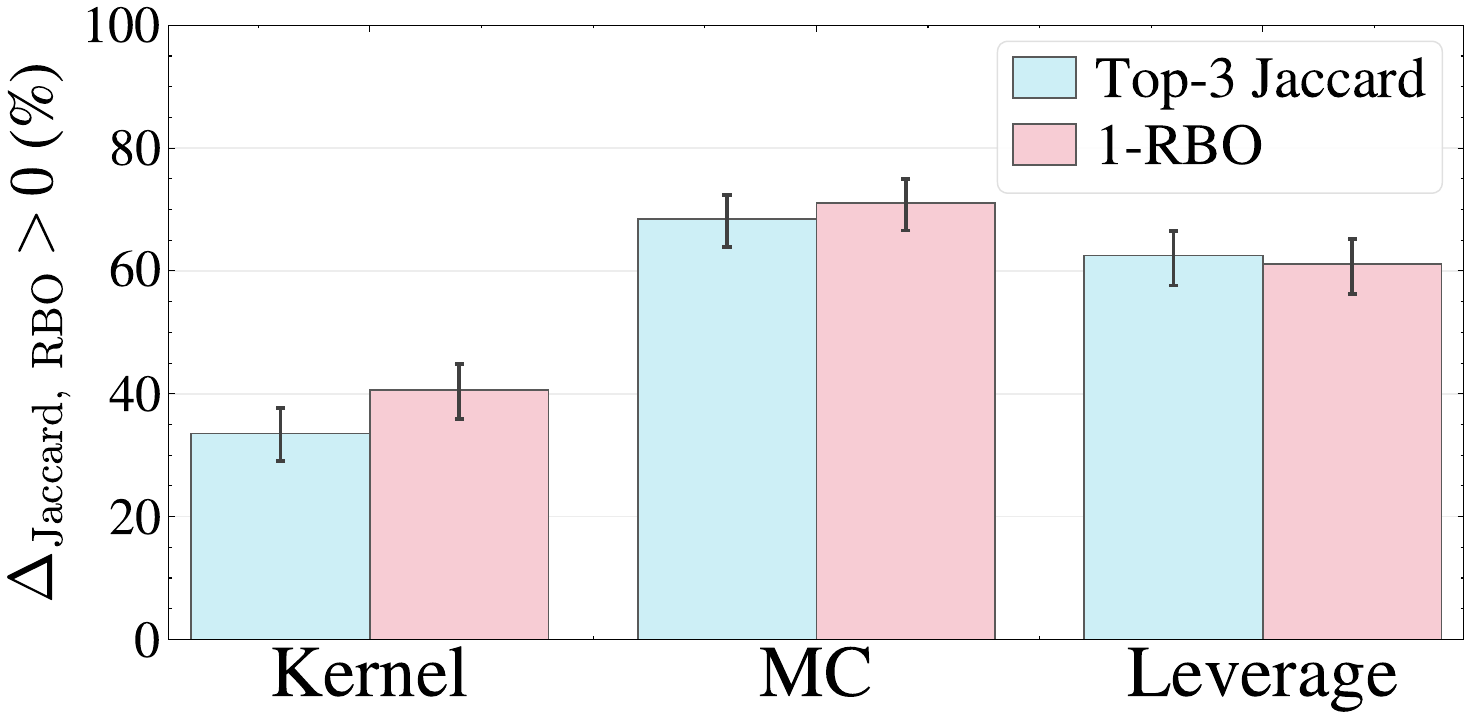}%
\caption{Aggregated by estimators.} 
\label{fig:delta_estimator} 
\end{subfigure}\hfill
\begin{subfigure}[t]{0.33\linewidth} 
\centering 
\includegraphics[width=\linewidth]{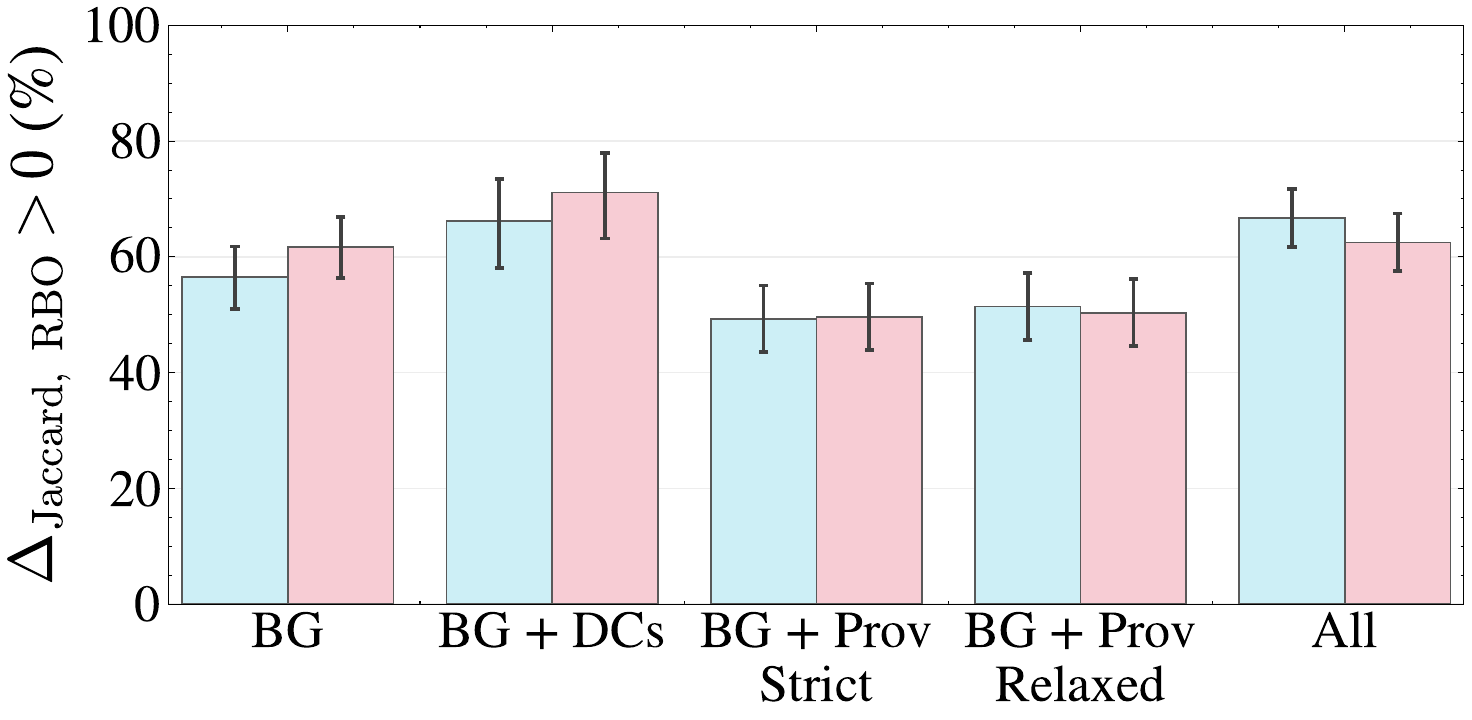} 
\caption{Aggregated by comparison modes.} 
\label{fig:delta_comparison} 
\end{subfigure}\hfill
\begin{subfigure}[t]{0.33\linewidth} 
\centering 
\includegraphics[width=\linewidth]{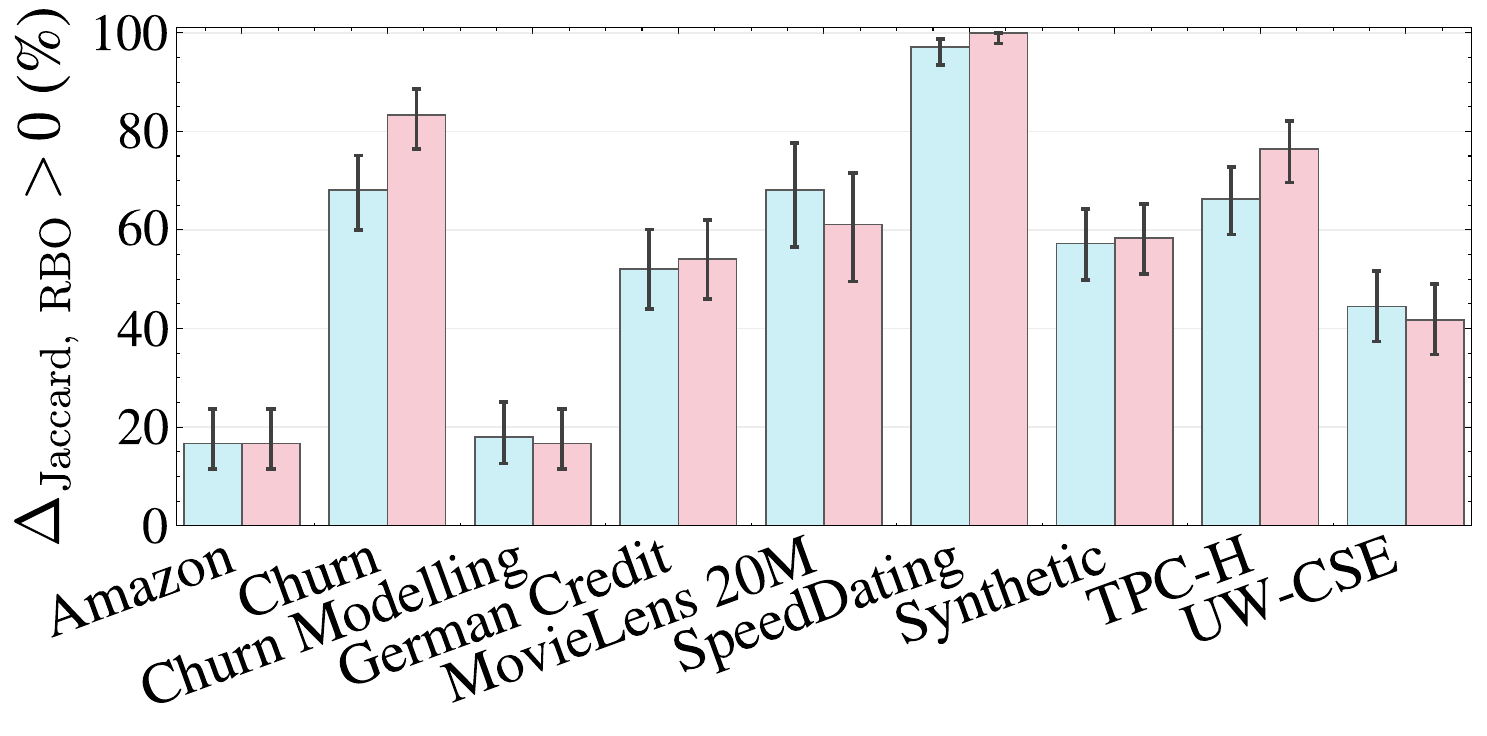} 
\caption{Aggregated by datasets.} 
\label{fig:delta_dataset} 
\end{subfigure}
\caption{Percentage of cases with $\Delta > 0$, sliced by estimators, comparison modes, and datasets. Error bars denote 95\% CIs.}
\label{fig:delta_positive}
\end{figure*}

These results generalize beyond two features. We repeat the experiment over four features (\texttt{age}, \texttt{life\_stage}, \texttt{empl}, \texttt{total\_amt}), where attribution ordering between a particular feature pair is no longer expected to hold in isolation, and examine how the overall attribution vector and feature ranking change as $\delta$ increases (Appendix~\ref{app:exp_intuition}). \RelShap remains unchanged across all metrics, while all baselines exhibit increasing attribution and ranking shifts; a setting with additional provenance constraints behaves identically (Appendix~\ref{app:exp_intuition}). This shows that eliminating impossible worlds makes \RelShap insensitive to relationally invalid perturbations that are never observed during training or testing.

Beyond the controlled setting, the UW-CSE dataset permits full background and coalition enumeration. \RelShap changes the feature ranking even under exact computation, showing that the change is not a sampling artifact. Quotient mode further reduces runtime while preserving the Shapley values exactly, providing an empirical verification of Prop.~\ref{prop:quotient_invariance}.
Additional validations on real datasets are in Appendix~\ref{app:exp_intuition}.

\subsection{Impossible Worlds in Practice}
\label{sec:exp:practice}

We compare default estimators against five \RelShap configurations with progressively richer constraints: BG (relational background mode); BG~+~DCs; BG~+~Prov, Strict/Relaxed (Relaxed subsumes Strict); and All (BG~+~Prov, Relaxed~+~DCs). We first quantify how frequently conventional Shapley computation is exposed to \emph{impossible worlds}. Across the three estimators, average violation prevalence rises from $64$--$74\%$ under BG to $82$--$93\%$ under BG~+~Prov, Relaxed on datasets without DCs, and from $68$--$70\%$ under BG to $72$--$76\%$ under All (Appendix~\ref{app:violation_details}): invalidity becomes more pervasive as richer constraints are incorporated.

\begin{table}[b]
\centering
\small
\setlength{\tabcolsep}{3pt}
\begin{tabular}{lcccc}
\toprule
Mode & Top-3 Jaccard & $1-\mathrm{RBO}$ & $\Delta_{\mathrm{Jaccard}}$ & $\Delta_{\mathrm{RBO}}$ \\
\midrule
BG & 0.452 & 0.190 & 0.057 & 0.050  \\
BG + DCs & 0.566 & 0.216 & 0.168 & 0.094 \\
BG + Prov, Strict & 0.377 & 0.142 & $-0.023$ & 0.011  \\
BG + Prov, Relaxed & 0.415 & 0.155 & 0.015 & 0.024  \\
All & 0.430 & 0.172 & 0.036 & 0.032  \\
\bottomrule
\end{tabular}
\caption{Explanation differences by comparison mode (Default vs.\ mode), averaged over datasets, models, and estimators; standard deviations in Appendix Table~\ref{tab:summary_main}.}
\label{tab:mode_effects}
\end{table}

Prevalence alone does not establish that explanations change; that also depends on the model's response in the invalid region. We therefore assess how \RelShap changes Shapley values relative to default estimators, finding shifts well beyond random ranking perturbation. For each instance, we measure ranking differences using \emph{Top-$3$ Jaccard distance} and \emph{$1-\mathrm{RBO}$} (rank-biased overlap)~\cite{webber2010similarity}, capturing the three most influential features and the full top-weighted ranking, respectively. Against the Mallows null baseline range of \citet{hwang2026explanation}, we define $\Delta$ as the signed deviation from the range midpoint; $\Delta>0$ indicates a change larger than expected at random, and we report $95\%$ confidence intervals (CIs) on the fraction of such cases.

Figure~\ref{fig:delta_positive} shows the percentage of cases with $\Delta>0$. This percentage exceeds $50\%$ across most \RelShap configurations, with changes occurring most frequently under MC, followed by Leverage SHAP and Kernel SHAP. Dataset-level effects are generally larger for constraint-rich datasets---reaching nearly $100\%$ for SpeedDating---but are not strictly monotonic in the number of constraints (Appendix~\ref{app:more_experiments}).
The effects are stable across predictive models, confirming that \RelShap is generally model-agnostic (Appendix~\ref{app:exp:semantic_effects}).
Together with Section~\ref{sec:exp:intuition}, these results indicate that the observed ranking differences reflect meaningful corrections toward relationally valid explanations. 

Table~\ref{tab:mode_effects} reports the magnitude of these changes by configuration: domain and denial constraints yield the largest corrections, provenance effects depend on the mode, and $\Delta > 0$ on average in all but one mode--metric pair.  On TPC-H, provenance mode replaces the top-3 features almost entirely (Top-3 Jaccard $0.91$--$0.96$ under Kernel SHAP and MC; Appendix Table~\ref{tab:prov_tpch_uwcse_xgboost}).

\subsection{Effectiveness of Running Time Optimization}
\label{sec:exp:runtime}

Quotient mode (Q) applies the quotient space coalition projection of Def.~\ref{def:quotient_projection} as an orthogonal acceleration on top of any configuration (Appendix~\ref{app:exp:scalability}). Across datasets and models, this optimization consistently reduces runtime for MC and Leverage SHAP, particularly on computationally costly datasets (\eg Amazon, TPC-H); Kernel SHAP also benefits from Q but is more variable, consistent with the estimator-dependent analysis in Section~\ref{sec:algorithms}. Added relational processing does not typically impose a runtime penalty, and the reduction from Q grows with both background size and coalition budget. Finally, observed reductions align with our combinatorial analysis: the empirical counterpart of $R_M$ in Thm.~\ref{thm:reduction} correlates closely with measured reduction (Pearson: mean $0.806$, median $0.875$; Kendall $\tau$: mean $0.741$, median $0.833$), reliably predicting \emph{when} reduction occurs, though its magnitude depends on the dataset and sampling strategy. 

\section{Conclusion}

Conventional Shapley value methods treat feature coalitions as unrestricted subsets, in effect querying relationally impossible worlds. \RelShap is the first estimator-agnostic framework to systematically enforce relational consistency. Extensive experiments show that the resulting explanations differ meaningfully from those of unconstrained estimators, and that in settings with known ground truth, the difference is an improvement. Quotient mode preserves Shapley values exactly while provably reducing runtime. 

\textbf{Limitations.}\quad Our validation establishes improvement with respect to relational validity and synthetic ground truth, not human judgment; a user study is important future work. \RelShap also treats the extracted, user-vetted constraints as correct: spurious FDs discovered from small data would propagate to explanations. See Appendix~\ref{app:limitations} for details.

\section*{Acknowledgments}
This work was supported by NSF Awards No. 
2312930 and 2326193. This work was supported in part through the NYU IT High Performance Computing resources, services, and staff expertise.

\bibliography{refs}

\newpage
\appendix
\section{Primer on Relational Concepts}
\label{app:relational_primer}

This appendix provides background on relational database concepts used throughout the paper, aimed at readers whose primary background is in machine learning or explainable AI rather than data management.

\subsection{Relational databases and tables}

A \emph{relational database} organizes data into a collection of \emph{tables} (also called \emph{relations}), each with a fixed set of named columns (\emph{attributes}) and a variable set of rows (\emph{tuples}). For example, in Example~\ref{ex:motivating}, the database contains two tables: \texttt{Applicants} with attributes \texttt{a\_id}, \texttt{age}, \texttt{life\_stage}, \texttt{empl}, and \texttt{Transactions} with attributes \texttt{t\_id}, \texttt{a\_id}, \texttt{amount}.

Tables are linked through shared attributes. In the example, \texttt{a\_id} appears in both tables: it uniquely identifies each applicant in \texttt{Applicants} and references the corresponding applicant in \texttt{Transactions}. One applicant may have many transactions (a \emph{one-to-many} relationship). This multi-table structure allows the database to store each applicant's attributes once, avoiding redundancy.

\subsection{Keys and functional dependencies}

A \emph{key} is a set of attributes that uniquely identifies each row in a table. In \texttt{Applicants}, \texttt{a\_id} is the \emph{primary key}: no two rows share the same \texttt{a\_id}. A key induces \emph{functional dependencies} (FDs): if \texttt{a\_id} uniquely identifies a row, then knowing \texttt{a\_id} determines all other attributes in that row. Formally, we write $\texttt{a\_id} \rightarrow \texttt{age}$, $\texttt{a\_id} \rightarrow \texttt{life\_stage}$, etc.

FDs are not limited to keys. Any deterministic relationship between attributes is an FD. For instance, if every 35-year-old in the data has life stage ``middle,'' every 63-year-old has life stage ``older,'' and so on, then $\texttt{age} \rightarrow \texttt{life\_stage}$ holds. This FD can be \emph{discovered} directly from the data, without any schema declaration. FD discovery is a well-studied problem in the database community~\cite{huhtala1999tane, papenbrock2015functional}.

\subsection{Queries and aggregation}

A \emph{query} retrieves and transforms data from one or more tables. The SQL query in Example~\ref{ex:motivating} joins \texttt{Applicants} with \texttt{Transactions} (matching rows by \texttt{a\_id}), then groups the result by applicant and computes the sum of transaction amounts (\texttt{SUM(t.amount) AS total\_amt}). The \texttt{GROUP BY} clause produces exactly one output row per applicant, which means the grouping attributes (here, \texttt{a\_id}) uniquely determine the aggregated value. This creates a \emph{query-induced FD}: $\texttt{a\_id} \rightarrow \texttt{total\_amt}$.

After the query executes, the identifier \texttt{a\_id} is typically dropped before model training, since identifiers carry no predictive signal. However, the FDs that \texttt{a\_id} induced remain latent in the flattened table.

\subsection{Integrity constraints beyond FDs}

Relational databases support several kinds of constraints beyond functional dependencies:

\paragraph{Domain constraints} restrict the values an attribute can take, possibly conditioned on other attributes. For example, if $\texttt{region} = \text{`North America'}$ then $\texttt{currency} \in \{\text{`USD'}, \text{`CAD'}, \text{`MXN'}\}$. These constraints capture real-world restrictions that the data must satisfy.

\paragraph{Denial constraints} express prohibitions: certain combinations of attribute values must never co-occur. A common example arises with one-hot encoded categorical variables. If a categorical attribute \texttt{gender} is encoded as three binary columns (\texttt{gender\_male}, \texttt{gender\_female}, \texttt{gender\_nonbinary}), then at most one column may equal~1 in any row. Standard Shapley estimators can produce completions like $(1, 1, 0)$, violating this mutual exclusivity.

\paragraph{Referential constraints} (foreign keys) require that a value appearing in one table must exist in another. In the example, every \texttt{a\_id} in \texttt{Transactions} must correspond to an existing \texttt{a\_id} in \texttt{Applicants}.

\subsection{Provenance}

\emph{Data provenance} tracks where each row in a derived dataset came from. When a query joins and aggregates multiple source tables to produce a flattened ML-ready table, provenance records, for each output row, which source tuples contributed to it. In Example~\ref{ex:motivating}, the row with $\texttt{age}=35$ in the flattened table was derived from applicant \texttt{a27} and transactions \texttt{t67}, \texttt{t68}, \texttt{t69}.

Provenance is useful because it recovers structural information that flattening discards. If we know that a particular output row came from a single source applicant, we can apply the source table's constraints (e.g., the schema-level FD $\texttt{a\_id} \rightarrow \texttt{empl}$) to further restrict which feature completions are valid during Shapley computation.

Provenance tracking is a well-studied area in the database community~\cite{green2007provenance, cheney2009provenance}. \RelShap uses a lightweight form of provenance that traces flattened rows back to their source tuples via identifier-induced FDs that are no longer explicit in the feature set (Section~\ref{sec:relshap}).

\subsection{What flattening loses}

ML pipelines typically receive only the final flattened table. This flattening discards several kinds of structural information:

\begin{itemize}[leftmargin=1.5em]
    \item \textbf{FDs} that were enforced by the schema or induced by the query (e.g., $\texttt{age} \rightarrow \texttt{life\_stage}$) become invisible; standard Shapley estimators freely generate combinations like $(\texttt{age}=35, \texttt{life\_stage}=\texttt{older})$ that violate them.
    \item \textbf{Domain constraints} such as valid value ranges or conditional restrictions are no longer enforced.
    \item \textbf{Denial constraints} such as one-hot mutual exclusivity are not represented in the flat schema.
    \item \textbf{Provenance} information linking each row to its source tuples is lost when identifiers are dropped.
\end{itemize}

\noindent
\RelShap recovers this information, either from the original relational schema and query (when available) or by discovering constraints directly from the data. Section~\ref{sec:relshap} describes the extraction procedure; the three sources of constraints (schema-level, query-level, and data-driven) are detailed in Section~\ref{sec:extracting_constraints}.

\section{Details about background data selection} \label{app:background_data}

\subsection{Background data selection methods}
\label{app:details_background}

Since $f$ is defined on the full feature space $\mathbb{R}^{|F|}$, the entries of $\mathbf{x}_S$ corresponding to $j \in F \setminus S$ must be filled in. Let $\mathcal{P}_{\mathrm{bg}}(\mathbf{z}_{F\setminus S})$ denote a background distribution over the absent features indexed by $F \setminus S$, let $B$ be the number of background samples (possibly all reference data points), and $\mathbf{z}_{F\setminus S}^{(1)}, \ldots, \mathbf{z}_{F\setminus S}^{(B)} \sim \mathcal{P}_{\mathrm{bg}}.$ For a coalition $S \subseteq F$, we
construct a completed input $\tilde{\mathbf{x}}(S, \mathbf{z}_{F\setminus S}) \in \mathbb{R}^{|F|}$ by
\[
\tilde{x}_j(S, \mathbf{z}_{F\setminus S}) =
\begin{cases}
x_j, & \text{if } j \in S, \\
z_j, & \text{if } j \in F \setminus S,
\end{cases}
\]
and define
\[
f(\mathbf{x}_S) =
\mathbb{E}_{\mathbf{z}_{F\setminus S} \sim \mathcal{P}_{\mathrm{bg}}}
\Bigl[
f\bigl(\tilde{\mathbf{x}}(S, \mathbf{z}_{F\setminus S})\bigr)
\Bigr].
\]
Methods differ in their choice of $\mathcal{P}_{\mathrm{bg}}$. We describe the four most representative instantiations below, and discuss the conceptual landscape they span at the end of this section.\\

\emph{Interventional approach~\cite{janzing2020feature, zern2023interventional}} treats feature absence as an intervention rather than as a conditioning event. Under the standard assumption of feature independence, the distribution of the remaining features unchanged, so $\mathcal{P}_{\mathrm{bg}}$ reduces to the empirical \emph{marginal} distribution of the training data, $\mathcal{P}_{\mathrm{bg}}^{\mathrm{Marginal}} = \pi_{F \setminus S}(\mathcal{D}_{\mathrm{train}})$; we refer to this construction as the \emph{marginal} approach throughout the paper. Default Kernel SHAP~\cite{lundbergL2017unified} is its most widely deployed implementation. Under this choice, dependencies between observed features in $\mathbf{x}$ and the imputed features in $\mathbf{z}$ are broken: the completed input $\tilde{\mathbf{x}}(S, \mathbf{z})$ may lie off the data manifold, although it may coincidentally fall on it.\\

\emph{Conditional approach~\cite{aas2021explaining}} sets $\mathcal{P}_{\mathrm{bg}}$ to the conditional distribution of the absent features given the observed ones, estimated from $\mathcal{D}_{\mathrm{train}}$ via, \eg a multivariate Gaussian assumption, a Gaussian copula, or an empirical conditional, $\mathcal{P}_{\mathrm{bg}}^{\mathrm{Conditional}} = \widehat{P}\bigl(\mathbf{Z}_{F\setminus S} \mid \mathbf{X}_S = \mathbf{x}_S \bigr)$. The estimated distribution may not faithfully recover the true underlying one and can introduce spurious dependencies, but the formulation aims to capture inter-feature dependencies and is referred to as \emph{observational}.\\

\emph{Causal approach~\cite{heskes2020causal}} extends the interventional formulation by exploiting a known causal graph among the features: $\mathcal{P}_{\mathrm{bg}}^{\mathrm{Causal}} = \widehat{P}\bigl(\mathbf{Z}_{F\setminus S} \mid do(\mathbf{X}_S=\mathbf{x}_S) \bigr)$, where Pearl's $do(\cdot)$ propagates the intervention through the assumed causal mechanisms rather than (as in the marginal case) leaving $\mathbf{Z}_{F\setminus S}$ unchanged. Two cases arise depending on the structure: \emph{cycle} (bidirectional, similar to the conditional approach) and \emph{chain} (unidirectional).\\

\emph{On-manifold approach~\cite{taufiq2023manifold}} does not specify a probabilistic form for $\mathcal{P}_{\mathrm{bg}}$. Instead, it restricts attention to a region of interest $Z$ that approximates the geometric data manifold and excludes off-manifold samples: $\mathcal{P}_{\mathrm{bg}}^{\mathrm{Manifold}} = \widehat{P}\bigl(\mathbf{Z}_{F\setminus S} \mid do(\mathbf{X}_S=\mathbf{x}_S),\; \mathbf{X}\in Z \bigr)$. Membership in $Z$ is determined by a surrogate manifold classifier $g$ (\eg via kernel density estimation), and off-manifold samples are generated through perturbation; the resulting explanations may depend on the quality of the classifier. ManifoldShap~\cite{taufiq2023manifold} is the canonical instance.\\

\subsection{Perspectives on background data selection} \label{app:perspectives_bg_selection}
The four formulations above span the \emph{true to the model} vs.\ \emph{true to the data} axis introduced in Section~\ref{sec:intro}: the marginal construction (default Kernel SHAP) stays \emph{true to the model} by ignoring inter-feature dependencies, while the conditional, causal, and on-manifold approaches each attempt, in different ways, to stay \emph{true to the data} by capturing or respecting such dependencies. \RelShap is orthogonal to this axis: rather than choosing \emph{what distribution} to put on the absent features, it constrains \emph{which configurations} the chosen distribution is allowed to place mass on. Given any base background distribution $\mathcal{P}_{\mathrm{bg}}^{\mathrm{base}}$ (marginal, conditional, causal, or manifold), \RelShap restricts its support to completions that satisfy the relational constraints $\Sigma$ extracted from the schema, query, and data (Def.~\ref{def:relshap:bg}). The two design dimensions therefore compose rather than conflict: \RelShap can be layered on top of any of the four formulations above.

By default we use the marginal (Kernel SHAP~\cite{lundbergL2017unified}) as the base, since it requires no distributional assumptions and is widely used; relational restriction then rules out completions that are infeasible under the schema and query (\eg $\mathtt{age}=35$ paired with $\mathtt{life\_stage}=\mathtt{older}$ when the FD $\mathtt{age}\to\mathtt{life\_stage}$ has been identified) without estimating any conditional density. This restriction preserves the no-distributional assumptions property advocated by \citet{sundararajan2020many} and \citet{janzing2020feature} while removing the structural off-manifold inputs \citet{fryeshapley} warns about. Appendix~\ref{app:true_data_model} formalizes this on a two-feature linear model.

\subsection{True to the data vs. True to the model} \label{app:true_data_model}

Let
\[
X_1 \sim \mathrm{Unif}\{1,2,\dots,8\},
\;\;
X_2 = g(X_1) = \left\lceil \frac{X_1}{2}\right\rceil .
\]
with
\[
\mathbb{E}[X_1]=4.5,
\;\;
\mathbb{E}[X_2]=2.5.
\]
Then fix a target point
\[
x=(x_1,x_2)=(k,b),
\;\;
b=\left\lceil \frac{k}{2}\right\rceil .
\]

\textbf{Kernel SHAP.}  
Absent features are filled independently from their marginal distributions by ignoring dependencies present in the data, which follows a \emph{true to the model} perspective.

\textbf{One-way (unidirectional) FD-aware ($X_1 \to X_2$).}  
If $X_1$ is observed, $X_2$ is deterministically recovered via $g$, following a \emph{true to the data} viewpoint.  
If only $X_2$ is observed, we assume no information about $X_1$ and fall back to the Kernel SHAP treatment (\emph{true to the model}). It aligns with the \RelShap approach when only the functional dependency is known, and also corresponds to the \textit{chain} case in ~\cite{heskes2020causal}. 

\textbf{Bidirectionally aware ($X_1 \leftrightarrow X_2$).}
This aligns with the \textit{cycle} case described in~\cite{heskes2020causal}, reflecting \emph{true to the data} perspective where observing either feature determines the other. In this example, knowing the function $g$ alone is not sufficient, since $g$ is not injective (many-to-one), and therefore an inverse function $g^{-1}$ does not exist. The preimage set
$g^{-1}(x_2)=\{x_1 \mid g(x_1)=x_2\}$
indicates that observing $X_2$ only reveals a range of possible values for $X_1$, rather than its exact value. Therefore, in this setting we assume that the dataset is fully known, i.e., the exact pairs $(X_1, X_2)$ are available.

Since the goal of this example is to illustrate how feature dependence affects attribution, we assume a simple linear model
\[
f(x_1,x_2)=\alpha x_1+\beta x_2,
\] where $\alpha, \beta \in \mathbb{R}$, and intentionally exclude nonlinear interactions or the absence of one of the features (i.e., single-feature settings discussed in~\cite{heskes2020causal}), so that the analysis focuses only on the presence of feature dependence.
In all three cases,

\[
f_{\emptyset}
=
\mathbb{E}[f(X_1,X_2)]
=
\alpha\mathbb{E}[X_1]+\beta\mathbb{E}[X_2]
=
4.5\alpha+2.5\beta.
\]
\[
f_{\{1,2\}}
=
f(k,b)
=
\alpha k+\beta b.
\]

\textbf{Kernel SHAP.}

\[
f_{\{1\}}
=
\mathbb{E}[f(k,X_2)]
=
\alpha k+\beta\mathbb{E}[X_2]
=
\alpha k+2.5\beta
\]
\[
f_{\{2\}}
=
\mathbb{E}[f(X_1,b)]
=
\alpha\mathbb{E}[X_1]+\beta b
=
4.5\alpha+\beta b
\]

The resulting Shapley values:

\[
\phi_1^{\mathrm{Marginal}}
=
\frac12(f_{\{1\}}-f_{\emptyset})
+
\frac12(f_{\{1,2\}}-f_{\{2\}})
=
\alpha(k-4.5)
\]
\[
\phi_2^{\mathrm{Marginal}}
=
\frac12(f_{\{2\}}-f_{\emptyset})
+
\frac12(f_{\{1,2\}}-f_{\{1\}})
=
\beta(b-2.5)
\]

\textbf{$X_1 \to X_2$.}

If $X_1$ is known, then $X_2$ is determined:
\[
f_{\{1\}}
=
f(k,g(k))
=
f(k,b)
=
\alpha k+\beta b
\]

If only $X_2$ is known, $X_1$ cannot be recovered since the mapping $g$ is not assumed to be known (otherwise a piecewise inverse could be specified). Thus the marginal distribution of $X_1$ is used,
\[
f_{\{2\}}
=
\mathbb{E}[f(X_1,b)]
=
\alpha\mathbb{E}[X_1]+\beta b
=
4.5\alpha+\beta b
\]

The resulting Shapley values:

\[
\phi_1^{\mathrm{FD}}
=
\frac12(f_{\{1\}}-f_{\emptyset})
+
\frac12(f_{\{1,2\}}-f_{\{2\}})
\]
\[
=
\frac12[(\alpha k+\beta b)-(4.5\alpha+2.5\beta)]
+
\frac12[(\alpha k+\beta b)-(4.5\alpha+\beta b)]
\]
\[
=
\alpha(k-4.5)+\frac{\beta}{2}(b-2.5)
\]
\[
\phi_2^{\mathrm{FD}}
=
\frac12(f_{\{2\}}-f_{\emptyset})
+
\frac12(f_{\{1,2\}}-f_{\{1\}})
\]
\[
=
\frac12[(4.5\alpha+\beta b)-(4.5\alpha+2.5\beta)]
+
\frac12[(\alpha k+\beta b)-(\alpha k+\beta b)]
\]
\[
=
\frac{\beta}{2}(b-2.5)
\]

\textbf{Difference from the marginal case.}

\[
\phi_1^{\mathrm{FD}}-\phi_1^{\mathrm{Marginal}}
=
\frac{\beta}{2}(b-2.5),
\]
\[
\phi_2^{\mathrm{FD}}-\phi_2^{\mathrm{Marginal}}
=
-\frac{\beta}{2}(b-2.5).
\]

Thus part of the attribution that would otherwise be assigned to the
derived feature $X_2$ is transferred to the root feature $X_1$ when the
dependency $X_1 \to X_2$ is respected. The exact magnitude of this effect depends on the model coefficient $\beta$ on $X_2$, and on the deviation of $b$ from its mean value $\mathbb{E}[X_2] = 2.5$.

\textbf{$X_1 \leftrightarrow X_2$.}
\[
f_{\{1\}}=f_{\{2\}}=f_{\{1,2\}}=\alpha k+\beta b
\]

The resulting Shapley values:
\[
\phi_1 = \phi_2
=
\frac12(f_{\{1,2\}}-f_{\emptyset})
=
\frac12(\alpha k+\beta b-4.5\alpha-2.5\beta)
\]

\textbf{Discussion.}

The three approaches produce qualitatively different attributions since they incorporate the dependency structure to different extents, even when all other conditions held fixed. The marginal approach ignores the relation between the variables and treats $X_1$ and $X_2$ as independently manipulable. As a result, the Shapley values are allocated according to the model coefficients, with only the mean subtracted from each feature. For $X_1 \leftrightarrow X_2$ case, the two variables become informationally indistinguishable, making the Shapley values split evenly between the two features. 
Overall, the resulting attributions still depend on the model coefficient, reflecting the \emph{true to the model} perspective. However, by incorporating the functional dependency structure ($X_1 \to X_2$), part of the attribution that would otherwise be assigned to the derived attribute $X_2$ is transferred to the root attribute $X_1$. This reflects a \emph{true to the data} perspective, where the explanation respects the dependency structure of the underlying data. Importantly, this simplified redistribution arises since we considered only two features with a single FD between them. As the number of features or FDs increases, the reallocation is governed by the joint dependency structure, making the resulting Shapley value more complex and less directly traceable.

\section{Details about coalition samplers}
\label{app:coalition_samplers}

In this section, we restate the formulations used in existing works using our notation and describe their coalition sampling procedures based on their code implementations. This clarification is necessary for our theoretical development in Section~\ref{sec:algorithms}.

\subsection{Kernel SHAP} \label{app:kernelshap}

\emph{Kernel SHAP~\cite{lundbergL2017unified}} estimates Shapley values by fitting a weighted linear surrogate whose solution coincides with $\phi(f)$. Each coalition is weighted by the \emph{Shapley kernel},
\[
w_{\mathrm{Kernel}}(S) = \frac{|F|-1}{\binom{|F|}{|S|}\, |S|\,(|F|-|S|)},
\]
which depends only on $|S|$ and concentrates weight on extreme (very small or very large) coalition sizes; coalitions are sampled accordingly. A coalition is a subset $S \subseteq F$, equivalently represented by a binary mask $z' \in \{0,1\}^{|F|}$ where $|F|$ denote the number of features considered in the explanation, and let $S(z')=\{j \in F \mid z'_j=1\}$.

Coalitions are sampled using a hybrid strategy. First, a subset size $s\in\{1,\ldots,|F|-1\}$ is drawn, excluding the trivial coalitions corresponding to the empty and full sets. These trivial cases are handled separately and computed once since there is only one coalition of each type, whereas intermediate sizes contain multiple candidates.

For non-trivial coalitions, Kernel SHAP defines a coalition sampling distribution
$\mathcal{P}^{\mathrm{Kernel}}_{\mathrm{coal}}$ by first allocating probability mass across subset sizes according to
\[
\mathcal{P}^{\mathrm{Kernel}}_{\mathrm{coal}}(|S|=s)\;\propto\;\frac{1}{s(|F|-s)}.
\]
Once a subset size $s$ is determined, a coalition $S\subseteq F$ with $|S|=s$ is sampled uniformly at random among all subsets of that size,
i.e.,
\[
\mathcal{P}^{\mathrm{Kernel}}_{\mathrm{coal}}(S\mid |S|=s)=\frac{1}{\binom{|F|}{s}}.
\]
Consequently, the overall coalition sampling distribution satisfies
\[
\mathcal{P}^{\mathrm{Kernel}}_{\mathrm{coal}}(S)
\;\propto\;
\frac{1}{\binom{|F|}{|S|}\,|S|(|F|-|S|)},
\]
which matches the size-dependent structure induced by the Shapley kernel stated below.

In practice, when the number of features is sufficiently small, Kernel SHAP deterministically enumerates all coalitions of a given size whenever the sampling budget permits. Concretely, if the requested number of samples exceeds the number of possible non-trivial coalitions (i.e., $2^{|F|}-2$), the algorithm simply enumerates the entire coalition space instead of performing random sampling. The randomized sampling
scheme described above is therefore used only when the feature dimension is large enough that full enumeration becomes infeasible.

During random sampling with replacement, duplicate coalitions may occasionally occur. Although the number of possible coalitions grows exponentially with the number of features-making such collisions relatively unlikely in practice-Kernel SHAP handles them explicitly. When a coalition that has already been observed is drawn again, the algorithm does not add a new row to the regression design matrix. Instead, it increases the weight associated with the existing coalition so that the resulting weighted regression correctly reflects the number of times that coalition was sampled.

Given a coalition budget $M$, we sample coalitions $S^{(1)},\ldots,S^{(M)} \sim \mathcal{P}_{\mathrm{coal}}^{\mathrm{Kernel}}$,
Kernel SHAP assigns each coalition a regression weight determined by the Shapley kernel:
\[
w_{\mathrm{Kernel}}\!\left(S^{(m)}\right)
\;\propto\;
\pi\bigl(|F|,|S^{(m)}|\bigr),
\]
where the Shapley kernel is a size-dependent function defined for nontrivial coalitions ($0<|S(z')|<|F|$) as:
\[
\pi(z')
=
\pi\bigl(|F|,|S(z')|\bigr)
=
\frac{|F|-1}{\binom{|F|}{|S(z')|}\,|S(z')|\,(|F|-|S(z')|)}.
\]

Then Kernel SHAP fits a linear surrogate $g(z') = \phi_0 + \sum_{j \in F} \phi_j z'_j$ by minimizing the weighted squared loss
\[
\min_{\phi_0,\{\phi_j\}_{j\in F}}
\sum_{m=1}^M
w\bigl(S^{(m)}\bigr)\,
\Bigl(
f\bigl(\tilde{\mathbf{x}}(S^{(m)})\bigr)
- \phi_0 - \sum_{j\in F} \phi_j z'^{(m)}_j
\Bigr)^2,
\]
subject to the two endpoint constraints induced by $\pi(z')=\infty$:
\[
g(\mathbf{0}) = f(\tilde{\mathbf{x}}(\emptyset))
\quad\text{and}\quad
g(\mathbf{1}) = f(\tilde{\mathbf{x}}(F)),
\]
which are equivalent to
$\phi_0 = f(\tilde{\mathbf{x}}(\emptyset))$ and
$\sum_{j\in F}\phi_j = f(\tilde{\mathbf{x}}(F)) - f(\tilde{\mathbf{x}}(\emptyset))$.
The resulting coefficients $\{\phi_j\}_{j\in F}$ are the Kernel SHAP attributions.

\subsection{Monte Carlo coalition estimator}
\label{app:mc}

\emph{Monte Carlo (MC) estimator} fixes a feature of interest $i \in F$ and samples coalitions $S \subseteq F \setminus \{i\}$ from the Shapley-weighted distribution
\[
\mathcal{P}_{\mathrm{coal}}^{\mathrm{MC}}(S \mid i) = \frac{|S|!(|F|-|S|-1)!}{|F|!}, \qquad w_{\mathrm{MC}}(S) = 1.
\]
Since $\mathcal{P}_{\mathrm{coal}}^{\mathrm{MC}}$ matches the coefficients in Eq.~\eqref{eq:eq1}, the resulting estimator is unbiased for $\phi_i(f)$. It is conceptually simple but may require many samples to converge. In practice, coalition sizes are sampled approximately uniformly. Each Monte Carlo draw first samples a pivot (a feature to explain) $i \sim \mathrm{Unif}(F).$ Conditioned on $i$, the estimator samples coalitions 
$S \subseteq F \setminus \{i\}$ by first sampling a coalition size
\[
K \sim \mathrm{Unif}(\{0,1,\ldots,|F|-1\}),
\]
and then sampling a subset $S$ uniformly at random among all size-$K$ subsets of
$F \setminus \{i\}$. This two-stage procedure induces a distribution over coalitions
$S \subseteq F \setminus \{i\}$.
For any fixed $i$ and coalition $S \subseteq F \setminus \{i\}$, the probability that $S$
is sampled is
\[
\begin{aligned}
\mathbb{P}(S \mid i)
&=
\mathbb{P}(|S|=k \mid i)\cdot\mathbb{P}(S \mid |S|=k,i) \\
&=
\frac{1}{|F|}\cdot\frac{1}{\binom{|F|-1}{|S|}}
=
\frac{|S|!(|F|-|S|-1)!}{|F|!},
\end{aligned}
\]
which exactly matches the Shapley coefficient in Eq.~\eqref{eq:eq1}.
Each Monte Carlo draw samples a pivot uniformly, a coalition size uniformly, and a subset uniformly within that size, inducing a Shapley-weighted distribution over coalitions.

\subsection{Leverage SHAP}
\label{app:leverage}

\emph{Leverage SHAP~\cite{muscoprovably}} solves the same regression formulation as Kernel SHAP but adopts a \emph{leverage score}-based sampling strategy with provable accuracy guarantees. It first selects a coalition size with approximately equal probability and then samples a coalition uniformly at random among all coalitions of that size, so all sizes appear with comparable frequency rather than the extremes being overrepresented. The regression weight (the leverage score) depends only on $|S|$. 
Let $Z \in \mathbb{R}^{\rho \times |F|}$ and $y \in \mathbb{R}^{\rho}$ denote the full regression
system (as in Kernel SHAP) over all coalitions with $0<|S|<|F|$ (which brings $\rho = 2^{|F|}-2$),
where each row corresponds to a coalition indicator vector. To enforce the efficiency constraint,
define the projection matrix $P = I - \frac{1}{|F|}\mathbf{1}\mathbf{1}^\top$ and consider the
unconstrained system $A = ZP$ and
$b = y - Z\mathbf{1}\,\frac{v(\mathbf{1})-v(\mathbf{0})}{|F|}$.

Leverage SHAP samples coalitions according to a distribution
$\mathcal{P}_{\mathrm{coal}}^{\mathrm{Leverage}}$ over $0<|S|<|F|$.
In this problem, the leverage scores admit a closed form: all coalitions of the same size share the same score, and the total sampling mass is allocated by coalition size.
Consequently, the sampling design is stratified by coalition size. Let $m_s$ denote the number of sampled coalitions of size $s$ among the $M$ draws with $\sum_{s=1}^{|F|-1} m_s = M$ for a given budget $M$. Within each stratum corresponding to coalition size $s$, $m_s$ coalitions are sampled uniformly without replacement from the $\binom{|F|}{s}$ possible coalitions of that size. Then every coalition $S$ with $|S|=s$ has the same inclusion probability
\[
\mathcal{P}_{\mathrm{coal}}^{\mathrm{Leverage}}(S)
= p(S)
=
\frac{m_{|S|}}{\binom{|F|}{|S|}}.
\]

Given the sampled coalitions $S^{(1)},\ldots,S^{(M)}$
generated by this stratified sampling design, the estimator solves a reweighted least-squares problem with per-sample weight
\[
w_{\mathrm{Leverage}}\!\left(S^{(m)}\right)
\propto
\frac{\pi\!\left(|S^{(m)}|\right)}
{\mathcal{P}_{\mathrm{coal}}^{\mathrm{Leverage}}\!\left(S^{(m)}\right)},
\]
where $\pi(\cdot)$ denotes the Shapley kernel defined in Appendix~\ref{app:kernelshap}.
The estimator then solves
{\scriptsize\[
\hat{\phi}_{\perp} =
\arg\min_{\beta}
\sum_{m=1}^{M}
w_{\mathrm{Leverage}}\!\left(S^{(m)}\right)
\Bigl(
\langle a(S^{(m)}), \beta \rangle - b(S^{(m)})
\Bigr)^2,
\]}
where $a(S)$ and $b(S)$ denote the corresponding row of $A$ and entry of $b$.
Finally, it returns
\[
\hat{\phi}
=
\hat{\phi}_{\perp}
+
\frac{v(\mathbf{1})-v(\mathbf{0})}{|F|}\mathbf{1},
\]
which satisfies the efficiency constraint by construction.
In practice, ridge regularization may be added if the subsampled normal matrix is ill-conditioned.

\subsection{Perspectives on coalition selection} \label{app:perspectives_coalition_selection}

The three estimators above span a design space defined by bias, variance, and accuracy guarantees: Kernel SHAP is faster but biased toward extreme coalition sizes; MC is unbiased but slow to converge; Leverage SHAP equalizes coverage across sizes and offers provable accuracy guarantees. Recent estimators improve this space along orthogonal directions: \citet{witterregression} reduce variance via residual estimation against a learned surrogate, while PolySHAP~\cite{fumagalli2026polyshap} and OddSHAP~\cite{fumagalli2026odd} use higher-order polynomial regression and Fourier-basis reduction, respectively, to improve accuracy and efficiency. 

\RelShap is orthogonal to all of these: rather than introducing a new sampling distribution or weighting scheme, it acts as a plug-in restriction layer that prunes the coalition space to relationally valid configurations before any estimator is applied (Def.~\ref{def:relshap:bg}). It therefore composes with any of these estimators without modifying their sampling distributions, weights, and accuracy guarantees, and contributes \emph{orthogonal acceleration} by collapsing coalitions that are equivalent under the relational constraints (Section~\ref{sec:relshap}). The most directly comparable approach is that of \citet{witter2026exact}, which likewise reduce the coalition space via equivalence classes, but using structural causal models that require explicit causal assumptions and intervention semantics; \RelShap derives its equivalence structure from declared schema constraints and observed FDs in the data, with no causal assumption required.

\section{Details about relational constraints} \label{app:details_relational_constraints}

\subsection{Data-level relational constraints} \label{app:relational_constraints}

\subsubsection{Discovered functional dependencies}
\label{app:fd_discovery}

For any $L \subseteq F$ and attribute $a \notin L$, let $\mathcal{D}/L$ denote the partition of $\mathcal{D}$ induced by equality on $L$. We say that $L \rightarrow a$ holds \emph{exactly} in $\mathcal{D}$ if, for every equivalence class $C \in \mathcal{D}/L$, attribute $a$ takes at most one distinct value within~$C$:
\[
\forall\, C \in \mathcal{D}/L,\;
\bigl|\mathrm{DISTINCT}_a(C)\bigr| \le 1,
\]
where $\mathrm{DISTINCT}_a(C)$ denotes the set of distinct values of attribute $a$ within group $C$.

We restrict the left-hand side to at most a user-specified bound---default~$2$, except for UW-CSE ($3$)---and consider only single-attribute right-hand sides; multi-attribute right-hand sides can be derived via Armstrong's axioms~\cite{armstrong1974dependency} during background data selection (Section~\ref{sec:relshap:bg_coal_main}).  We retain only minimal FDs: for a fixed $a$, if $L \rightarrow a$ holds and there exists $L' \subset L$ such that $L' \rightarrow a$ also holds, then $L \rightarrow a$ is discarded.  FDs discovered at this stage may overlap with those extracted at the schema and query levels; we retain all and let the user decide which levels to include.

\paragraph{Approximate FDs.} To tolerate limited violations, we additionally support approximate FDs.  For a candidate pair $(L,a)$, the approximation error is
\[
\mathrm{err}(L \rightarrow a)
=
\frac{\bigl|\bigl\{
C \in \mathcal{D}/L
\;\big|\;
\bigl|\mathrm{DISTINCT}_a(C)\bigr| > 1
\bigr\}\bigr|}
{|\mathcal{D}/L|}.
\]
We record $L \rightarrow a$ as an approximate FD if $\mathrm{err}(L \rightarrow a) \le \varepsilon$, where $\varepsilon$ is a user-defined tolerance threshold. While our implementation supports approximate FDs, we restrict all experiments to exact FDs.

\subsubsection{Conditional domain rules}
\label{app:domain_rules}

For each discovered FD $A \rightarrow B$ where $\mathrm{dom}(B)$ has bounded cardinality (at most a user-defined threshold, default~$20$), we iterate over values $s \in \mathrm{dom}(B)$ and construct implications of the form
\[
(B = s) \Rightarrow \varphi(A),
\]
where $\varphi(A)$ is a domain restriction on $A$ estimated from the tuples in $\mathcal{D}$ satisfying $B = s$.  Two cases arise:

\begin{itemize}
\item \textbf{Categorical~$A$.} $\varphi(A)$ restricts $A$ to the set of distinct values co-occurring with $B = s$ in the data.
\item \textbf{Continuous~$A$.} $\varphi(A)$ restricts $A$ to an empirical interval $[\ell_s, u_s]$, where $\ell_s$ and $u_s$ are computed either from the minimum and maximum values or from user-specified lower and
upper quantiles to reduce sensitivity to outliers.
\end{itemize}

\noindent
For example, from an FD $\texttt{age} \rightarrow \texttt{life\_stage}$, suppose that the observed domain of \texttt{life\_stage} in $\mathcal{D}$ is $\mathrm{dom}(\texttt{life\_stage}) = \{\texttt{young},\texttt{middle},\texttt{older}\}$. For each $s \in \mathrm{dom}(\texttt{life\_stage})$, we construct
\[
(\texttt{life\_stage}=s) \Rightarrow
\bigl(\ell_s \le \texttt{age} \le u_s \bigr),
\]
where $\ell_s$ and $u_s$ are estimated empirically from tuples in $\mathcal{D}$ with $\texttt{life\_stage}=s$.

\subsubsection{Denial constraints}
\label{app:denial_constraints}

We identify denial constraints via cycles in the exact FD set. For a set of attributes $\{A,B,C\}$, if $A \rightarrow B$, $B \rightarrow C$, and $C \rightarrow A$ all hold in $\mathcal{D}$, the attributes are mutually determined and form a cyclic dependency. For each cycle, we test for two characteristic patterns:

\begin{itemize}
\item \textbf{Scaled one-hot.} Each entry takes values in $\{0,\kappa\}$ for some $\kappa \in \mathbb{Z}_{>0}$, and in every tuple at most one column in the set is nonzero.
\item \textbf{$n$-ary exclusive or (XOR).} All entries are in $\{0,1\}$, and the sum over the $c$ columns has
a fixed parity modulo~$2$.
\end{itemize}

\noindent
\textbf{Caveat.} A categorical variable could be manually one-hot encoded by the user to avoid imposing an artificial ordinal structure (\eg encoding categories as 1, 2, 3) on inherently unordered categories.  Suppose that the model is trained directly on these columns without integrating the preprocessing step into a pipeline
using \texttt{ColumnTransformer} from the \texttt{scikit-learn} package.  In this case, Shapley value computations without background data selection in RelShap may produce infeasible inputs, since $f$ refers only to the prediction model and does not include the preprocessing transformation.  For example, under the standard
Kernel SHAP procedure~\cite{lundbergL2017unified}, if we explain an instance with $\texttt{gender\_male}=1$ and consider a female background data point, the completed input becomes
$($\texttt{gender\_male},\texttt{gender\_female},
\texttt{gender\_nonbinary}$)=(1, 1, 0)$,
which does not correspond to any valid \texttt{gender} category. Such combinations violate the implicit one-hot constraint but are not prevented by conventional SHAP implementations.  This motivates reflecting relational constraints, especially denial constraints, during background data selection.

\subsection{On the choice of reference data} \label{app:split}

We first clarify that including test data in FD discovery or lookup table (LUT) construction does not lead to a data leakage. These steps are performed solely during the Shapley value computation and are independent of model training. Referencing test data in the post-hoc explanation phase does not alter the trained model parameters or influence predictive performance.

Second, the choice of reference data itself has negligible impact when the training and full data distributions remain aligned. Prior work such as Fool SHAP~\cite{laberge2022fool} shows that SHAP explanations can be substantially distorted when the background data is adversarially manipulated or heavily skewed. Thus, the key issue is not which reference dataset is selected - $D_{\text{train}}$ or $D$ - but whether the reference distribution exhibits substantial shift from the model’s training distribution. In this work, we do not consider significant distributional shifts, and thus such choice does not materially affect the explanations and is treated as a design choice.

Existing approaches including Kernel SHAP~\cite{lundbergL2017unified} and the conditional approach~\cite{aas2021explaining}, typically use only $D_{\mathrm{train}}$ as the background data. While the intention is to respect the training data distribution on which the model was fitted, explanations are ultimately evaluated on test instances. Consequently, absent features are filled using reference data points solely from $D_{\mathrm{train}}$, which may introduce unrealistic or distributionally mismatched combinations, making such replacements semantically ambiguous.

Table~\ref{tab:4cases} summarizes the four possible combinations of these choices and their potential implications, which may or may not arise in practice depending on the specific configuration. 
In our experiments, we use $\mathcal{D}$ for discovery and $\mathcal{D}_{\mathrm{train}}$ for LUT construction by default.
We evaluate each combination along three properties: \emph{Mismatch} indicates whether the discovered relational constraints may not be realizable using values available in the LUT, potentially yielding invalid or undefined instantiations during background or coalition construction for test data points. It occurs only when $\mathrm{Discovery} \nsubseteq \mathrm{LUT}$.
\emph{Hit} indicates whether all values required to instantiate missing features of a test instance are present in the LUT, so that background data points can be constructed without failure. In particular, this is guaranteed when the LUT is constructed from the full dataset $\mathcal{D}$.
\emph{Test} indicates whether test data are used to construct either the discovered constraints or the LUT.

\begin{table}[t]
\caption{Four possible combinations of discovery data for data-driven constraint extraction and lookup table (LUT) for background data selection, evaluated in terms of possible mismatch, hit (coverage), and the use of test data. \cmark \ indicates that the corresponding phenomenon may occur, while \xmark \ denotes that it cannot occur. For mismatch and test, non-occurrence is preferable, whereas for hit, occurrence is desirable.}
\label{tab:4cases}
\centering
\begin{tabular}{c|ccc}
\hline
Discovery--LUT
& Mismatch
& Hit
& Test \\
\hline
$\mathcal{D}_{\mathrm{train}}$--$\mathcal{D}_{\mathrm{train}}$
& \xmark
& \xmark
& \xmark \\

$\mathcal{D}_{\mathrm{train}}$--$\mathcal{D}$
& \xmark
& \cmark
& \cmark \\

$\mathcal{D}$--$\mathcal{D}_{\mathrm{train}}$
& \cmark
& \cmark
& \cmark \\

$\mathcal{D}$--$\mathcal{D}$
& \xmark
& \cmark
& \cmark \\
\hline
\end{tabular}
\end{table}

\section{Additional mathematical details} \label{app:add_math_details}

\begin{definition}[$\Sigma$-coalition equivalence] \label{def:relshap:equiv}
Under the relationally consistent background semantics in Def.~\ref{def:relshap:bg}, we define an equivalence relation over coalitions. Given $\Sigma$ and a feature of interest $i \in F$, for any coalitions $S, S' \subseteq F \setminus \{i\}$, $S$ and $S'$ are \emph{$\Sigma$-equivalent}, i.e.,
$S \sim_{\Sigma} S',$
if they induce the same set of relationally valid completions under $\Sigma$.
Formally,
{\scriptsize
\[
\bigl\{
\tilde{\mathbf{x}}(S,\mathbf{z}_{F\setminus S})
\,\big|\,
\tilde{\mathbf{x}}(S,\mathbf{z}_{F\setminus S}) \models \Sigma
\bigr\}
=
\bigl\{
\tilde{\mathbf{x}}(S',\mathbf{z}_{F\setminus S'})
\,\big|\,
\tilde{\mathbf{x}}(S',\mathbf{z}_{F\setminus S'}) \models \Sigma
\bigr\}.
\]}
\end{definition}

\begin{definition}[$\Sigma$-coalition equivalence class]\label{def:relshap:equiv_class}
Following Def.~\ref{def:relshap:equiv}, for any coalition $S \subseteq F \setminus \{i\}$, 
its \emph{$\Sigma$-equivalence class} is defined as:
\[
[S]_{\Sigma}
= \{ T \subseteq F \setminus \{i\} \mid T \sim_{\Sigma} S \}.
\]
\end{definition}

\begin{definition}[$\Sigma$-coalition space]\label{def:relshap:space}
\emph{$\Sigma$-coalition space} is the quotient space of the original coalition space $\{ S \subseteq F \setminus \{i\} \}$ under the equivalence relation $\sim_{\Sigma}$, which collapses coalitions that are indistinguishable under $\Sigma$. Formally,
\[
\mathcal{C}_{\Sigma}
=
\{ S \subseteq F \setminus \{i\} \} \big/ \sim_{\Sigma}
=
\{ [S]_{\Sigma} \mid S \subseteq F \setminus \{i\} \}.
\]
\end{definition}

As noted in Section~\ref{sec:relshap:bg_coal_main}, non-trivial $\sim_{\Sigma}$ implies $|\mathcal{C}_{\Sigma}| < 2^{|F|-1}$ in the deterministic setting, without considering randomness in coalition sampling; see Appendix~\ref{app:proof:coalition_reduction} for the proof.

\section{Proofs} \label{app:proof}

\subsection{Proof of Prop.~\ref{prop:quotient_invariance}}

\begin{proof}

By Def.~\ref{def:relshap:equiv}, for any $S,S' \subseteq F\setminus\{i\}$ with
$S \sim_{\Sigma} S'$,
$f\!\left(\tilde{\mathbf{x}}(S,\mathbf{z})\right)
=f\!\left(\tilde{\mathbf{x}}(S',\mathbf{z}')\right)$
for all relationally valid completions
$\mathbf{z},\mathbf{z}'$ under $\Sigma$.
Let $\{C_1,\dots,C_K\}$ be the partition of
$\{ S \subseteq F\setminus\{i\} \}$ induced by $\sim_{\Sigma}$. Starting from the Shapley value definition and grouping coalitions by their $\Sigma$-equivalence classes, we obtain:
\[
\begin{aligned}
\phi_i
&=
\sum_{k=1}^{K}
\sum_{S \in C_k}
w(S)\,
\Big(
f(\tilde{\mathbf{x}}(S\cup\{i\},\mathbf{z}))
-
f(\tilde{\mathbf{x}}(S,\mathbf{z}))
\Big) \\
&=
\sum_{k=1}^{K}
\left(
\sum_{S \in C_k} w(S)
\right)
\Big(
f(\tilde{\mathbf{x}}(\hat{S}_k\cup\{i\},\mathbf{z}))
-
f(\tilde{\mathbf{x}}(\hat{S}_k,\mathbf{z}))
\Big),
\end{aligned}
\]
where $\hat{S}_k$ is any representative of $C_k$.
Since the original Shapley sum can be partitioned according to $\Sigma$-equivalence classes and the inner expression is identical for all $S \in C_k$, the above expression coincides with the Shapley computation over $\mathcal{C}_{\Sigma}$.
\end{proof}

\subsection{Proof of Coalition Space Reduction} \label{app:proof:coalition_reduction}
Here we prove the claim in Section~\ref{sec:relshap:bg_coal_main} that non-trivial $\sim_{\Sigma}$ implies $|\mathcal{C}_{\Sigma}| < 2^{|F|-1}$.

\begin{proof}
The original coalition domain
$\{ S \subseteq F \setminus \{i\} \}$ has cardinality
$\lvert \{ S \subseteq F \setminus \{i\} \} \rvert = 2^{|F|-1}$. The $\Sigma$-coalition space $\mathcal{C}_{\Sigma}$ consists of the
$\sim_{\Sigma}$-equivalence classes over this domain, which form a partition;
thus each element of the domain belongs to exactly one equivalence class. If $\sim_{\Sigma}$ is non-trivial, then there exist distinct
$S, S' \subseteq F \setminus \{i\}$ such that $S \sim_{\Sigma} S'$. Hence at least one equivalence class contains at least two elements.
Since the equivalence classes partition the domain, having a block of size at least two implies that the number of blocks is strictly smaller than the
number of elements being partitioned. Therefore,
$\lvert \mathcal{C}_{\Sigma} \rvert \;<\; 2^{|F|-1}.$
\end{proof}

\section{Algorithms} \label{app:algo_details}

\begin{algorithm}[t]
\caption{\textsc{CanonicalizeCoal}}
\label{alg:canon_mask}
\begin{algorithmic}[1]
\State $S \gets \{ f_j \in F \mid m_j = 1 \}$
\If{$S$ in $\textsc{CanonMap}$}
    \State $C \gets \textsc{CanonMap}[S]$
\Else
    \State $C \gets \textsc{FDClosure}(S)$ \Comment{Appendix Alg.~\ref{alg:fdclosure}}
    \State $\textsc{CanonMap}[S] \gets C$
\EndIf
\State Set $m_j \gets 1$ for all $f_j \in C$
\State \Return $m$
\end{algorithmic}
\end{algorithm}

Each coalition mask sampled by the base estimator is \emph{canonicalized}: it is expanded to its FD-closed representative under $\Sigma_{\mathrm{FD}}$, so that $\Sigma$-equivalent coalitions map to the same canonical mask (Alg.~\ref{alg:canon_mask}). 
A coalition mask $m$ is a binary vector over the features $F$: $m_j = 1$ marks feature $f_j$ as observed (provided by the base sampler) and $m_j = 0$ marks it as unobserved (to be completed from the background data points).  Alg.~\ref{alg:canon_mask} maps each mask to its FD-closed form by computing $\textsc{FDClosure}(S)$ (Appendix~Alg.~\ref{alg:fdclosure}), where $S = \{f_j \mid m_j = 1\}$, and setting to $1$ every entry in the closure.  A memoization cache $\textsc{CanonMap}$ stores computed closures, so each distinct input coalition incurs the closure computation at most once.

\begin{algorithm}[ht]
\caption{\textsc{FDClosure}: Canonicalizing a coalition under the Armstrong's axioms}
\label{alg:fdclosure}
\begin{algorithmic}[1]
\State \textbf{Input:} coalition feature set $S$
\State \textbf{Output:} FD-closed coalition $C$

\State $C \gets S$; initialize an empty queue $Q$
\For{each FD rule $L \rightarrow r$ in $\Sigma_{\mathrm{FD}}$}
    \State $\tau_{L \rightarrow r} \gets |L \setminus C|$
    \Comment{\# of unsatisfied LHS attributes}
    \If{$\tau_{L \rightarrow r} = 0$ and $r \notin C$}
        \State $C \gets C \cup \{r\}$; enqueue $r$ into $Q$
    \EndIf
\EndFor

\While{$Q$ not empty}
    \State $\mathrm{{attr}_{new}} \gets$ dequeue from $Q$
    \For{each FD rule $L \rightarrow r$ with $\mathrm{{attr}_{new}} \in L$}
        \State $\tau_{L \rightarrow r} \gets \tau_{L \rightarrow r} - 1$
        \If{$\tau_{L \rightarrow r} = 0$ and $r \notin C$}
            \State $C \gets C \cup \{r\}$; enqueue $r$ into $Q$
        \EndIf
    \EndFor
\EndWhile

\State \Return $C$
\end{algorithmic}
\end{algorithm}

In this section, we provide the algorithmic and implementation details. In Alg.~\ref{alg:fdclosure}, for each FD rule $L \rightarrow r$, the algorithm maintains a counter $\tau_{L \rightarrow r} = |L \setminus C|$,
which records how many attributes in the left-hand side (LHS) are not yet present in the current closure $C$.
When this counter reaches zero, the rule is triggered and $r$ is added to $C$.
When a new attribute $\mathrm{{attr}_{new}}$ is added to $C$, it is enqueued and used to update only those FD rules whose LHS contains $\mathrm{{attr}_{new}}$, by decrementing
their counters. This incremental propagation avoids repeatedly scanning all FD rules and yields an efficient computation of the FD closure.
As noted earlier, we restrict the right-hand side (RHS) of each FD to be a single attribute, while allowing the LHS to include up to a user-specified maximum number of attributes (default 2).

\begin{algorithm}[ht]
\caption{\textsc{BuildHits}: Build per-identifier candidate sets via inverted index}
\label{alg:buildhits}
\begin{algorithmic}[1]
\Function{BuildHits}{$D,\; S,\; \hat{\mathbf{x}},\; \tau,\; \mathrm{mode}$}
\State $\textit{hits} \gets []$
\For{each $\texttt{id} \in D$}
    \State $K_{\texttt{id}} \gets \{ f \in S \mid f \text{ determines } \texttt{id} \}$
    \If{$K_{\texttt{id}} = \emptyset$} \textbf{continue} \EndIf
    \State $\mathcal{C}_{\texttt{id}} \gets \bigcap_{f \in K_{\texttt{id}}} \mathcal{I}[\texttt{id}][f][\hat{x}_f]$
    \Comment{O$(|K_{\texttt{id}}|)$}
    \If{mode $= \textsc{Strict}$ \textbf{and} $0 < |\mathcal{C}_{\texttt{id}}| \leq \tau$}
        \State append $(\texttt{id},\; K_{\texttt{id}},\; \mathcal{C}_{\texttt{id}})$ to $\textit{hits}$
    \ElsIf{mode $= \textsc{Relaxed}$ \textbf{and} $|\mathcal{C}_{\texttt{id}}| > 0$}
        \State $R_{\texttt{id}}^{\textsf{const}} \gets$ features constant across $\mathcal{RS}[\texttt{id}][\mathcal{C}_{\texttt{id}}]$
        \State append $(\texttt{id},\; K_{\texttt{id}},\; \mathcal{C}_{\texttt{id}},\; R_{\texttt{id}}^{\textsf{const}})$ to $\textit{hits}$
    \EndIf
\EndFor
\State \Return $\textit{hits}$
\EndFunction
\end{algorithmic}
\end{algorithm}

\begin{algorithm}[ht]
\caption{\textsc{ApplyByIntersection}: Multi-identifier joint rule application}
\label{alg:applybyintersection}
\begin{algorithmic}[1]
\Function{ApplyByIntersection}{$\textit{hits}, \mathcal{RS}, S, m, \Sigma_{\mathrm{FD}},\; \tau,\; \mathrm{mode}$}
\State $\mathcal{P} \gets \bigcap_{(\texttt{id},\ldots)\,\in\,\textit{hits}} \mathcal{RS}[\texttt{id}][\mathcal{C}_{\texttt{id}}]$
\Comment{Alg.~\ref{alg:buildhits}}
\If{$\mathcal{P} = \emptyset$} \Return $m$ \EndIf

\For{each $(\texttt{id}, K_{\texttt{id}}, \mathcal{C}_{\texttt{id}}, R_{\texttt{id}}^{\textsf{const}}) \in \textit{hits}$}
\For{each $(L \rightarrow r) \in \Sigma_{\mathrm{FD}}$ involving $\texttt{id}$ with $L \subseteq S$}
    \If{(mode = \textsc{Strict} \textbf{and} $|\{\texttt{id}(p) \mid p \in \mathcal{P}\}| \leq \tau$)
        \State \textbf{or} (mode = \textsc{Relaxed} \textbf{and} $r$ is constant on $\mathcal{P}$)}
        \State $m_r \gets 1$
    \EndIf
\EndFor
\EndFor
\State \Return $m$
\EndFunction
\end{algorithmic}
\end{algorithm}

\begin{algorithm}[ht]
\caption{\textsc{ProvExpand}: Dropped-identifier based local mask expansion }
\label{alg:prov}
\begin{algorithmic}[1]
\State \textbf{Input:} mask $m$; instance $\mathbf{x}$; dropped identifiers $D$; FD set $\Sigma_{\mathrm{FD}}$ ($|\mathrm{LHS}|\le 2$); inverted index $\mathcal{I}$; row-sets $\mathcal{RS}$; threshold $\tau$; mode $\in\{\textsc{Strict},\textsc{Relaxed}\}$
\State \textbf{Output:} expanded mask $m$

\If{$(m,\,\mathbf{x}) \in \textsf{cache}$} \Return $\textsf{cache}[(m,\,\mathbf{x})]$ \EndIf
\State $S \gets \{ f_j \mid m_j = 1 \}$
\State $\hat{\mathbf{x}} \gets \textsc{Canonicalize}(\mathbf{x})$ \Comment{$\hat{\mathbf{x}}$ cached per instance}
\Statex
\Statex \textit{// Stage 1: build per-identifier candidate sets via inverted index}
\State $\textit{hits} \gets \Call{BuildHits}{D,\; S,\; \hat{\mathbf{x}},\; \tau,\; \mathrm{mode}}$

\Statex
\Statex \textit{// Stage 2: multi-identifier row intersection}
\If{$|\textit{hits}| \geq 2$}
    \State $m \gets \Call{ApplyByIntersection}{\textit{hits}, \mathcal{RS}, S, m, \allowbreak \Sigma_{\mathrm{FD}}, \tau, \mathrm{mode}}$
    \State $S \gets \{ f_j \mid m_j = 1 \}$
\EndIf

\Statex
\Statex \textit{// Stage 3: per-identifier rule application (using updated $S$)}
\For{each $\texttt{id} \in D$ with $0 < |\mathcal{C}_{\texttt{id}}| \leq \tau$}
    \For{each $(L \rightarrow r) \in \Sigma_{\mathrm{FD}}$ involving $\texttt{id}$ with $L \subseteq S$}
        \State $m_r \gets 1$
    \EndFor
\EndFor
\State $S \gets \{ f_j \mid m_j = 1 \}$

\Statex
\Statex \textit{// Stage 4: Over-threshold identifiers (mode = \textsc{Relaxed} only)}
\If{mode $= \textsc{Relaxed}$}
    \For{each $\texttt{id} \in D$}
        \State $K_{\texttt{id}} \gets \{ f \in S \mid f \text{ determines } \texttt{id} \}$
        \State $\mathcal{C}_{\texttt{id}} \gets \bigcap_{f \in K_{\texttt{id}}} \mathcal{I}[\texttt{id}][f][\hat{x}_f]$
        \If{$|\mathcal{C}_{\texttt{id}}| > \tau$}
            \State $R_{\texttt{id}}^{\textsf{const}} \gets$ features constant across $\mathcal{RS}[\texttt{id}][\mathcal{C}_{\texttt{id}}]$
            \For{each $(L \rightarrow r) \in \Sigma_{\mathrm{FD}}$ involving $\texttt{id}$ with $L \subseteq S$, $r \in R_{\texttt{id}}^{\textsf{const}}$}
                \State $m_r \gets 1$
            \EndFor
        \EndIf
    \EndFor
\EndIf

\State $\textsf{cache}[(m,\,\mathbf{x})] \gets m$
\State \Return $m$
\end{algorithmic}
\end{algorithm}

\begin{algorithm}[ht]
\caption{\textsc{ICRepair}: Repairing input data points under integrity constraints (domain and denial constraints)}
\label{alg:icrepair}
\begin{algorithmic}[1]
\For{each input data point $\mathbf{x}$}
    \State $\mathbf{x}_{\mathrm{fixed}} \gets \{x_j \mid m_j = 1 \}$ \Comment{fixed (do not modify)}
    \State $\mathbf{x}_{\mathrm{repair}} \gets \{x_j \mid m_j = 0 \}$ \Comment{repairable}

    \For{each rule $(\ell \Rightarrow r) \in \Sigma_{\mathrm{dom}}$}
        \If{$\mathbf{x}_{\mathrm{repair}}$ satisfies $\ell$ but violates $r$}
            \State modify $\mathbf{x}_{\mathrm{repair}}$ to satisfy $r$ \Comment{policy-dependent}
        \EndIf
    \EndFor

    \For{each $\psi \in \Sigma_{\mathrm{den}}$}
        \If{$\mathbf{x}_{\mathrm{repair}}$ violates $\psi$}
            \State modify only $\mathbf{x}_{\mathrm{repair}}$ to satisfy $\psi$
        \EndIf
    \EndFor
\EndFor
\State \Return $\mathbf{x}$
\end{algorithmic}
\end{algorithm}

Alg.~\ref{alg:buildhits} resolves candidate entity identifiers for each dropped identifier column using a pre-built inverted index $\mathcal{I}[\texttt{id}][f][v] \to \textit{frozenset of id-values}$, requiring $O\!\left(\sum_{\texttt{id} \in D} |K_{\texttt{id}}|\right)$
frozenset intersection operations per coalition (covering both single and composite keys on the LHS), and avoiding any per-coalition scan or groupby operation over the reference data.
Alg.~\ref{alg:applybyintersection} handles the case where two or more identifiers are simultaneously active. It intersects pre-built row-index sets $\mathcal{RS}$ (from Alg.~\ref{alg:buildhits}) across all active identifiers to obtain the jointly consistent row set $\mathcal{P}$, and applies FD rules whose conclusions are consistent within $\mathcal{P}$.
After calling \textsc{BuildHits} and \textsc{ApplyByIntersection}, Alg.~\ref{alg:prov} applies FD rules per identifier using the updated $S$, and if under \emph{Relaxed} mode, it additionally activates features that are constant across over-threshold entity groups. The result for each (mask, instance) pair is cached to avoid recomputation across the many coalitions that share the same pattern.
Alg.~\ref{alg:icrepair} only repairs values in the absent feature coordinates ($m_j=0$) so that the constructed inputs satisfy $\Sigma_{\mathrm{dom}} \cup \Sigma_{\mathrm{den}}$ under the chosen repair policy. For numeric constraints, we consider the following three repair policies:

\textbf{Minimal edit (projection).}
The value is projected to the nearest feasible point within the allowed interval.
For example, under
\[
\{\texttt{life\_stage}=\texttt{middle}\} \Rightarrow 18 \le \texttt{age} \le 64,
\]
if an instance has \texttt{life\_stage}=\texttt{middle} but has $\texttt{age}=16$, and the \texttt{age} attribute is repairable (i.e., $m_j=0$), the value is minimally adjusted to the nearest valid boundary, namely $\texttt{age}=18$, thereby enforcing the smallest possible modification.

\textbf{Random feasible resampling.}
A new value is sampled uniformly at random from feasible values observed in the lookup table (reference data), subject to satisfying the constraint. In the above example, any value appearing in the reference table within the range $[18,\,64]$ may be selected.

\textbf{Distance-weighted random resampling.}
A new feasible value is sampled from the reference table with probability proportional to a distance-based weight, so that values closer to the original value are more likely to be selected. In the above example, values such as 18 or 19 are more likely than distant values such as 60. Let $o$ be the original value and let $\mathcal{V}=\{v_1,\ldots,v_K\}$ be the set of feasible values in the lookup table.
Define $d_k = |v_k - o|$ and $w_k = \frac{1}{1+d_k}$, and sample $v_k$ with probability
$\Pr[v=v_k] = \frac{w_k}{\sum_{j=1}^K w_j}.$

For categorical constraints including denial constraints, there is typically no natural notion of distance. For example, under a constraint of the form
\[
\{\texttt{region} = \texttt{`EU'} \} \Rightarrow \texttt{currency} \in \{\texttt{EUR}, \texttt{CHF}, \texttt{GBP}\}
\] or
{\small
\[
\begin{aligned}
\texttt{gender\_male}, \texttt{\_female}, \texttt{\_nonbinary} &\in \{0,1\}, \\
\texttt{gender\_male} + \texttt{\_female} + \texttt{\_nonbinary} &= 1 .
\end{aligned}
\]}

a feasible category is selected uniformly at random from the admissible values observed in the reference data.

Furthermore, in the implementation, Kernel SHAP internally represents coalitions using masks defined only over varying features, i.e., features whose values can differ from the instance under the background distribution.
RelShap follows the same convention by first identifying such varying features and operating on the corresponding varying feature masks before applying relational canonicalization and refinement.
In addition, throughout RelShap, masks are encoded as compact bit masks to support fast canonicalization, memoization, and lookup.

\section{Closure system and lattice induced by functional dependencies} \label{app:closure_lattice}

\subsection{Closure system} \label{app:closure}

\begin{remark}
\label{rem:lattice}
Importantly, under Armstrong's axioms~\cite{armstrong1974dependency}, 
The family of FD-closed sets induced by $\operatorname{cl}_{\Sigma}$ forms a closure system (\ie it is closed under arbitrary intersections), 
and the set of closed sets ordered by inclusion ($\subseteq$) forms a lattice.
In this lattice, the meet operation corresponds to set intersection, 
and the join operation corresponds to the closure of the union of two sets \ie for any $C_1, C_2 \in \mathcal{C}_{\Sigma}$, consider an FD-based closure operator $\operatorname{cl}_{\Sigma}$ (equivalent to \textsc{FDClosure} in Appendix~Alg.~\ref{alg:fdclosure}), then
\[
C_1 \wedge C_2 = C_1 \cap C_2, 
\qquad
C_1 \vee C_2 = \operatorname{cl}_{\Sigma}(C_1 \cup C_2).
\]

In particular, $\operatorname{cl}_{\Sigma}$ is monotone: for any $S,T \subseteq F \setminus \{i\}$,
\[
S \subseteq T \;\rightarrow\; \operatorname{cl}_{\Sigma}(S) \subseteq \operatorname{cl}_{\Sigma}(T).
\]

These setups differentiate \RelShap from existing frameworks. Witter et al.~\cite{witter2026exact} also exploit equivalence classes 
and lattice structure, but in a different setting. Their equivalence classes arise from causal relationships in a structural causal model (SCM), where interventions on some variables can render others redundant, not from the Armstrong-style functional dependency closure system (\ie a Moore family) used in \RelShap.
In addition, \citet{maafa2018algorithms} study lattice-structured coalition spaces to reduce computations, but not for ML predictions. Importantly, both works redefine the coalition space of the cooperative game itself, while \RelShap operates on the original coalition space induced by estimators.

\end{remark}

\subsection{Lattice induced by functional dependencies}
\label{app:fd_lattice}

\begin{lemma}[Structure of $\mathcal{C}_{\Sigma}$ {\cite[Thm.~3.1]{demetrovics1992functional}}]

\label{lemma:c_sigma_structure} $\mathcal{C}_{\Sigma}$ can be identified with the family of $\Sigma$-closed sets as:
\[
\mathcal{C}_{\Sigma}
=
2^{F \setminus \{i\}}
\setminus
\bigcup_{L \to R \in \Sigma_{\mathrm{FD}},\, r \in R \setminus L}
[L,\; (F \setminus \{i\}) \setminus \{r\}],
\]
$ \text{where} \;
[L,\; (F \setminus \{i\}) \setminus \{r\}]
=
\{\, S \subseteq F \setminus \{i\} \mid
L \subseteq S \subseteq (F \setminus \{i\}) \setminus \{r\} \,\}.
$
\end{lemma}

Intuitively, for each FD $L \to R$, any set that contains $L$ but omits some
attribute $r \in R$ violates the dependency and therefore cannot be a closed set. Several representative FD patterns for which closed-form expressions or simple counting formulas can be derived are given in Demetrovics \textit{et al.}~\cite{demetrovics1992functional}. These include: single FD; unary chains (derived from Prop.~6.1 in ~\cite{demetrovics1992functional}; $\Sigma=\{L_1 \to L_2,\; L_2 \to L_3,\; \ldots,\; L_{p-1} \to L_p\}$); disjoint unary chains (multiple unary chains on pairwise disjoint attribute sets); and nested left-hand sides (derived from Prop.~6.10; 
$\Sigma=\{
L_1 \to r_1,
L_2 \to r_2,
\ldots,
L_p \to r_p,\}, \;L_1 \subseteq L_2 \subseteq \cdots \subseteq L_p
$).
Here we restrict our analysis to the single FD case, and let $U$ denote the coalition universe $F\setminus\{i\}$. Moreover, following~\cite{demetrovics1985minimum}, when $\Sigma$ decomposes over attribute-disjoint blocks, both $|\mathcal C_\Sigma|$ and $N_C(k)$ can be computed by combining the corresponding quantities of each block, reducing the counting problem to smaller independent components. Specifically, if
\[
\Sigma = \bigcup_{j=1}^{p} \Sigma_j,
\;\;
\mathcal{C}_{\Sigma} \cong \prod_{j=1}^{p} \mathcal{C}_{\Sigma_j},
\]
then
\[
|\mathcal C_\Sigma|
=
\prod_{j=1}^{p} |\mathcal C_{\Sigma_j}|,
\;\;
N_C(k)
=
\sum_{\substack{k_1+\cdots+k_p = k}}
\prod_{j=1}^{p} N^{(j)}_{C_j}(k_j),
\]
where $C = (C_1,\dots,C_p)$ with $C_j \in \mathcal{C}_{\Sigma_j}$.

\begin{proposition}[Single FD {\cite[Derived from Thm.~3.1]{demetrovics1992functional}}]
\label{prop:single_fd_count}
Suppose $\Sigma = \{L \to R\}$, where $L,R \subseteq U$ and $L \cap R = \varnothing$. Then the number of $\Sigma$-closed sets is
\[
|\mathcal{C}_{\Sigma}|
=
2^{|U|}
-
2^{\,|U|-|L|}
+
2^{\,|U|-|L|-|R|}.
\]

Moreover, for each $C \in \mathcal{C}_\Sigma$,
\[
N_C(k)
=
\begin{cases}
\mathbf{1}\{k=|C|\}, & L\nsubseteq C,\\[6pt]
\displaystyle \binom{|R|}{\,k-(|C|-|R|)\,}, & L\subseteq C,
\end{cases}
\]
where $\binom{a}{b}=0$ if $b<0$ or $b>a$.

\end{proposition}

\begin{proof}
A set $S \subseteq U$ is $\Sigma$-closed iff $L \subseteq S \rightarrow R \subseteq S.$ A subset violates the dependency if it contains $L$ but omits at least one attribute of $R$. The number of sets containing $L$ is $2^{|U|-|L|}$, and among them those containing $R$ are $2^{|U|-|L|-|R|}$. Hence the number of violating sets is $2^{|U|-|L|}-2^{|U|-|L|-|R|}$, and subtracting these from $2^{|U|}$ yields the result.

If $L\nsubseteq C$, then no subset of $C$ can trigger the FD, so
$\operatorname{cl}_\Sigma(S)=S$ for all $S\subseteq C$, and thus the only generator of $C$ is $S=C$. If $L\subseteq C$, then necessarily $R\subseteq C$
since $C$ is closed, and $\operatorname{cl}_\Sigma(S)=C$ holds exactly for
$S=(C\setminus R)\cup T$ with $T\subseteq R$, which gives the binomial count.
\end{proof}

\subsection{Proof of Thm.~\ref{thm:reduction}} \label{app:runtime_reduction_projection_proof}

\begin{proof}

For each $C \in \mathcal{C}_\Sigma$, define the indicator random variable
\[
I_C = \mathbf{1}\{\exists\, m \le M \mid [S^{(m)}]_\Sigma = C\}.
\]
Then the number of distinct observed classes is
$
K_M = \sum_{C \in \mathcal{C}_\Sigma} I_C.
$ By linearity of expectation,
\[
\mathbb{E}[K_M]
= \sum_{C \in \mathcal{C}_\Sigma} \mathbb{E}[I_C]
= \sum_{C \in \mathcal{C}_\Sigma} \Pr(I_C=1)
\]
\[
=\sum_{C \in \mathcal{C}_\Sigma}
\left(1-(1-q(C))^M\right).
\]
Conditioning on the coalition size,
\[
q(C) = \sum_{k=0}^{|F|-1}
\Pr([S]_\Sigma=C \mid |S|=k)\,p_k.
\]
For every $C \in \mathcal{C}_\Sigma$, since $r(C)$ is the size of a minimal generator of $C$, we have $N_C(k)=0$ for all $k<r(C)$. Then the support of the sum can be restricted to $k\ge r(C)$,
\[
q(C)
=
\sum_{k=r(C)}^{|F|-1}
p_k\,\frac{N_C(k)}{\binom{|F|-1}{k}}.
\]
Since $S$ is drawn uniformly conditional on $|S|=k$,
\[
\Pr([S]_\Sigma=C \mid |S|=k)
=
\frac{N_C(k)}{\binom{|F|-1}{k}}.
\]

\end{proof}

This is a coupon collector-type proof~\cite{flajolet1992birthday, motwani1995randomized}.

\begin{remark}
Monte Carlo coalition sampling draws coalitions independently, and therefore the samples are i.i.d.
Kernel SHAP involves additional implementation steps including handling duplicate coalitions, but the underlying sampling procedure draws coalitions independently with replacement, and we therefore model the samples as i.i.d.
Leverage SHAP employs stratified sampling without replacement within each coalition-size stratum. Since the coalition space grows exponentially with the number of features, the probability of sampling the same coalition more than once is negligible for practical budgets. We therefore approximate the resulting sampling process by an i.i.d. coalition sampling model.
\end{remark}

\begin{remark}
Considering the Shapley value definition in Section~\ref{sec:background}, the overall runtime improvement is proportional to $R_M$ up to constant multiplicative factors. 
First, Shapley value computation requires averaging over background data, yielding a cost proportional to the number of background points used. 
Second, for Monte Carlo estimator, each marginal contribution involves both $S$ and $S \cup \{i\}$, introducing at most a factor of two (which could be smaller in practice due to caching or invariant features in the actual computation), whereas this does not apply to Kernel SHAP and Leverage SHAP since marginal contributions are implicitly accounted for within the solver and are not computed via explicit pairwise evaluations.
\end{remark}

\subsection{Estimator-specific expected runtime reduction} \label{app:estimator_specific}

\begin{corollary}[Monte Carlo coalition estimator]
\label{cor:mc}
From Section~\ref{app:mc}, 
\[
p_k=\frac{1}{|F|}, \;\;
q(C)
= \frac{1}{|F|}
\sum_{k=r(C)}^{|F|-1}
\frac{N_C(k)}{\binom{|F|-1}{k}}.
\]
\end{corollary}

\begin{corollary}[Kernel SHAP]
\label{cor:kernel_runtime}
From Section~\ref{app:kernelshap},
\[
p_k
=
\frac{1}{Z_{|F|-1}}\frac{1}{k(|F|-1-k)},
\] for $k=1,\dots,|F|-2$ where the normalization factor
\[
Z_{|F|-1}
=\frac{1}{|F|-1}
\sum_{t=1}^{|F|-2}
\left(
\frac{1}{t}+\frac{1}{|F|-1-t}
\right)
\]
\[
=\frac{2}{|F|-1}\sum_{t=1}^{|F|-2}\frac{1}{t}
=\frac{2H_{|F|-2}}{|F|-1}.
\]
and $H_{|F|-2}$ is the harmonic number.
\end{corollary}

\begin{corollary}[Leverage SHAP]
\label{cor:lev_runtime}
Let $m_k$ denote the number of sampled coalitions of size $k$, with
$\sum_{k=0}^{|F|-1} m_k = M$. Hence
\[
p_k=\frac{m_k}{M}, \;\; 
q(C)
= \sum_{k=r(C)}^{|F|-1}
\frac{m_k}{M}
\frac{N_C(k)}{\binom{|F|-1}{k}},
\]
\end{corollary}

\subsubsection{Interpretation}
Since $M$ is fixed, the degree of runtime reduction is determined by $q(C)$. If many classes $C$ have small $q(C)$, new classes tend to appear across samples and few repetitions occur, yielding $R_M \approx 0$ and little runtime reduction. In contrast, if some classes have large $q(C)$, the same classes are repeatedly sampled and the projection merges these repetitions, increasing $R_M$.

The probability $q(C)$ is structurally determined by the interaction between the coalition-size distribution $p_k$ and the collapse counts $N_C(k)$. Here $p_k$ describes how frequently coalitions of size $k$ are sampled by the estimator, while $N_C(k)$ counts how many coalitions of size $k$ collapse to the same closure class $C$. Therefore, if $p_k$ concentrates probability mass on coalition sizes where $N_C(k)$ is large, the resulting $q(C)$ becomes large and the same classes are repeatedly sampled, leading to greater runtime reduction. Conversely, if $p_k$ spreads probability mass broadly across sizes or $N_C(k)$ remains small across levels, then $q(C)$ stays small and new classes continue to appear, resulting in little runtime reduction.

Consider Example~\ref{ex:motivating} with the discovered FD $\mathtt{age} \to \mathtt{life\_stage}$ over $\{\mathtt{age}, \mathtt{life\_stage}, \allowbreak \mathtt{empl}, \allowbreak \mathtt{total\_amt}\}$.  Without RelShap, the base sampler treats all $2^4 = 16$ coalitions as distinct.  Under quotient-space projection, any coalition containing $\mathtt{age}$ is expanded to also include
$\mathtt{life\_stage}$, so $\{\mathtt{age}\}$ and $\{\mathtt{age}, \mathtt{life\_stage}\}$ collapse to the same
equivalence class, as do $\{\mathtt{age}, \mathtt{empl}\}$ and $\{\mathtt{age}, \mathtt{life\_stage}, \mathtt{empl}\}$, and so on.  The effective number of distinct classes $|\mathcal{C}_{\Sigma}|$ drops from $16$ to $12$ (see Appendix~Prop.~\ref{prop:single_fd_count}), and any duplicate draw reuses the cached evaluation. Thm.~\ref{thm:reduction} generalizes this: $q(C)$ measures how likely a sampled coalition is to land in class $C$, and $R_M$ quantifies the expected fraction of redundant coalition evaluations eliminated by quotient projection, as a function of the sampler's size distribution $p_k$ and the FD-induced collapse
counts $N_C(k)$.

The decomposition admits a combinatorial interpretation: for each $C \in \mathcal{C}_\Sigma$, $N_C(k)$ counts the number of size-$k$ coalitions whose closure generates $C$ (\ie size-$k$ generators of $C$), and $q(C)$ aggregates these counts weighted by the sampler's size distribution $p_k$.  The collapse counts $N_C(k)$ depend on the FD structure of $\Sigma_{\mathrm{FD}}$; representative cases are given in Appendix~\ref{app:fd_lattice}.  Intuitively, FDs with small left-hand sides induce equivalence classes with small rank $r(C)$, so that many coalitions collapse to the same class, yielding large $N_C(k)$ at small sizes.  Runtime reduction therefore emerges from the interaction between the estimator's size distribution and the FD-induced collapse structure:

\begin{itemize}[leftmargin=*]
\item \emph{Kernel SHAP} places larger probability mass on extreme (small and large) coalition sizes.  Since FDs in practice tend to have small left-hand sides, the resulting classes have small $r(C)$ and large $N_C(k)$ for small $k$, precisely the sizes Kernel SHAP oversamples.  The runtime improvement is therefore amplified when $\Sigma_{\mathrm{FD}}$ contains many low-arity dependencies.
\item \emph{Monte Carlo and Leverage SHAP} distribute probability more uniformly across coalition sizes and do not strongly favor any particular $k$.  The runtime improvement depends mainly on the overall magnitude of the collapse counts $N_C(k)$ across all sizes; that is, on how many coalitions share the same closure.
\end{itemize}

\section{Details of datasets} \label{app:dataset}

For each dataset, we summarize the identifiers that are dropped (used in provenance-aware mode), together with a high-level dataset overview in Table~\ref{tab:dataset_stats}. For the runtime required to extract schema-, query-, and data-level relational constraints, including the refinement step, over the entire dataset ($T_{\mathrm{extract}}$), all datasets complete within 12 seconds except for SpeedDating and MovieLens 20M. In both cases, the cost is dominated by data-driven FD extraction—due to the large number of features ($|F|$) in SpeedDating and the large number of data points in MovieLens 20M—taking 482.687s and 98.922s, respectively.
Table~\ref{tab:dataset_sources} lists the corresponding machine learning benchmark sources (OpenML~\cite{vanschoren2014openml}, OpenML-CC18~\cite{bischl2021openml}, TabZilla~\cite{mcelfresh2023neural}, and TabRed~\cite{rubachev2024tabred}).
Entity-relationship diagrams (ERDs) with detailed queries and the full set of extracted relational constraints are provided in our supplementary material.
As expected, query execution over the normalized schema preserves the full content of the original dataset.

\textbf{Amazon employee access.}
This dataset addresses a classification task that predicts whether access to a resource is approved or denied (\texttt{ACTION}) based on role of employees and resource-related information. Since the original data is provided as a single flattened table, role (job) types identified by \texttt{ROLE\_CODE} are separated into a \texttt{Role} table, while employee's resource access request instances are stored in an \texttt{Employee} table. Each instance is assigned an artificial identifier \texttt{row\_id} to enable row-level traceability to the original data.

\textbf{Churn.}
This dataset targets a binary classification task that predicts customer churn from customer account attributes and call usage statistics. We apply a normalization to create \texttt{Customer}, \texttt{Usage}, and \texttt{Period} tables: \texttt{Customer} stores customer-level profile information, \texttt{Period} enumerates call categories (DAY, EVE, NIGHT, INTL), and \texttt{Usage} stores per-customer and per-period call usage statistics. In the final flattened table, period-specific usage records are pivoted into distinct feature columns for each customer (e.g., total\_day\_minutes, total\_eve\_calls, total\_night\_charge), reflecting the underlying normalized structure and the functional dependencies it induces.

\textbf{Churn Modelling.} This dataset addresses a binary classification task that predicts whether a bank customer leaves the bank (Exited) from customer attributes. The provided flattened table includes two identifier-like columns, \texttt{RowNumber} and \texttt{CustomerId}, which are typically removed in the downstream machine learning task; however, we use them to normalize the data into \texttt{Customer} and \texttt{Profile} tables. The \texttt{Customer} table represents customer entities and stores basic personal information, while the \texttt{Profile} table represents one record per row identified by \texttt{RowNumber}, linked to \texttt{Customer} through \texttt{CustomerId}, and stores customer financial information.

\textbf{German Credit (credit-g).} This dataset is designed for a binary classification task that predicts whether an applicant has a good or bad credit risk based on demographic and financial attributes. We apply a semantic normalization that separates applicant-level attributes into the \texttt{Applicant} table and credit application–specific attributes into the \texttt{Application} table. Each record is assigned a synthetic identifier \texttt{application\_id}, and \texttt{Application} references \texttt{Applicant} through a one-to-one relationship.

\textbf{SpeedDating.} This dataset is designed for a binary classification task that predicts whether two participants ultimately select each other (whether a match occurs), based on participant profile (survey) information and mutual evaluation data collected from experimental speed dating events. Although the original data is provided as a single flattened table, the semantics naturally separate participant-level attributes from encounter-level attributes. Accordingly, we normalize the data into three tables: \texttt{Person}, \texttt{Interaction}, and \texttt{BucketRule}.
Table \texttt{Person} stores attributes that are fixed for a participant (e.g., preference criteria)
Table \texttt{Interaction} stores encounter-level attributes observed when two participants meet (e.g., ratings of the partner)
Since the column \texttt{wave} denotes an independent experimental session, participants can interact only with others within the same \texttt{wave}. Therefore, all identifiers and joins are performed within \texttt{wave}. 
The relational structure of this dataset inherently induces a self-join on the \texttt{Person} table, since each interaction links two persons, identified by \texttt{self\_id} and \texttt{partner\_id}, within the same entity.
Since the dataset contains both continuous variables and their discretized counterparts (d\_*), we construct a \texttt{BucketRule} table that records the mapping from each base feature to its bucket label. 
At the query-level, we derive composite features such as age difference and same-race indicators, which also induce FDs. 
Participant–partner identities are recovered by enforcing reciprocal consistency within each wave. Rows without a unique match are removed, yielding 8,364 rows from the original 8,378, with each interaction represented by two directed rows.

\textbf{TPC-H.} We define a binary classification task at the supplier level that predicts whether a supplier is high risk based on operational performance signals derived from transactional data. A final table used for ML task contains per-supplier aggregates summarizing transaction volume, financial performance, delivery behavior, and return behavior, together with supplier account balance and geographic attributes (nation and region). We derive a binary label \texttt{supplier\_risk} by thresholding the late delivery rate and return rate at their 75th percentiles; a supplier is labeled positive if either rate exceeds its threshold.

\textbf{UW-CSE.}
We study a multi-class classification problem that predicts a students program phase (\texttt{inPhase}), a commonly studied task on the UW-CSE dataset, using information about advising and course enrollment. 
We derive a student-centric table in which each row summarizes a students academic context via aggregated statistics, such as the number of advisors and counts of courses at different levels, along with the students years in program. The identifier \texttt{p\_id} is kept only to index instances and is excluded from the feature set. The relational design further gives rise to a self-join on \texttt{person}, since the \texttt{advisedBy} relation associates one person with another person in different roles (student or professor).

\textbf{MovieLens 20M.}
This dataset targets a binary classification task that predicts whether a user gives a movie a high rating. 
Since the task is defined over user-movie pairs, the final table retains both \texttt{userId} and \texttt{movieId}, and includes movie attributes, genre identifiers, and a representative tag feature. The specific machine learning task and query are defined by the authors.

\textbf{Synthetic.}
This is a synthetic relational dataset constructed by the authors for a binary classification task that predicts loan approval as an extension of Example~\ref{ex:motivating}. It consists of three tables: \texttt{Applicants}, \texttt{Transactions}, and \texttt{Items}, and includes FD: \texttt{age} $\rightarrow$ \texttt{life\_stage} as well as conditional domain constraints between \texttt{region} and \texttt{currency}, designed to evaluate how RelShap incorporates conditional domain knowledge and denial constraints (\texttt{gender}). The final learning table is generated by a query with transaction-level aggregation (e.g., transaction count and total quantity) and top item selection; these operations induce FDs, and thus identifiers \texttt{a\_id} (applicant) and \texttt{t\_id} (transaction) are dropped.

\begin{table*}[t!]
\centering
\caption{ML: flattened datasets with predefined machine learning tasks (relational structure is normalized by the authors). DB: relational schemas are given and machine learning tasks are defined by the authors via queries. $|F|$: \# of input features excluding the target. Train/test splits: scikit-learn's \texttt{train\_test\_split} with an 8:2 ratio. 
Due to memory constraints, for model training, we used 2.5\% of the MovieLens 20M training sets. For background data selection, we used 400,000 samples for MovieLens 20M, and 2,000 for SpeedDating. $M_{\mathrm{conv}}$, $B$: coalition budget and the chosen background size for convergence, $T_{\mathrm{extract}}$: runtime (in seconds) taken to extract relational constraints over the entire dataset. 
}
\label{tab:dataset_stats}

\begin{tabular}{c l r r r r r r r}
\toprule
Type & Dataset & $|F|$ & \# data points & \# train & \# test & $M_{\mathrm{conv}}$ & $B$ & $T_{\mathrm{extract}} (sec)$ \\
\midrule
\multirow{5}{*}{ML} 
& Amazon & 9 & 32{,}769 & 26{,}215 & 6{,}554 & $2^7$ & $26{,}215$ & 7.139 \\
& Churn & 20 & 5{,}000 & 4{,}000 & 1{,}000 & $2^9$ & $4{,}000$ & 6.961  \\
& Churn Modelling & 11 & 10{,}000 & 8{,}000 & 2{,}000 & $2^5$ & $8{,}000$ & 6.028  \\
& German Credit & 20 & 1{,}000 & 800 & 200 & $2^8$ & $800$ & 6.513  \\
& SpeedDating & 121 & 8{,}364 & 6{,}691 & 1{,}673 & $2^9$ & $2{,}000$ & 482.687  \\
\midrule
\multirow{4}{*}{DB} 
& TPC-H & 22 & 10{,}000 & 8{,}000 & 2{,}000 & $2^8$ & $8{,}000$ & 11.533 \\
& UW-CSE & 7 & 140 & 112 & 28 & $2^6$ & 112 & 6.712  \\
& MovieLens 20M & 5 & 20{,}000{,}264 & 1{,}600{,}211 & 400{,}053 & $2^3$ & $10{,}000$ & 98.922  \\
& Synthetic & 11 & 3{,}000 & 2{,}400 & 600 & $2^6$ & $2{,}400$ & 7.107 \\
\bottomrule
\end{tabular}
\end{table*}

\begin{table*}[ht]
\centering
\caption{ML benchmark inclusion for the tabular ML datasets used in our experiments.}
\label{tab:dataset_sources}

\begin{tabular}{l c c c c}
\toprule
Dataset 
& OpenML ID
& OpenML-CC18
& TabZilla
& TabRed \\
\midrule
Amazon Employee Access & 43900 & \cmark & \cmark &  \\
Churn                 & 40701 &  & \cmark &  \\
Churn Modeling        & N/A &  & & \cmark \\
German Credit              & 31 & \cmark & \cmark & \cmark \\
SpeedDating           & 40536 &  & \cmark & \cmark \\
\bottomrule
\end{tabular}
\end{table*}

\section{Experiments} \label{app:more_experiments}

\subsection{Detailed experimental setups} \label{app:detailed_exp_setups}

\textbf{Hardware and reproducibility.} All experiments ran on a machine with an Intel Xeon Platinum 8592+ CPU (128 cores, 512\,GB RAM) and an NVIDIA A100 GPU (80\,GB VRAM) for deep learning models. Implementation details for reproducibility are in our code repository.
Table~\ref{app:tuning_space} lists the hyperparameter settings for training predictive models; since Shapley values are post-hoc and model-agnostic, we do not optimize predictive performance.
For all experiments, we use three random seeds: 2026, 2027, and 2028. Software versions, environment configuration, and full implementation details for reproducibility are available in our code repository. Availability and summary of relational constraints per dataset are summarized in Tables~\ref{tab:relshap_options} and~\ref{tab:constraints_summary}.

\textbf{Convergence protocol.} For each dataset we identify a unified convergence budget $M_{\mathrm{conv}}$ at which Shapley estimates stabilize, using Top-$3$ Jaccard distance as the convergence criterion, after verifying qualitatively similar convergence across estimators~\cite{covert2020improving, yuan2022empirical}.

\begin{table}[t]
\small
\centering
\caption{Hyperparameter settings for each model.}
\begin{tabular}{p{1.5cm} l}
\toprule
Model & Hyperparameter space \\
\midrule
Logistic Regression
& regularization parameter $\in [10^{-4}, 10^{2}]$ \\
\midrule
Random Forest 
& number of trees $\in \{300, 500, 800\}$ \\
& maximum tree depth $\in \{\text{None}, 6, 10, 16\}$ \\
& minimum samples per leaf $\in \{1, 2, 5\}$ \\
& feature subsampling \\
& $\in \{\text{sqrt}, \text{log2}, \text{None}\}$ \\
\midrule
XGBoost 
& number of trees $\in \{300, 500, 800\}$ \\
& maximum tree depth $\in \{3, 4, 6, 8\}$ \\
& learning rate $\in \{0.03, 0.05, 0.1\}$ \\
& row/feature subsampling \\
& $\in \{0.7, 0.8, 1.0\}$ \\
& minimum child weight $\in \{1, 3, 5\}$ \\
& L2 regularization strength \\
& $\in \{0.5, 1.0, 2.0\}$ \\
\midrule
MLP-PLR 
& embedding dimension $\in \{12, 24, 32\}$ \\
& \# of frequencies $\in \{8, 16, 32\}$ \\
& frequency scale $\in [0.01, 0.2]$ \\
& hidden layer sizes \\
& $\in \{(128), (256), (256,128), (512,256)\}$ \\
& dropout rate $\in \{0.0, 0.1, 0.2\}$ \\
& learning rate $\in \{3\!\times\!10^{-4}, 10^{-3}, 3\!\times\!10^{-3}\}$ \\
& weight decay $\in \{0, 10^{-5}, 10^{-4}, 10^{-3}\}$ \\
\bottomrule
\label{app:tuning_space}
\end{tabular}
\end{table}

\begin{table}[ht]
\centering
\caption{\RelShap configuration options.  Each option can be combined with any dataset, model, and coalition estimator.
  Availability depends on the constraint types present in the dataset (Table~\ref{tab:constraints_summary}).  Options above the mid-rule change the explanation; the option below it affects only running time.}
\begin{tabular}{lp{5.8cm}}
\toprule
Option & Description \\
\midrule
BG
  & Background data filtered to satisfy $\Sigma_{\mathrm{FD}}$ (Def.~\ref{def:relshap:bg}). \\
DCs
  & Domain and denial constraints ($\Sigma_{\mathrm{dom}} \cup \Sigma_{\mathrm{den}}$) enforced during background filtering. \\
Prov, Strict 
  & Identifier-induced FDs applied when the dropped identifier is uniquely determined  from the coalition. \\
Prov, Relaxed
  & Identifier-induced FDs applied when attribute values are constant across candidate identifiers (\textbf{subsumes Strict}). \\
\midrule
Quotient (Q)
  & Coalitions mapped to FD-closed canonical form; $\Sigma$-equivalent coalitions deduplicated.  \textbf{Performance optimization, Shapley values unchanged} (Prop.~\ref{prop:quotient_invariance}). \\
\bottomrule
\end{tabular}
\label{tab:relshap_options}
\end{table}

\begin{table}[ht]
\centering
\small
\caption{Summary of relational constraints per dataset. $|F|$: number of features. \# FD: total functional dependencies (breakdown $(|\mathrm{LHS}|{=}1 / 2)$). \# DCs: domain and denial constraints (lower/upper bounds counted separately). \# Prov: FDs associated with provenance-aware mode.
}
\begin{tabular}{lcccc}
\toprule
Dataset & $|F|$ & \# FD & \# DCs & \# Prov \\
\midrule
Amazon          & 9   & 4 (4/0)     & 0   & 9   \\
Churn           & 20  & 29 (25/4)   & 0   & 20  \\
Churn Modelling & 11  & 45 (3/42)   & 0   & 24  \\
German Credit   & 20  & 6 (0/6)     & 0   & 20  \\
SpeedDating     & 121 & 58 (55/3)   & 209 & 231 \\
TPC-H           & 22  & 426 (106/320) & 5 & 22  \\
UW-CSE          & 7   & 6 (6/0)     & 6   & 7   \\
MovieLens 20M   & 5   & 3 (3/0)     & 0   & 0   \\
Synthetic       & 11  & 19 (2/17)   & 9   & 29  \\
\bottomrule
\end{tabular}
\label{tab:constraints_summary}
\end{table}

\subsection{Detailed experimental results}

\subsubsection{Semantic effects} \label{app:exp:semantic_effects}

Table~\ref{tab:summary_main} summarizes explanation differences and their deviations from the randomized null baseline across datasets, predictive models, estimators, and comparison modes. Table~\ref{tab:bg_vs_bgdc_dataset} further provides the per-dataset breakdown for the comparison modes yielding the largest deviations from the null baseline.

\begin{table*}[ht]
\centering
\setlength{\tabcolsep}{5pt}

\begin{subtable}{\textwidth}
\centering
\caption{By dataset}

\begin{tabular}{lcccc}
\toprule
Dataset & Top-3 Jaccard & $1-\mathrm{RBO}$ & $\Delta_{\mathrm{Jaccard}}$ & $\Delta_{\mathrm{RBO}}$ \\
\midrule
Amazon & 0.106 $\pm$ 0.225 & 0.046 $\pm$ 0.100 & -0.294 $\pm$ 0.225 & -0.129 $\pm$ 0.100 \\
Churn & 0.446 $\pm$ 0.283 & 0.171 $\pm$ 0.109 & 0.044 $\pm$ 0.283 & 0.059 $\pm$ 0.109 \\
Churn Modelling & 0.123 $\pm$ 0.256 & 0.048 $\pm$ 0.104 & -0.277 $\pm$ 0.256 & -0.109 $\pm$ 0.104 \\
German Credit & 0.520 $\pm$ 0.288 & 0.194 $\pm$ 0.116 & 0.118 $\pm$ 0.288 & 0.082 $\pm$ 0.116 \\
SpeedDating & 0.804 $\pm$ 0.189 & 0.245 $\pm$ 0.123 & 0.403 $\pm$ 0.189 & 0.213 $\pm$ 0.123 \\
TPC-H & 0.612 $\pm$ 0.365 & 0.245 $\pm$ 0.142 & 0.212 $\pm$ 0.365 & 0.140 $\pm$ 0.142 \\
UW-CSE & 0.263 $\pm$ 0.243 & 0.121 $\pm$ 0.113 & -0.129 $\pm$ 0.243 & -0.072 $\pm$ 0.113 \\
MovieLens 20M & 0.479 $\pm$ 0.178 & 0.265 $\pm$ 0.094 & 0.129 $\pm$ 0.178 & 0.050 $\pm$ 0.094 \\
Synthetic & 0.500 $\pm$ 0.229 & 0.222 $\pm$ 0.103 & 0.099 $\pm$ 0.229 & 0.065 $\pm$ 0.103 \\
\bottomrule
\end{tabular}
\end{subtable}

\vspace{4pt}

\begin{subtable}{\textwidth}
\centering
\caption{By predictive model}

\begin{tabular}{lcccc}
\toprule
Model & Top-3 Jaccard & $1-\mathrm{RBO}$ & $\Delta_{\mathrm{Jaccard}}$ & $\Delta_{\mathrm{RBO}}$ \\
\midrule
Logistic Regression & 0.425 $\pm$ 0.361 & 0.168 $\pm$ 0.145 & 0.028 $\pm$ 0.361 & 0.033 $\pm$ 0.167 \\
Random Forest & 0.445 $\pm$ 0.331 & 0.173 $\pm$ 0.135 & 0.048 $\pm$ 0.331 & 0.039 $\pm$ 0.156 \\
XGBoost & 0.436 $\pm$ 0.337 & 0.173 $\pm$ 0.135 & 0.039 $\pm$ 0.336 & 0.039 $\pm$ 0.158 \\
MLP-PLR & 0.434 $\pm$ 0.335 & 0.170 $\pm$ 0.134 & 0.037 $\pm$ 0.335 & 0.035 $\pm$ 0.156 \\
\bottomrule
\end{tabular}
\end{subtable}

\vspace{4pt}

\begin{subtable}{\textwidth}
\centering
\caption{By estimator}

\begin{tabular}{lcccc}
\toprule
Estimator & Top-3 Jaccard & $1-\mathrm{RBO}$ & $\Delta_{\mathrm{Jaccard}}$ & $\Delta_{\mathrm{RBO}}$ \\
\midrule
Kernel SHAP & 0.301 $\pm$ 0.283 & 0.101 $\pm$ 0.097 & -0.096 $\pm$ 0.282 & -0.033 $\pm$ 0.124 \\
Leverage SHAP & 0.518 $\pm$ 0.383 & 0.214 $\pm$ 0.155 & 0.121 $\pm$ 0.383 & 0.079 $\pm$ 0.178 \\
MC & 0.488 $\pm$ 0.308 & 0.199 $\pm$ 0.124 & 0.090 $\pm$ 0.309 & 0.065 $\pm$ 0.148 \\
\bottomrule
\end{tabular}
\end{subtable}

\vspace{4pt}

\begin{subtable}{\textwidth}
\centering
\caption{By comparison mode (Default vs. Mode)}

\begin{tabular}{lcccc}
\toprule
Mode & Top-3 Jaccard & $1-\mathrm{RBO}$ & $\Delta_{\mathrm{Jaccard}}$ & $\Delta_{\mathrm{RBO}}$ \\
\midrule
BG & 0.452 $\pm$ 0.337 & 0.190 $\pm$ 0.146 & 0.057 $\pm$ 0.337 & 0.050 $\pm$ 0.156 \\
BG + DCs & 0.566 $\pm$ 0.324 & 0.216 $\pm$ 0.128 & 0.168 $\pm$ 0.322 & 0.094 $\pm$ 0.151 \\
BG + Prov, Strict & 0.377 $\pm$ 0.356 & 0.142 $\pm$ 0.139 & -0.023 $\pm$ 0.355 & 0.011 $\pm$ 0.171 \\
BG + Prov, Relaxed & 0.415 $\pm$ 0.342 & 0.155 $\pm$ 0.130 & 0.015 $\pm$ 0.342 & 0.024 $\pm$ 0.158 \\
All & 0.430 $\pm$ 0.321 & 0.172 $\pm$ 0.128 & 0.036 $\pm$ 0.322 & 0.032 $\pm$ 0.150 \\
\bottomrule
\end{tabular}
\end{subtable}

\vspace{4pt}

\caption{Summary of explanation differences (Top-3 Jaccard, $1-\mathrm{RBO}$) and deviations from the null baseline ($\Delta$), aggregated by dataset, predictive model, estimator, and comparison mode.}
\label{tab:summary_main}
\end{table*}

\begin{table*}[ht]
\centering
\setlength{\tabcolsep}{5pt}

\begin{subtable}{\textwidth}
\centering
\caption{Default vs. BG}
\begin{tabular}{lcccc}
\toprule
Dataset & Top-3 Jaccard & $1-\mathrm{RBO}$ &$\Delta_{\mathrm{Jaccard}}$ & $\Delta_{\mathrm{RBO}}$ \\
\midrule
Amazon & $0.265 \pm 0.381$ & $0.119 \pm 0.171$ & $-0.136 \pm 0.381$ & $-0.056 \pm 0.171$ \\
Churn & $0.553 \pm 0.197$ & $0.214 \pm 0.078$ & $0.151 \pm 0.197$ & $0.102 \pm 0.078$ \\
Churn Modelling & $0.297 \pm 0.427$ & $0.125 \pm 0.180$ & $-0.104 \pm 0.427$ & $-0.033 \pm 0.180$ \\
German Credit & $0.391 \pm 0.417$ & $0.161 \pm 0.177$ & $-0.011 \pm 0.417$ & $0.049 \pm 0.177$ \\
SpeedDating & $0.644 \pm 0.258$ & $0.190 \pm 0.135$ & $0.243 \pm 0.258$ & $0.158 \pm 0.135$ \\
TPC-H & $0.629 \pm 0.296$ & $0.250 \pm 0.111$ & $0.229 \pm 0.296$ & $0.144 \pm 0.111$ \\
UW-CSE & $0.335 \pm 0.240$ & $0.162 \pm 0.117$ & $-0.057 \pm 0.240$ & $-0.031 \pm 0.117$ \\
MovieLens 20M & $0.479 \pm 0.179$ & $0.265 \pm 0.095$ & $0.129 \pm 0.179$ & $0.050 \pm 0.095$ \\
Synthetic & $0.474 \pm 0.320$ & $0.222 \pm 0.147$ & $0.073 \pm 0.320$ & $0.065 \pm 0.147$ \\
\bottomrule
\end{tabular}
\end{subtable}

\vspace{4pt}

\begin{subtable}{\textwidth}
\centering
\caption{Default vs. BG + DCs}
\begin{tabular}{lcccc}
\toprule
Dataset & Top-3 Jaccard & $1-\mathrm{RBO}$ & $\Delta_{\mathrm{Jaccard}}$ & $\Delta_{\mathrm{RBO}}$ \\
\midrule
SpeedDating & $0.845 \pm 0.139$ & $0.234 \pm 0.114$ & $0.444 \pm 0.139$ & $0.202 \pm 0.114$ \\
TPC-H & $0.593 \pm 0.308$ & $0.238 \pm 0.113$ & $0.192 \pm 0.308$ & $0.133 \pm 0.113$ \\
UW-CSE & $0.335 \pm 0.240$ & $0.162 \pm 0.117$ & $-0.057 \pm 0.240$ & $-0.031 \pm 0.117$ \\
Synthetic & $0.494 \pm 0.338$ & $0.230 \pm 0.152$ & $0.093 \pm 0.338$ & $0.072 \pm 0.152$ \\
\bottomrule
\end{tabular}
\end{subtable}

\vspace{4pt}

\begin{subtable}{\textwidth}
\centering
\caption{Default vs. All}
\begin{tabular}{lcccc}
\toprule
Dataset & Top-3 Jaccard & $1-\mathrm{RBO}$ & $\Delta_{\mathrm{Jaccard}}$ & $\Delta_{\mathrm{RBO}}$ \\
\midrule
SpeedDating & $0.883 \pm 0.122$ & $0.274 \pm 0.118$ & $0.482 \pm 0.122$ & $0.242 \pm 0.118$ \\
TPC-H & $0.622 \pm 0.447$ & $0.247 \pm 0.178$ & $0.222 \pm 0.447$ & $0.142 \pm 0.178$ \\
UW-CSE & $0.372 \pm 0.270$ & $0.165 \pm 0.123$ & $-0.020 \pm 0.270$ & $-0.028 \pm 0.123$ \\
Synthetic & $0.471 \pm 0.185$ & $0.205 \pm 0.074$ & $0.070 \pm 0.185$ & $0.047 \pm 0.074$ \\
\bottomrule
\end{tabular}
\end{subtable}

\vspace{4pt}

\caption{Dataset-wise explanation differences for Default vs. BG, Default vs. BG + DCs, and Default vs. All. We report per-dataset results for the two modes with the largest $\Delta_{\mathrm{Jaccard}}$ and $\Delta_{\mathrm{RBO}}$. Values are mean $\pm$ standard deviation; results for BG + DCs and All are shown only when applicable.}

\label{tab:bg_vs_bgdc_dataset}
\end{table*}

\begin{figure*}[!ht]
  \centering
  \includegraphics[width=\textwidth]{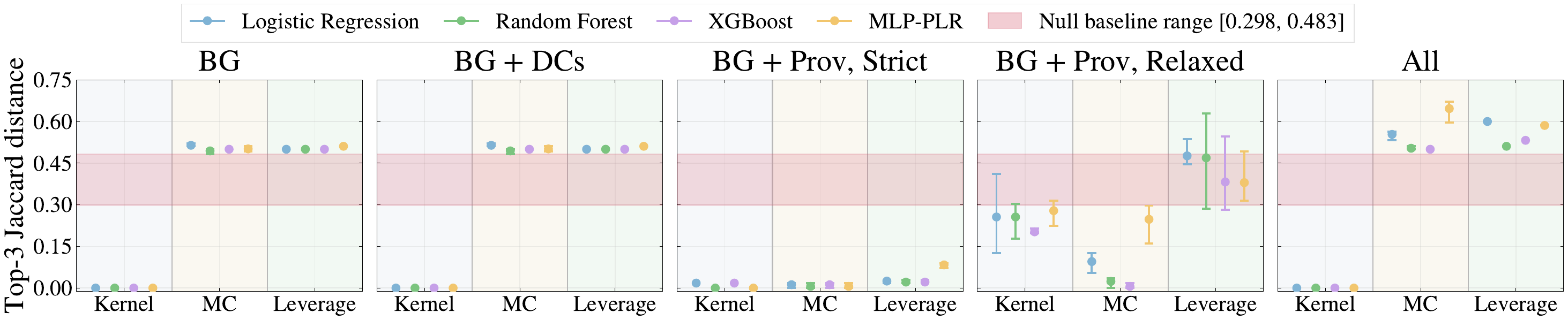}
  \caption{Top-$3$ Jaccard distance between default estimators and RelShap configurations on UW-CSE dataset. Values above the shaded region indicate that RelShap induces larger changes than expected under the randomized null baseline.}
  \label{fig:uwcse_jaccard}
\end{figure*}

\begin{figure}[!ht]
  \centering
  \includegraphics[width=\linewidth]{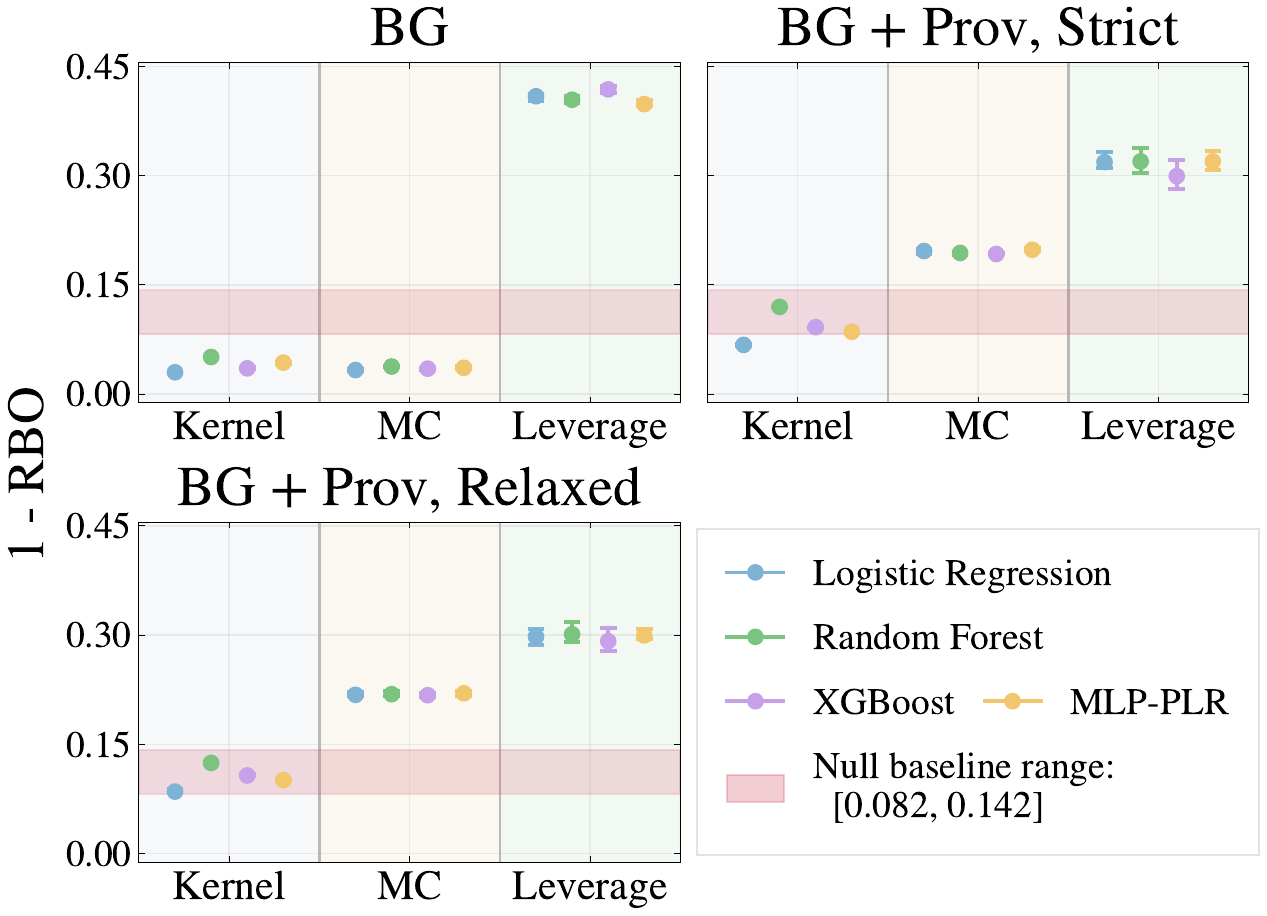}
  \caption{$1-\mathrm{RBO}$ between default estimators and \RelShap configurations on German Credit. Values above the shaded region exceed the randomized null baseline. This dataset has no domain or denial constraints, so those modes are omitted.}
  \label{fig:credit_rbo}
\end{figure}

\begin{figure*}[!ht]
  \centering
  \includegraphics[width=\textwidth]{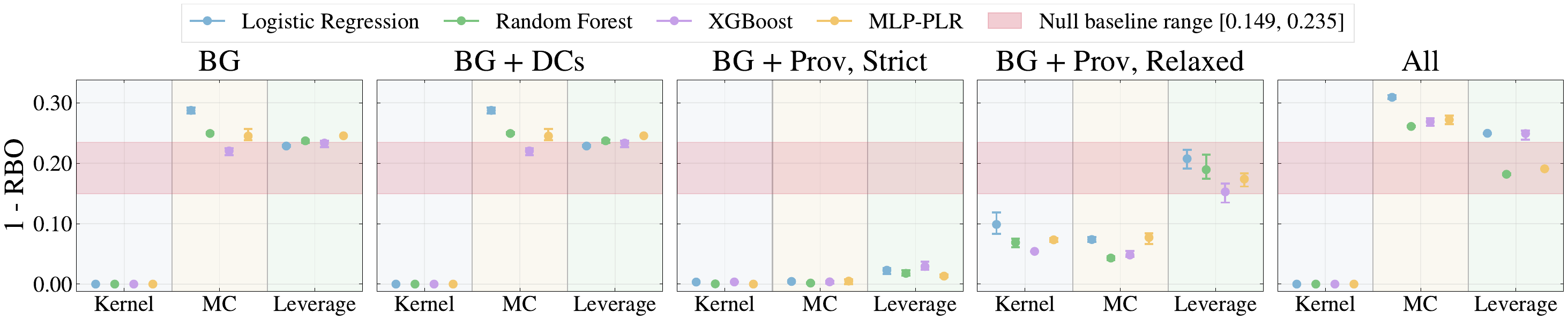}
  \caption{$1-\mathrm{RBO}$ between default estimators and RelShap configurations on UW-CSE dataset. Values above the shaded region indicate that RelShap induces larger changes than expected under the randomized null baseline.}
  \label{fig:uwcse_rbo}
\end{figure*}

\begin{figure}[!ht]
  \centering
  \includegraphics[width=\linewidth]{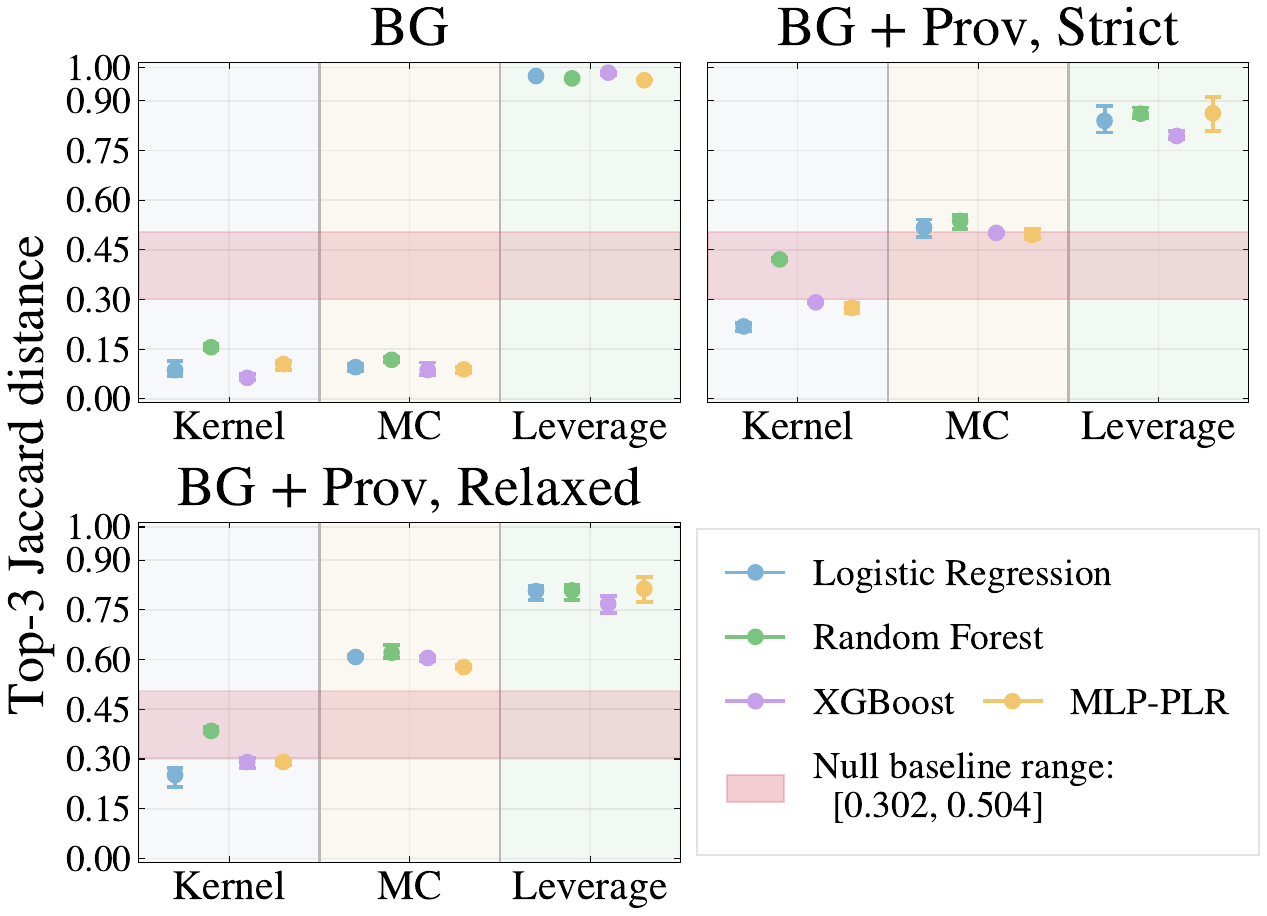}
  \caption{Top-3 Jaccard distance between default estimators and RelShap configurations on German Credit dataset. Values above the shaded region indicate that RelShap induces larger changes than expected under the randomized null baseline. Note that German Credit dataset does not contain any domain or denial constraints and therefore the results are not provided.}
  \label{fig:credit_jaccard}
\end{figure}

Figures~\ref{fig:uwcse_jaccard} and~\ref{fig:credit_rbo} illustrate two complementary regimes. UW-CSE is a small database (DB) dataset ($|F|{=}7$) where we can afford full coalition enumeration ($M_{\mathrm{conv}} = 2^{|F|-1}$), providing an exact comparison; German Credit is a medium-scale ML dataset ($|F|{=}20$) with sampled coalitions.  In both cases, several \RelShap configurations produce distances that exceed the shaded null-baseline region, indicating shifts beyond what random perturbation can explain.

Note that distances do not necessarily increase as more constraint types are added: each configuration changes \emph{which} background points are admissible, and the resulting rankings can move in any direction relative to the default. On UW-CSE, for instance, all features are functionally determined by the dropped identifier \texttt{p\_id}, so provenance-aware mode collapses the background to near-deterministic completions. In this setting, the resulting explanation differences from the default sampler, which draws from the small training set ($|\mathcal{D}_{\mathrm{train}}| = 112$), remain at or below the null baseline. In contrast, on TPC-H, \texttt{suppkey} induces richer cross-table structure over a larger pool ($|\mathcal{D}_{\mathrm{train}}| = 8{,}000$), and provenance-aware mode produces consistently large positive $\Delta$ values (Appendix Table~\ref{tab:prov_tpch_uwcse_xgboost}). The complementary metric pairing---$1{-}\mathrm{RBO}$ on UW-CSE and Top-$3$ Jaccard on German Credit---is in Appendix
Figures~\ref{fig:uwcse_rbo} and~\ref{fig:credit_jaccard}.

\begin{table*}[ht]
\centering
\small
\setlength{\tabcolsep}{1pt}

\textbf{(a) Default vs. BG + Prov}
\vspace{4pt}

\begin{tabular}{llcccc}

\toprule
Dataset & Est. & Top-3 Jaccard & $1-\mathrm{RBO}$ & $\Delta_{\mathrm{Jaccard}}$ & $\Delta_{\mathrm{RBO}}$ \\
\midrule
\multirow{3}{*}{TPC-H}
& Ker & 0.922 / 0.914 & 0.376 / 0.376 & 0.522 / 0.514 & 0.271 / 0.271 \\
& MC   & 0.954 / 0.959 & 0.388 / 0.376 & 0.553 / 0.558 & 0.283 / 0.270 \\
& Lev & 0.000 / 0.000 & 0.000 / 0.000 & -0.401 / -0.401 & -0.105 / -0.105 \\
\midrule
\multirow{3}{*}{UW-CSE}
& Ker & 0.018 / 0.202 & 0.004 / 0.054 & -0.374 / -0.190 & -0.190 / -0.139 \\
& MC   & 0.012 / 0.006 & 0.004 / 0.048 & -0.380 / -0.386 & -0.189 / -0.145 \\
& Lev & 0.021 / 0.382 & 0.029 / 0.153 & -0.371 / -0.010 & -0.164 / -0.040 \\
\midrule
\end{tabular}

\vspace{4pt}

\textbf{(b) BG vs. BG + Prov}

\vspace{4pt}

\begin{tabular}{llcccc}
\toprule
Dataset & Est. & Top-3 Jaccard & $1-\mathrm{RBO}$ & $\Delta_{\mathrm{Jaccard}}$ & $\Delta_{\mathrm{RBO}}$ \\
\midrule
\multirow{3}{*}{TPC-H}
& Ker & 0.914 / 0.924 & 0.356 / 0.361 & 0.513 / 0.523 & 0.250 / 0.255 \\
& Lev & 0.965 / 0.965 & 0.369 / 0.369 & 0.565 / 0.565 & 0.263 / 0.263 \\
& MC   & 0.881 / 0.883 & 0.316 / 0.310 & 0.480 / 0.482 & 0.211 / 0.205 \\
\midrule
\multirow{3}{*}{UW-CSE}
& Ker. & 0.018 / 0.202 & 0.004 / 0.054 & -0.374 / -0.190 & -0.190 / -0.139 \\
& Lev & 0.504 / 0.326 & 0.233 / 0.223 & 0.112 / -0.066 & 0.040 / 0.030 \\
& MC   & 0.504 / 0.500 & 0.224 / 0.242 & 0.112 / 0.108 & 0.031 / 0.049 \\
\bottomrule
\end{tabular}

\caption{Explanation differences for provenance-aware comparisons on TPC-H and UW-CSE with XGBoost. Each entry is reported as Strict / Relaxed mode. Est. denotes Estimator and Ker, MC, and Lev denote Kernel SHAP, Monte Carlo, and Leverage SHAP, respectively.}
\label{tab:prov_tpch_uwcse_xgboost}
\end{table*}

\subsubsection{Computational effects} \label{app:exp:computational_effects} 

We measure the relative reduction in runtime per explanation sample as $(T_{\text{ref}} - T_{\text{target}})/T_{\text{ref}} \times 100\,(\%)$, where $T_{\text{ref}}$ denotes the runtime of the reference method (Default or RelShap (BG)) and $T_{\text{target}}$ denotes that of RelShap (BG/Q). Figure~\ref{fig:runtime_vs_speedup_combined} presents runtime reduction versus the empirical speedup factor. 
Figure~\ref{fig:uwcse_rbo} shows $1-\mathrm{RBO}$ results on the UW-CSE dataset, and Figure~\ref{fig:credit_jaccard} shows Top-3 Jaccard distance on the German Credit dataset, complementing the main text with the opposite pairing.
Figure~\ref{fig:runtime_scatter_app} shows pairwise runtime comparisons across all datasets, models, and comparison modes. Overall, runtime reduction is observed in a large fraction of cases, with Monte Carlo and Leverage SHAP showing particularly strong gains, where reduction occurs in 91.55\% and 84.72\% of cases, respectively, while Kernel SHAP shows a lower rate of 43.06\%. This pattern indicates that the benefit of quotient space coalition projection in RelShap is estimator-dependent.

\begin{figure*}[ht]
  \centering
  \includegraphics[width=\textwidth, height=235pt]{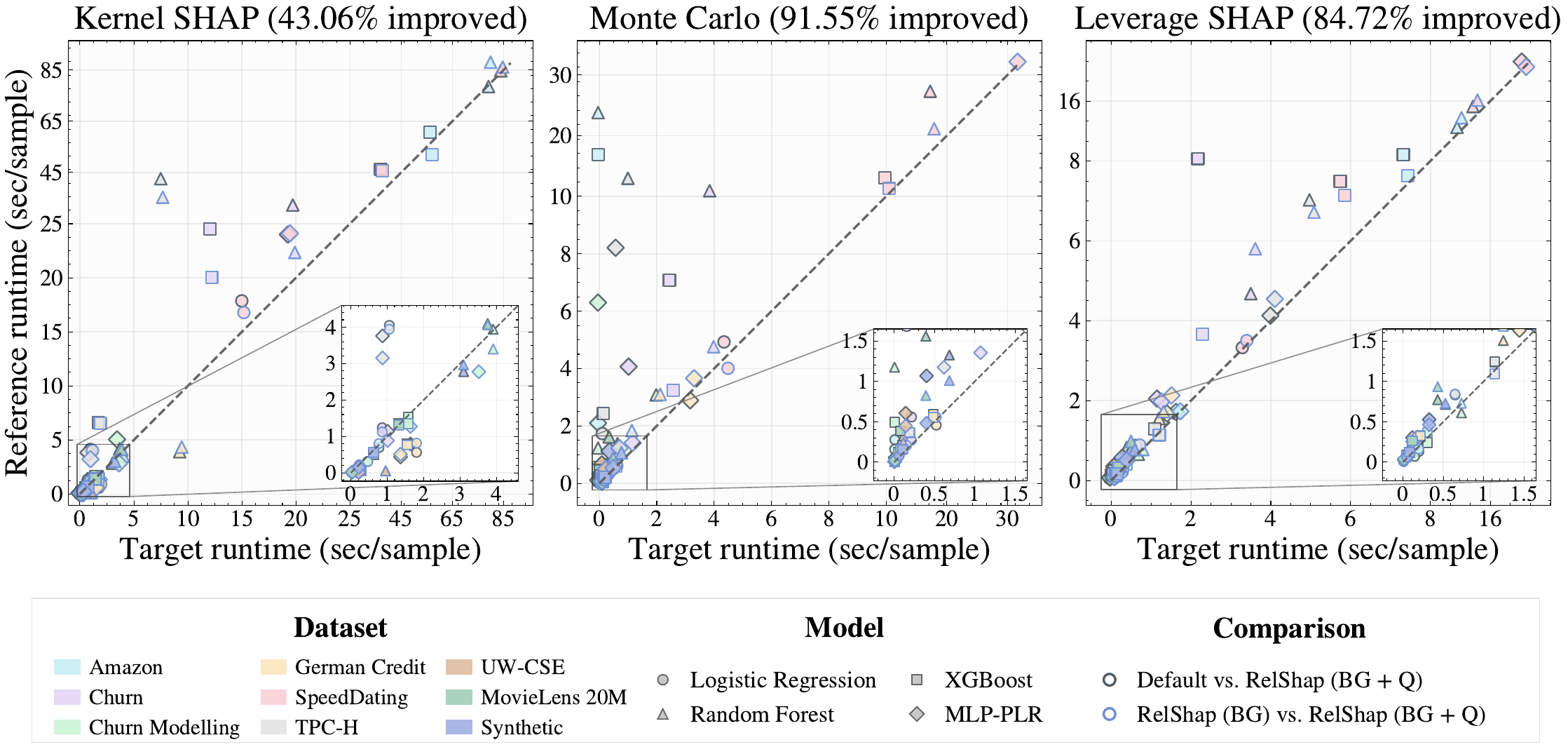}
  \caption{Runtime comparison between the reference method and RelShap (BG/Quotient). Each point corresponds to a dataset–model–configuration pair. Points above the diagonal indicate runtime reduction.}
\label{fig:runtime_scatter_app}
\end{figure*}

\begin{table}[ht]
\centering
\setlength{\tabcolsep}{5pt}
\renewcommand{\arraystretch}{0.95}
\caption{Percentage of cases with runtime reduction across estimators and models.}
\begin{tabular}{lcc}
\toprule
Estimator / Model & Default vs. BG/Q & BG vs. BG/Q \\
\midrule
Kernel SHAP   & 47.22\% & 42.59\% \\
Monte Carlo   & 95.24\% & 85.19\% \\
Leverage SHAP & 81.48\% & 84.26\% \\
\midrule
Logistic Regression & 69.14\% & 64.20\% \\
Random Forest       & 80.25\% & 80.25\% \\
XGBoost             & 85.19\% & 75.31\% \\
MLP-PLR             & 62.82\% & 62.96\% \\
\bottomrule
\end{tabular}
\label{tab:runtime_specifics}
\end{table}

Table~\ref{tab:runtime_by_dataset} summarizes the frequency of runtime reduction across datasets. The runtime reduction is consistently observed for Monte Carlo and Leverage SHAP across most datasets, often exceeding 90\%. In contrast, Kernel SHAP shows higher variability, with little to no reduction on some datasets (e.g., German Credit and UW-CSE) but strong improvements on others (e.g., SpeedDating and TPC-H). Building on the analysis in Section~\ref{sec:algorithms}, we observe that, in general, runtime reduction becomes more pronounced as the total number of FDs increases. In particular, for Kernel SHAP, roughly the benefit tends to emerge only when there is a sufficiently large number of FDs with $|{\mathrm{LHS}}|=1$ (as shown in Table~\ref{tab:constraints_summary}), which is more evident for Kernel SHAP than for the other two estimators.

\begin{figure*}[ht]
\centering

\begin{subfigure}[t]{0.32\textwidth}
  \centering
  \includegraphics[width=\linewidth]{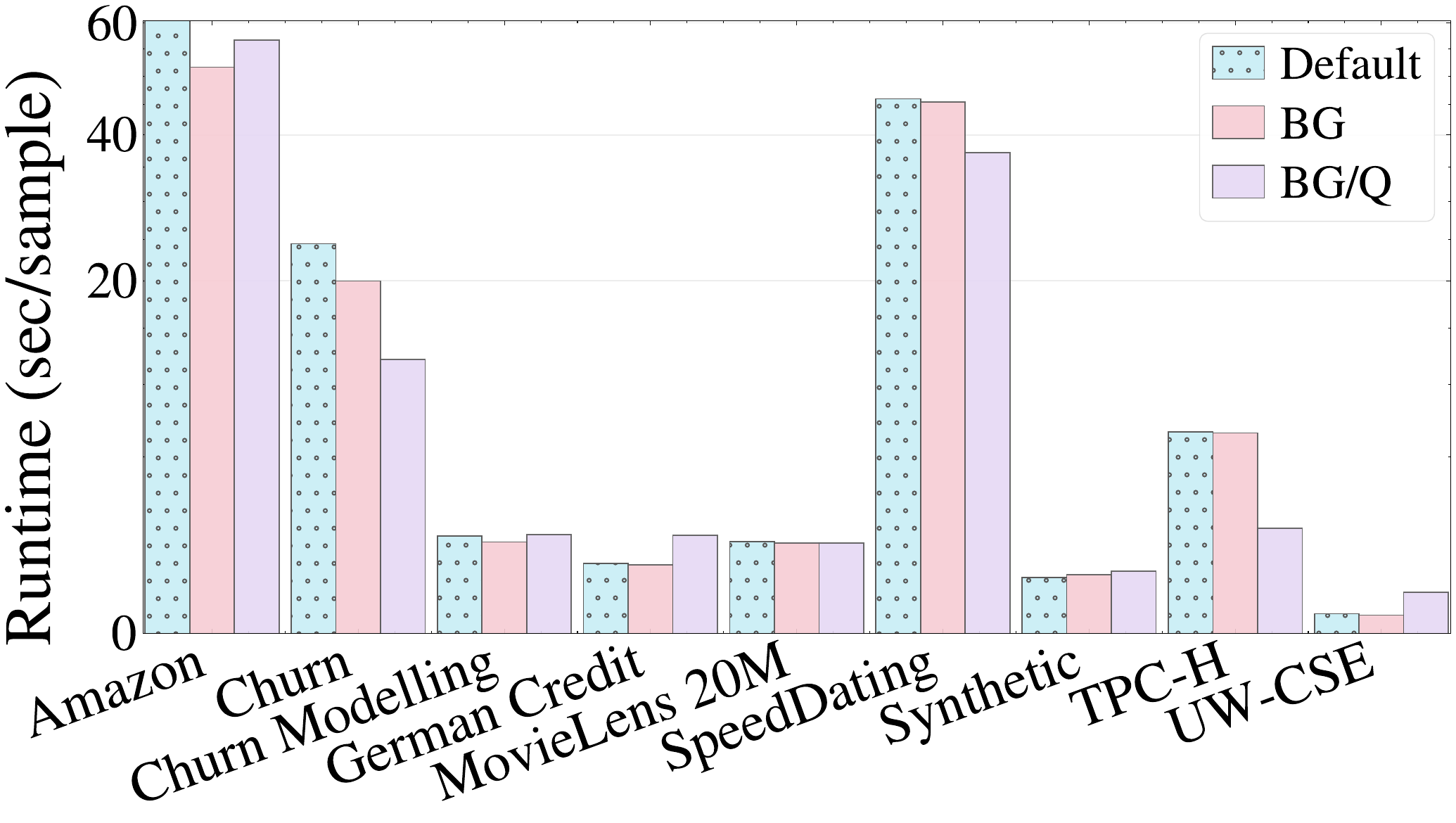}
  \caption{Kernel SHAP.}
\end{subfigure}
\hfill
\begin{subfigure}[t]{0.32\textwidth}
  \centering
  \includegraphics[width=\linewidth]{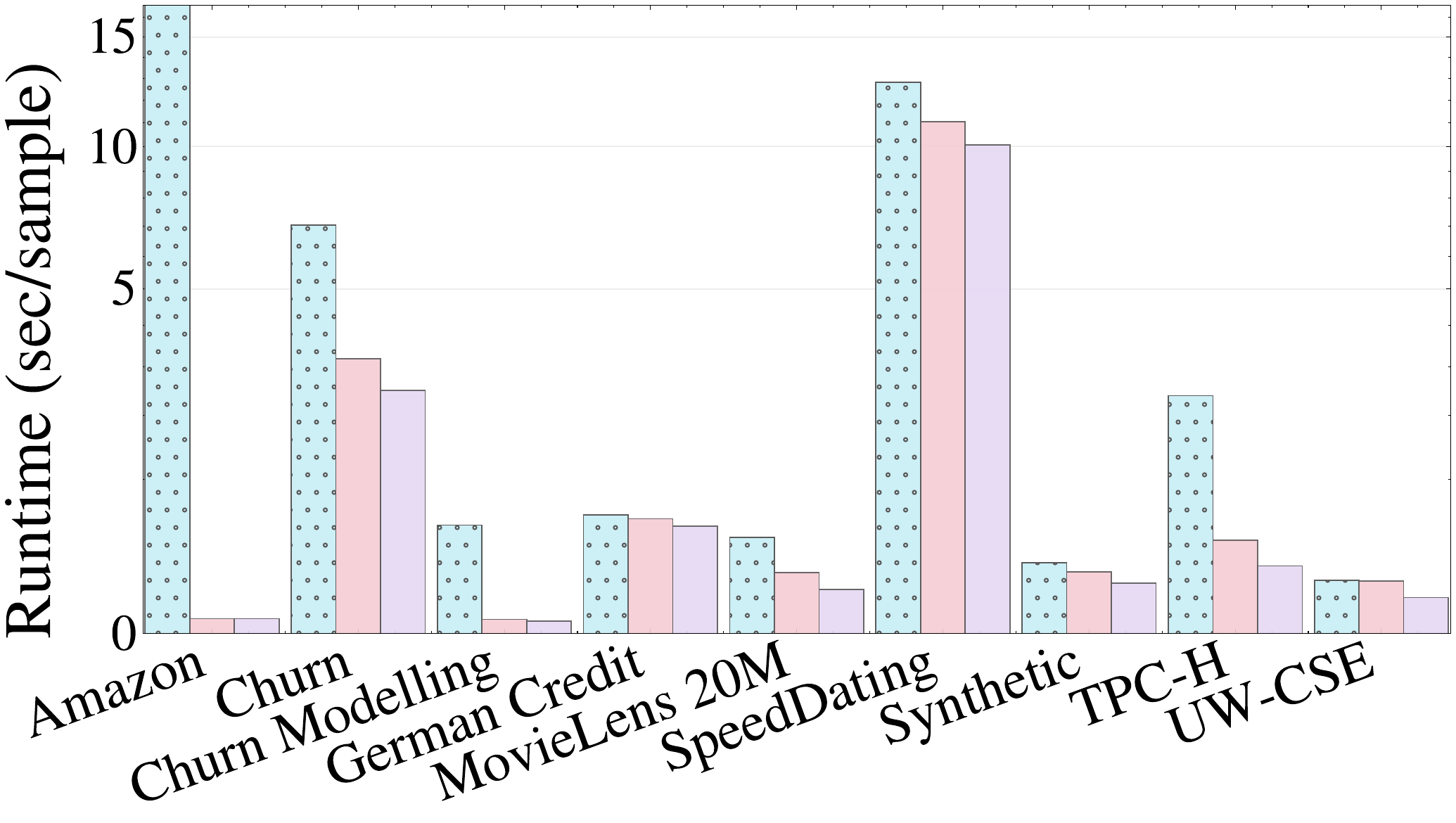}
  \caption{Monte Carlo.}
\end{subfigure}
\hfill
\begin{subfigure}[t]{0.32\textwidth}
  \centering
  \includegraphics[width=\linewidth]{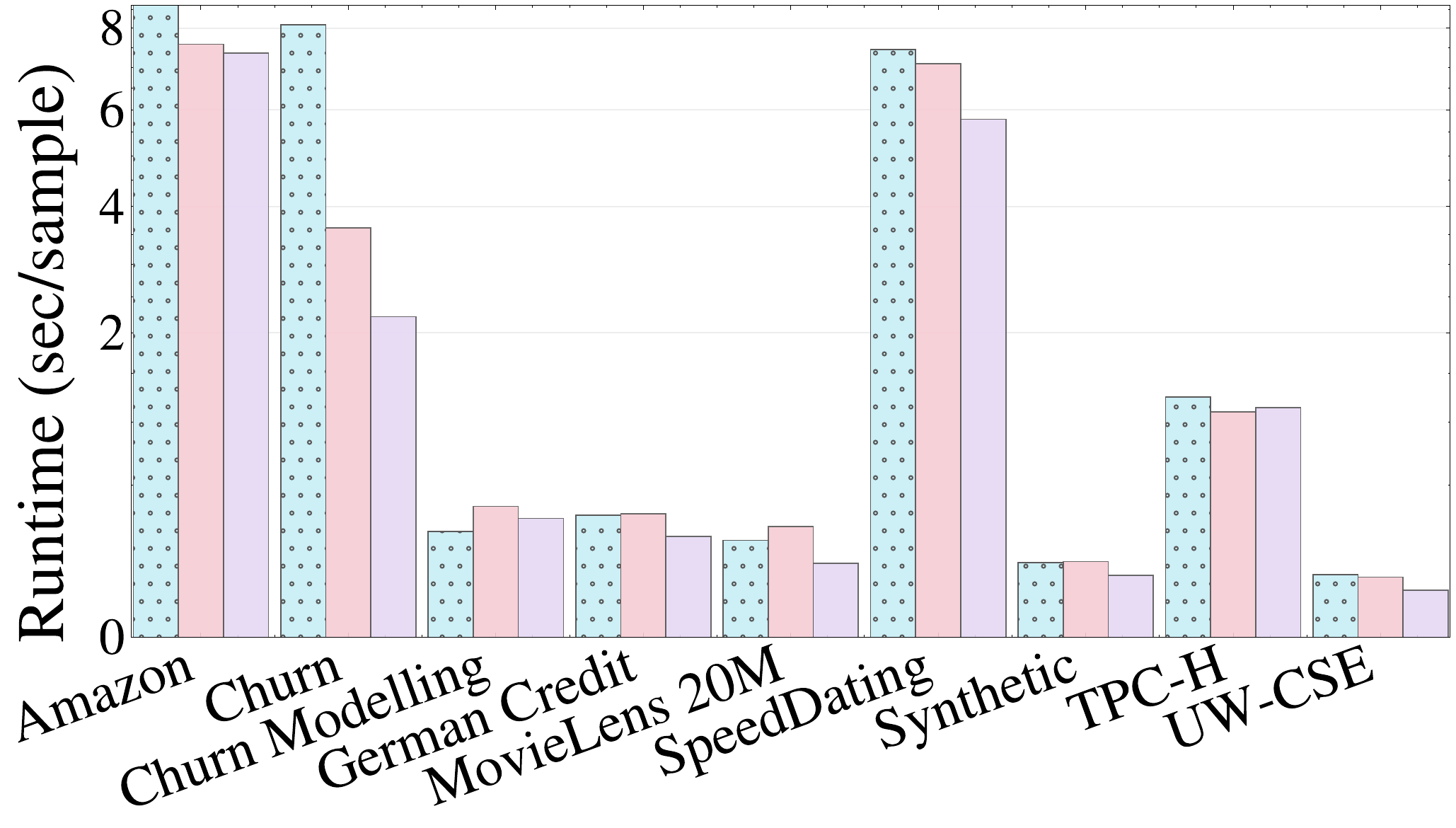}
  \caption{Leverage SHAP.}
\end{subfigure}

\caption{Average running time (sec/sample) across datasets under XGBoost; y-axis uses a square root scaling.}
\label{fig:sec_sample_all}
\end{figure*}

Figure~\ref{fig:sec_sample_all} provides a more detailed view of the actual runtime per test sample across datasets. Table~\ref{tab:runtime_specifics} summarizes the percentage of cases exhibiting runtime reduction across estimators and models. Monte Carlo consistently shows the highest reduction rates (95.24\% and 85.19\%), followed by Leverage SHAP, while Kernel SHAP exhibits more moderate gains. These results are expected as described in Section~\ref{sec:algorithms}. Across models, the variation is relatively small compared to differences across estimators, although Random Forest and XGBoost exhibit slightly higher reduction rates. This pattern indicates that the benefit of quotient space coalition projection in RelShap is estimator-dependent.

\begin{table*}[ht]
\centering
\small
\setlength{\tabcolsep}{3pt}
\renewcommand{\arraystretch}{0.95}
\caption{Running time (sec/sample) under XGBoost for datasets  without denial constraints (mean over three seeds).  These  datasets do not admit the full-featured configuration (BG/Q~+~Prov, Relaxed~+~DCs); see Table~\ref{tab:runtime_modes_b} for datasets with DCs.  When provenance is active, quotient projection deduplicates using both global and provenance-induced FDs; however, provenance lookups add their own overhead, so BG/Q~+~Prov is not guaranteed to be faster than BG/Q alone.  Dashes (-) indicate inapplicable configurations  (Table~\ref{tab:constraints_summary}).  See Appendix~Figure~\ref{fig:sec_sample_all} for a visual comparison. Here, Ker, MC, and Lev denote Kernel SHAP, Monte Carlo, and Leverage SHAP, respectively. Prov, St/Rel denotes Prov, Strict/Relaxed mode.}
\label{tab:runtime_modes_a}
\begin{tabular}{l ccc ccc ccc ccc ccc}
\toprule
& \multicolumn{3}{c}{Amazon}
& \multicolumn{3}{c}{Churn}
& \multicolumn{3}{c}{Churn Modelling}
& \multicolumn{3}{c}{German Credit}
& \multicolumn{3}{c}{MovieLens 20M} \\
\cmidrule(lr){2-4}\cmidrule(lr){5-7}\cmidrule(lr){8-10}
\cmidrule(lr){11-13}\cmidrule(lr){14-16}
Mode
& Ker & MC & Lev
& Ker & MC & Lev
& Ker & MC & Lev
& Ker & MC & Lev
& Ker & MC & Lev \\
\midrule
Default
& 60.389 & 16.642 & 8.612
& 24.425 &  7.026 & 8.084
&  1.523 &  0.493 & 0.241
&  0.793 &  0.590 & 0.321
&  1.362 &  0.387 & 0.202 \\
BG
& 51.599 &  0.009 & 7.584
& 19.977 &  3.181 & 3.616
&  1.346 &  0.008 & 0.369
&  0.759 &  0.554 & 0.327
&  1.309 &  0.157 & 0.264 \\
BG/Q
& 56.624 &  0.009 & 7.362
& 12.094 &  2.494 & 2.216
&  1.576 &  0.006 & 0.305
&  1.550 &  0.485 & 0.218
&  1.317 &  0.082 & 0.117 \\
\midrule
BG~+~Prov, Strict
& 52.411 & 0.005 & 7.832
& 22.490 & 0.865 & 5.015
& 1.263 & 0.004 & 0.288
& 0.742 & 0.220 & 0.255
& - & - & - \\
BG/Q~+~Prov, Strict
& 52.725 & 16.360 & 8.739
& 17.261 &  0.782 & 7.146
&  1.292 &  0.498 & 0.245
&  0.913 &  0.260 & 0.306
&    -   &     -  &   -   \\
BG~+~Prov, Relaxed
& 55.019 & 0.006 & 6.715
& 21.232 & 0.706 & 6.182
& 1.286 & 0.004 & 0.290
& 0.714 & 0.114 & 0.314
& - & - & - \\
BG/Q~+~Prov, Relaxed
& 56.626 & 16.160 & 7.703
& 20.260 &  0.797 & 5.771
&  1.626 &  0.470 & 0.262
&  0.973 &  0.157 & 0.342
&    -   &    -   &    -  \\
\bottomrule
\end{tabular}
\end{table*}

\begin{table*}[ht]
\centering
\setlength{\tabcolsep}{3pt}
\renewcommand{\arraystretch}{0.95}
\caption{Running time (sec/sample) under XGBoost for datasets with denial constraints (mean over three seeds).  The last row shows the full-featured \RelShap configuration.}
\label{tab:runtime_modes_b}
\begin{tabular}{l ccc ccc ccc ccc}
\toprule
& \multicolumn{3}{c}{SpeedDating}
& \multicolumn{3}{c}{TPC-H}
& \multicolumn{3}{c}{UW-CSE}
& \multicolumn{3}{c}{Synthetic} \\
\cmidrule(lr){2-4}\cmidrule(lr){5-7}\cmidrule(lr){8-10}
\cmidrule(lr){11-13}
Mode
& Ker & MC & Lev
& Ker & MC & Lev
& Ker & MC & Lev
& Ker & MC & Lev \\
\midrule
Default
& 45.985 & 12.817 & 7.447
&  6.527 &  2.383 & 1.244
&  0.062 &  0.119 & 0.084
&  0.505 &  0.212 & 0.119 \\
BG
& 45.408 & 11.038 & 7.098
&  6.464 &  0.366 & 1.094
&  0.054 &  0.115 & 0.078
&  0.558 &  0.160 & 0.123 \\
BG/Q
& 37.227 & 10.068 & 5.785
&  1.777 &  0.192 & 1.138
&  0.271 &  0.054 & 0.048
&  0.622 &  0.106 & 0.082 \\
\midrule
BG~+~DCs
& 225.200 & 64.119 & 20.932
&   9.132 &  0.651 &  2.658
&   0.034 &  0.019 &  0.015
&   2.248 &  0.685 &  0.628 \\
BG/Q~+~DCs
& 244.800 & 73.635 & 22.820
&   2.988 &  0.377 &  2.364
&   0.064 &  0.011 &  0.012
&   2.587 &  0.553 &  0.333 \\
\midrule
BG~+~Prov, Strict
& 38.582 & 0.683 & 7.530
& 7.478 & 0.315 & 2.001
& 0.007 & 0.012 & 0.010
& 0.519 & 0.063 & 0.109 \\
BG/Q~+~Prov, Strict
& 36.976 &  1.006 & 9.491
&  6.167 &  0.142 & 1.725
&  0.066 &  0.108 & 0.076
&  0.483 &  0.077 & 0.124 \\
BG~+~Prov, Relaxed
& 42.115 & 0.731 & 10.293
& 6.660 & 0.171 & 1.971
& 0.008 & 0.010 & 0.012
& 0.469 & 0.073 & 0.125 \\
BG/Q~+~Prov, Relaxed
& 40.604 &  0.872 & 9.032
&  7.598 &  0.182 & 1.734
&  0.056 &  0.072 & 0.094
&  0.498 &  0.083 & 0.145 \\
\midrule
All
& 43.558 & 3.232 & 10.649
&  9.239 &  0.297 & 3.589
&  0.035 &  0.014 & 0.013
&  2.481 &  0.442 & 0.606 \\
All/Q
& 42.117 &  4.043 & 9.752
&  2.682 &  0.091 & 1.197
&  0.474 &  0.034 & 0.049
&  2.521 &  0.200 & 0.455 \\
\bottomrule
\end{tabular}
\end{table*}

Tables~\ref{tab:runtime_modes_a} and~\ref{tab:runtime_modes_b} provide a per-mode breakdown. Several patterns emerge. First, BG alone already reduces runtime for Monte Carlo on several datasets (\eg Amazon drops from $16.642$ to $0.009$\,sec per sample, and TPC-H from $2.383$ to $0.366$). Adding quotient projection (BG/Q) yields further gains, particularly for Kernel SHAP on datasets with rich FD structure: on TPC-H, Kernel SHAP drops from $6.527$ (Default) to $1.777$ (BG/Q), and on Churn from $24.425$ to $12.094$. Second, the benefit is estimator-dependent: Monte Carlo and Leverage SHAP show consistent reductions across datasets, while Kernel SHAP improves primarily on datasets with many FDs having $|\mathrm{LHS}|{=}1$ (Table~\ref{tab:constraints_summary}). On datasets with predominantly $|\mathrm{LHS}|{=}2$ FDs (\eg German Credit), Kernel SHAP shows little or no improvement, and BG/Q can even incur slight overhead from the canonicalization step without sufficient deduplication to compensate. Third, provenance-aware modes (BG/Q~+~Prov) exhibit dataset-dependent behavior: on TPC-H, Monte Carlo benefits substantially ($0.142$\,sec/sample under Prov, Strict vs.\ $0.192$ under BG/Q), while on other datasets the additional provenance lookups can offset the gains from deduplication. The full-featured mode (BG/Q~+~Prov, Relaxed~+~DCs), available on datasets with all constraint types, shows competitive runtimes: on TPC-H, Monte Carlo achieves $0.091$\,sec/sample, a $26{\times}$ reduction over Default. Across models, variation is small compared to differences across estimators, confirming that the runtime benefit is primarily estimator-driven. It is noteworthy that since DCs may introduce additional lookup costs and filtering costs, DC and All modes have a probability of slowing down the computations, even with or without Q mode. Therefore, quotient mode provides runtime benefits, but this does not always happen. Q can reduce runtime with rich FD structures desirable estimators; otherwise, overhead may dominate.

\begin{table}[hb]
\centering
\caption{Percentage of cases showing runtime reduction across datasets and estimators, including Default vs. RelShap (BG) and RelShap (BG) vs. RelShap (BG/Q). The highest value in each dataset (row) is highlighted in bold.}
\begin{tabular}{lccc}
\toprule
Dataset & Kernel & MC & Leverage \\
\midrule
Amazon & 33.33\% & \textbf{75.00\%} & 50.00\% \\
Churn & 87.50\% & 95.83\% & \textbf{100.0\%} \\
Churn Modelling & 16.67\% & \textbf{91.67\%} & 54.17\% \\
German Credit & 0.000\% & 75.00\% & \textbf{91.67\%} \\
SpeedDating & \textbf{91.67\%} & 76.19\% & 79.17\% \\
TPC-H & \textbf{100.0\%} & \textbf{100.0\%} & 79.17\% \\
UW-CSE & 0.000\% & \textbf{100.0\%} & 91.67\% \\
MovieLens 20M & 41.67\% & \textbf{100.0\%} & \textbf{100.0\%} \\
Synthetic & 33.33\% & 95.83\% & \textbf{100.0\%} \\
\bottomrule
\end{tabular}
\label{tab:runtime_by_dataset}
\end{table}

\paragraph{Scalability} \label{app:exp:scalability} 

\begin{figure*}[t]
  \centering

  \begin{subfigure}[t]{0.49\textwidth}
    \centering
    \includegraphics[width=\linewidth]{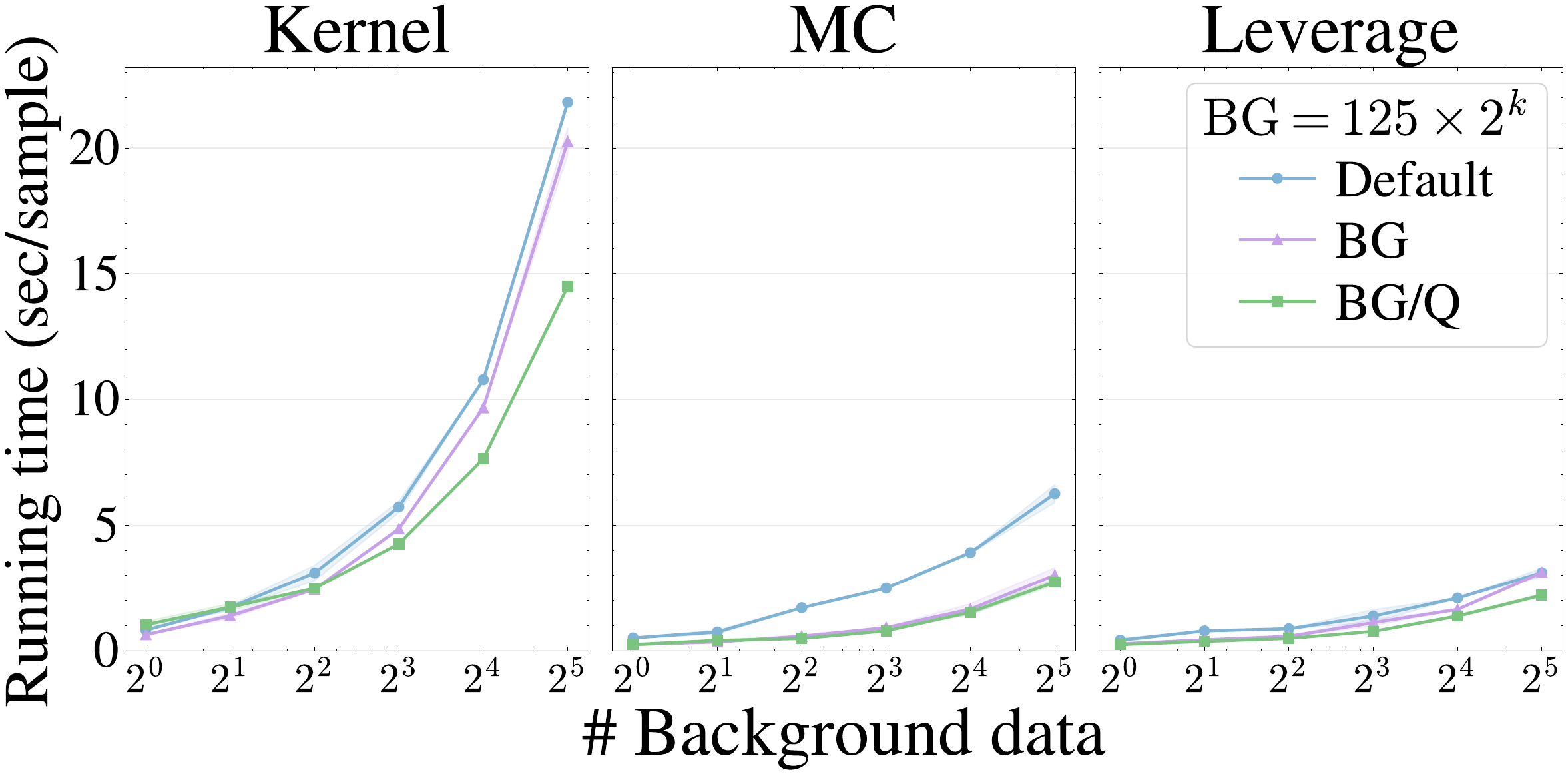}
    \caption{Running time as the number of background data points varies for Churn.}
    \label{fig:scale_churn_bg}
  \end{subfigure}\hfill
  \begin{subfigure}[t]{0.49\textwidth}
    \centering
    \includegraphics[width=\linewidth]{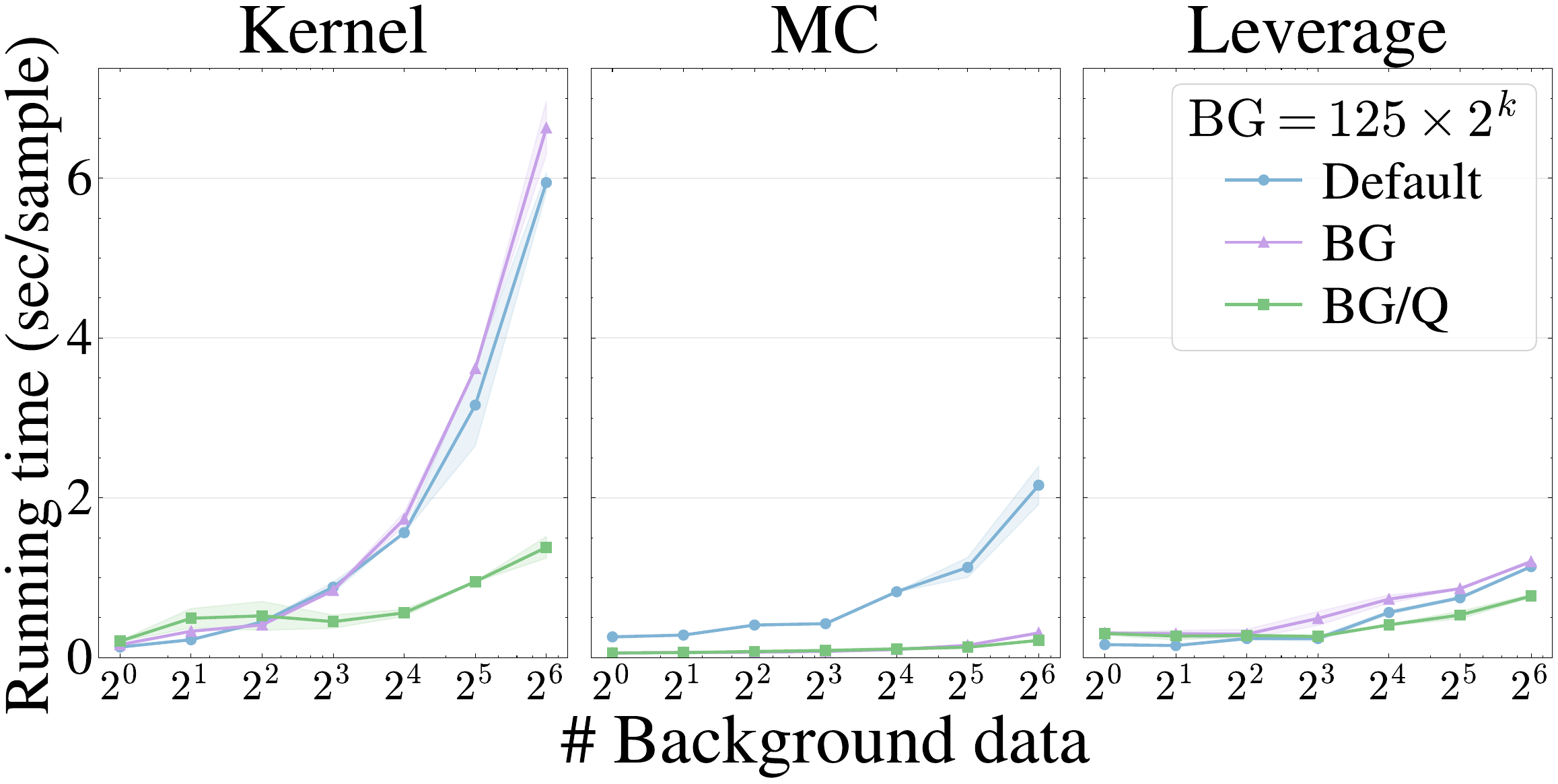}
    \caption{Running time as the number of background data points varies for TPC-H.}
    \label{fig:scale_tpch_bg}
  \end{subfigure}

  \vspace{0.5em}

  \begin{subfigure}[t]{0.49\textwidth}
    \centering
    \includegraphics[width=\linewidth]{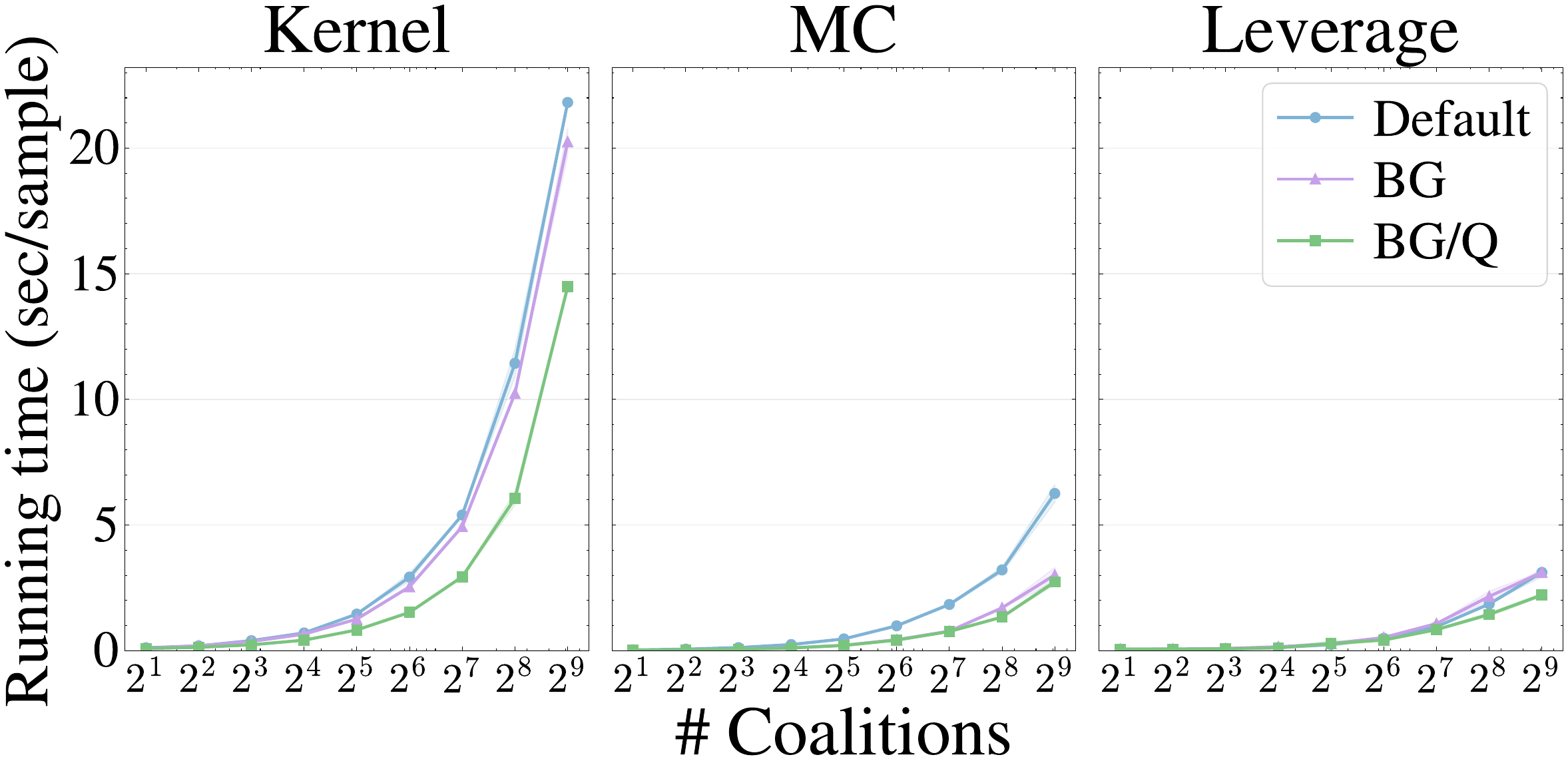}
    \caption{Running time as the number of coalitions varies for Churn.}
    \label{fig:scale_churn_coal}
  \end{subfigure}\hfill
  \begin{subfigure}[t]{0.49\textwidth}
    \centering
    \includegraphics[width=\linewidth]{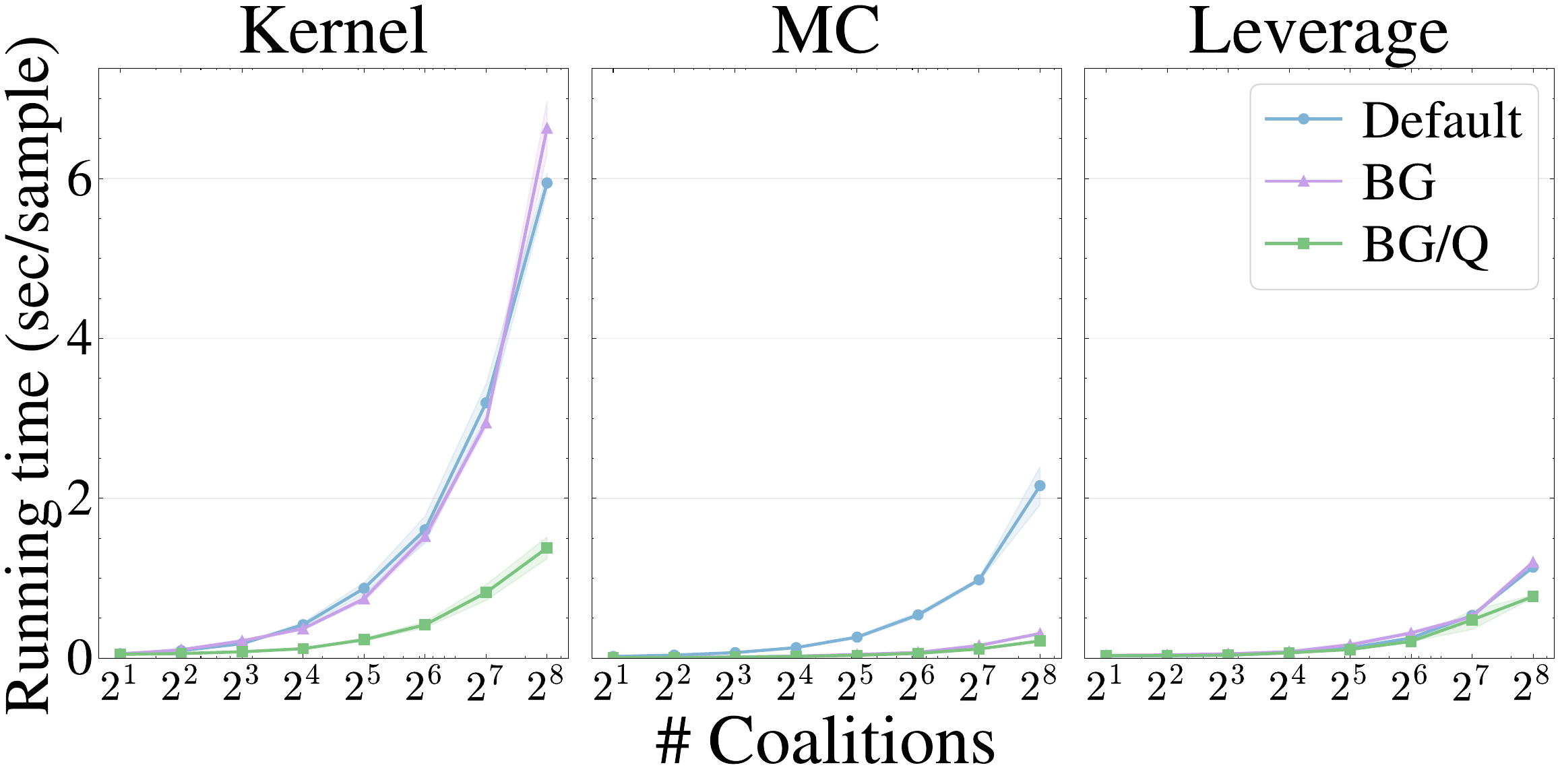}
    \caption{Running time as the number of coalitions varies for TPC-H.}
    \label{fig:scale_tpch_coal}
  \end{subfigure}

  \caption{Average running time (sec/sample) as the number of background data points varies up to the per-dataset maximum $B$ and the number of coalitions varies up to the convergence budget $M_{\mathrm{conv}}$ in Table~\ref{tab:dataset_stats}, for Churn and TPC-H (log scale on the x-axis).}
  \label{fig:scalability}
\end{figure*}

Figure~\ref{fig:scalability} shows running time per sample as the number of background data points (Churn) and coalitions (TPC-H) increase. On Churn (Figure~\ref{fig:scale_churn_bg}), all three methods grow with background size, but the gap between Default and BG/Q widens as $B$ increases: for Kernel SHAP, BG/Q remains roughly flat while Default grows super-linearly; for Monte Carlo, BG alone already provides a large reduction that BG/Q maintains; for Leverage SHAP, the separation is moderate but consistent. 

On TPC-H (Figure~\ref{fig:scale_tpch_coal}), runtime grows with the coalition budget $M$ for all methods, but BG/Q maintains a substantial gap below Default across the full range for Kernel SHAP and Monte Carlo. For Leverage SHAP, the three curves are closer together, reflecting the smaller number of collapsible coalitions under this estimator's sampling distribution. The complementary pairing in the Appendix confirms these patterns: TPC-H vs.\ background data (Appendix~Figure~\ref{fig:scale_tpch_bg}) shows the same widening gap, with BG/Q providing the largest separation for Kernel SHAP and MC; Churn vs.\ coalitions (Appendix~Figure~\ref{fig:scale_churn_coal}) shows all three curves growing together but with BG/Q consistently below Default, and the gap widening at larger coalition budgets. Overall, the reduction from BG/Q grows with problem size across both axes, indicating that \RelShap becomes more effective with larger data and coalition budgets.

\begin{figure*}[ht]
  \centering

  \begin{subfigure}[t]{0.49\textwidth}
    \centering
    \includegraphics[width=\linewidth]{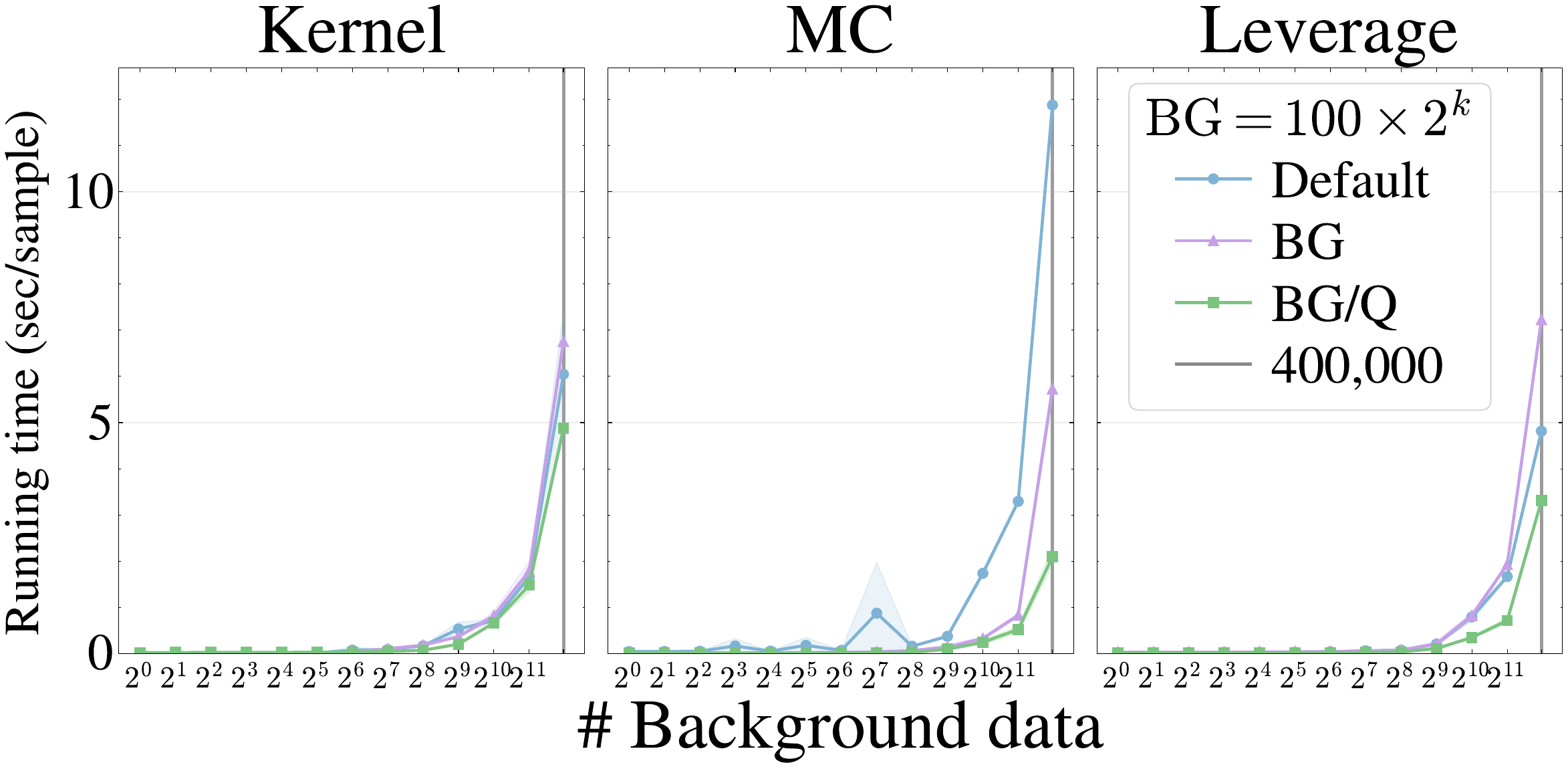}
    \caption{Running time as the number of background data points varies for MovieLens-20M with MLP-PLR.}
    \label{fig:scalability_movielens_bg}
  \end{subfigure}\hfill
  \begin{subfigure}[t]{0.49\textwidth}
    \centering
    \includegraphics[width=\linewidth]{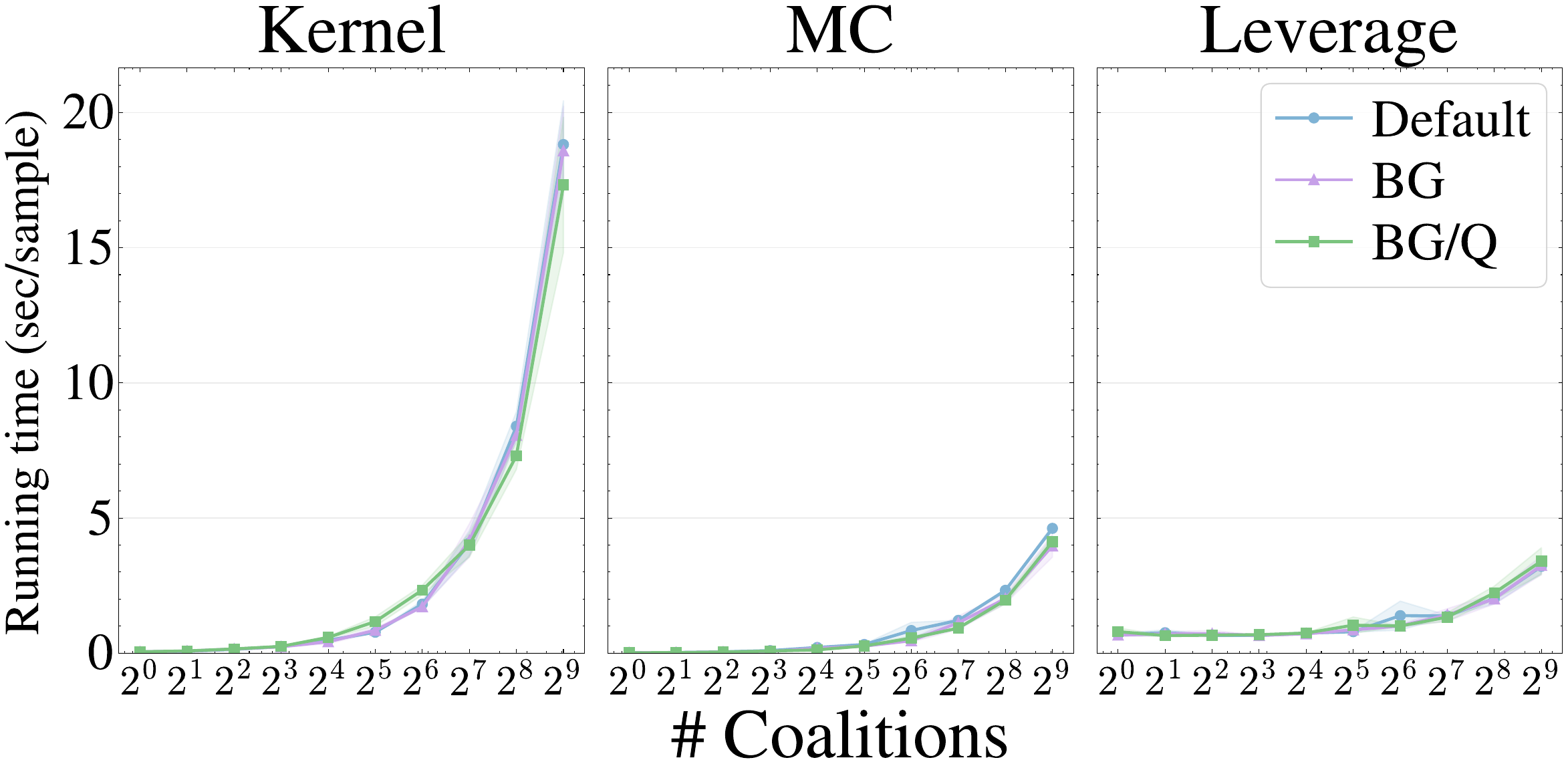}
    \caption{Running time as the number of coalitions varies for SpeedDating with logistic regression.}
    \label{fig:scalability_speeddating_coal}
  \end{subfigure}

  \caption{Average running time (sec/sample) as the number of background data points and coalitions varies for additional model--dataset pairs.}
  \label{fig:scalability_additional}
\end{figure*}

\begin{figure*}[ht]
  \centering

  \begin{subfigure}[t]{0.49\textwidth}
    \centering
    \includegraphics[width=\linewidth]{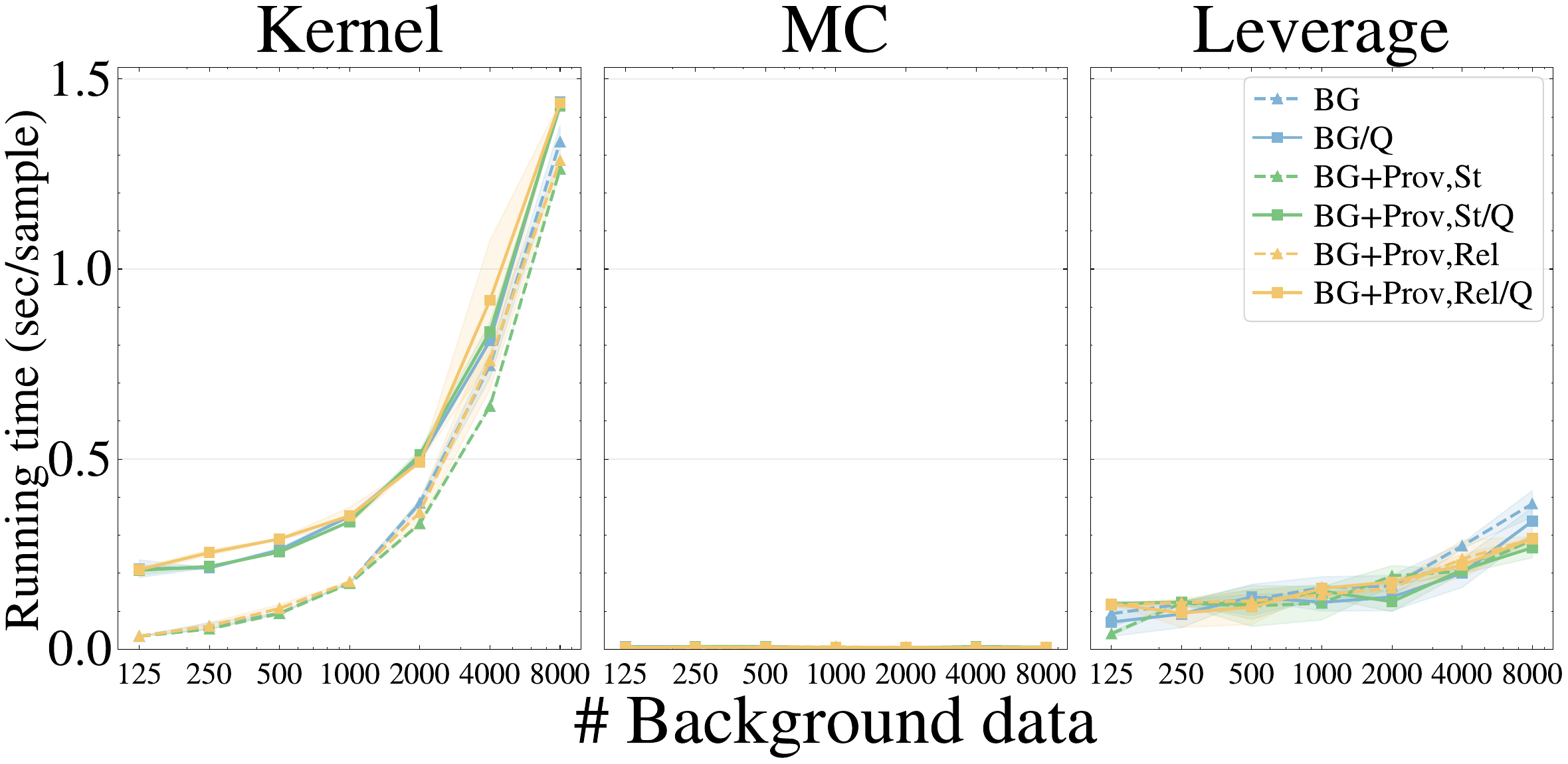}
    \caption{Running time as the number of background data points varies for Churn Modelling with XGBoost under richer provenance settings.}
    \label{fig:scalability_churnmodelling_bg}
  \end{subfigure}\hfill
  \begin{subfigure}[t]{0.49\textwidth}
    \centering
    \includegraphics[width=\linewidth]{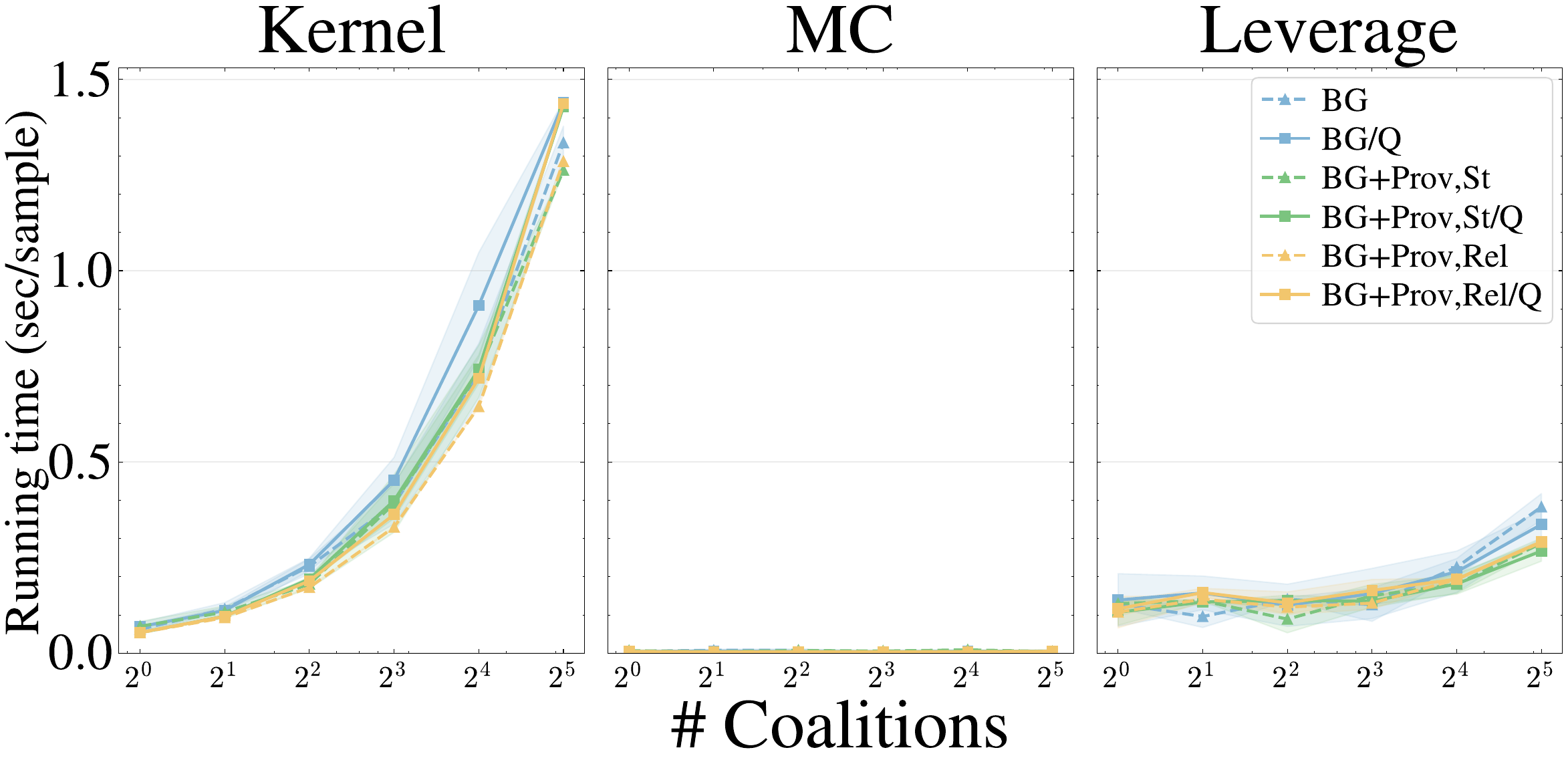}
    \caption{Running time as the number of coalitions varies for Churn Modelling with XGBoost under richer provenance settings.}
    \label{fig:scalability_churnmodelling_coal}
  \end{subfigure}

  \caption{Average running time (sec/sample) as the number of background data points and coalitions varies under richer provenance settings for Churn Modelling with XGBoost.}
  \label{fig:scalability_richer}
\end{figure*}

Additional model--dataset pairs and richer constraint settings exhibit the same overall scaling behavior, although the magnitude of the gain depends on the estimator and the available coalition redundancy. On MovieLens 20M dataset with MLP-PLR (Figure~\ref{fig:scalability_movielens_bg}), BG/Q increasingly separates from Default as the background size grows, with particularly large reductions for Monte Carlo and Leverage SHAP and a smaller but clear reduction for Kernel SHAP. On SpeedDating dataset with logistic regression (Figure~\ref{fig:scalability_speeddating_coal}), all configurations grow similarly with the coalition budget, although BG/Q remains consistently below Default at larger $M$. Churn Modelling dataset with XGBoost (Figure~\ref{fig:scalability_richer}) further shows that richer provenance settings preserve the same scaling trends across both axes: running times increase with background size and coalition budget. This suggests that enriching the relational semantics does not fundamentally alter the scalability profile: runtime is driven primarily by the estimator, evaluation budget, and the quotient mode, allowing richer constraint-aware explanations to be incorporated without introducing a qualitatively different scaling behavior.

\paragraph{Alignment with combinatorial analysis} \label{app:exp:alignment_combinatorial}

We measure this using the empirical counterpart of $R_M$ in Thm.~\ref{thm:reduction}, $\hat{R}_{M_{\mathrm{conv}}}
= 1 - 1/{\hat{S}_{M_{\mathrm{conv}}}}$ where
$
\hat{S}_{M_{\mathrm{conv}}}
=\frac{1}{n_{\text{explain}}}
\sum_{n_{\text{explain}}}
\frac{1}{B}
\sum_{b=1}^{B}
\frac{M_{\mathrm{conv}}}{K^{\text{quot}}_{b}}
$
and $K^{\text{quot}}_{b}$ is the number of distinct equivalence classes after quotient mode. Figure~\ref{fig:runtime_vs_speedup_combined} shows an alignment of real runtime reduction and our combinatorial analysis in Thm.~\ref{thm:reduction}. This confirms that our analysis reliably predicts \emph{when} reduction will occur, although its exact magnitude remains contingent on the dataset and sampling strategy. 

\begin{figure*}[ht]
  \centering

  \begin{subfigure}[t]{\textwidth}
    \centering
    \includegraphics[width=\textwidth]{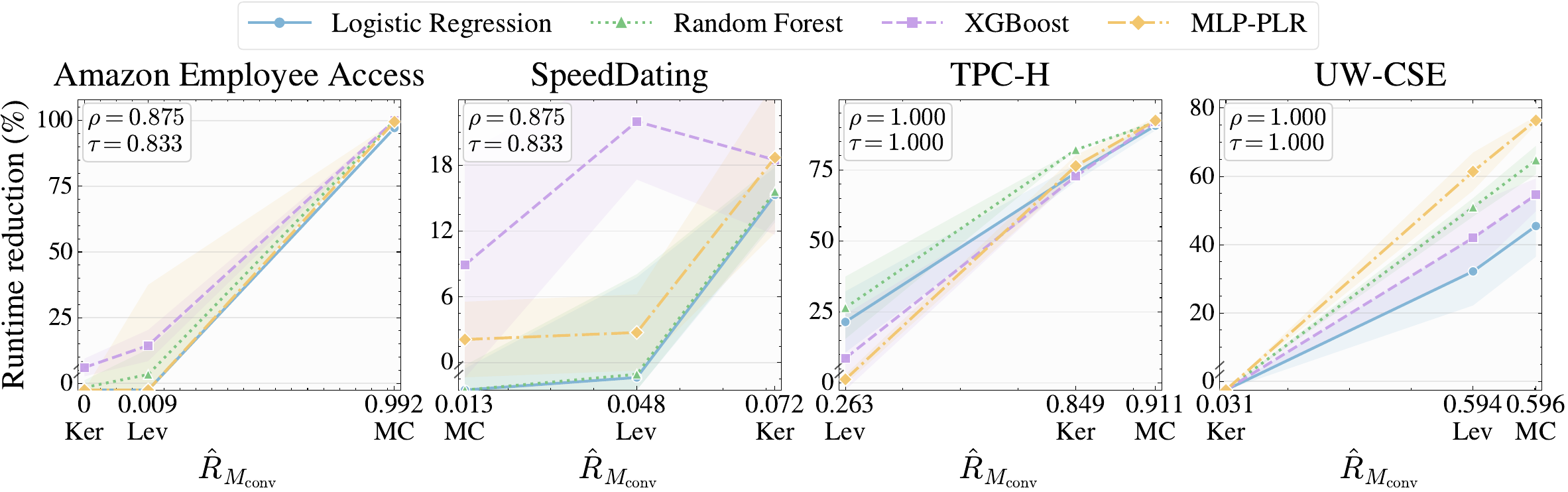}
    \label{fig:runtime_vs_speedup_main}
  \end{subfigure}

  \begin{subfigure}[t]{\textwidth}
    \centering
    \includegraphics[width=\textwidth]{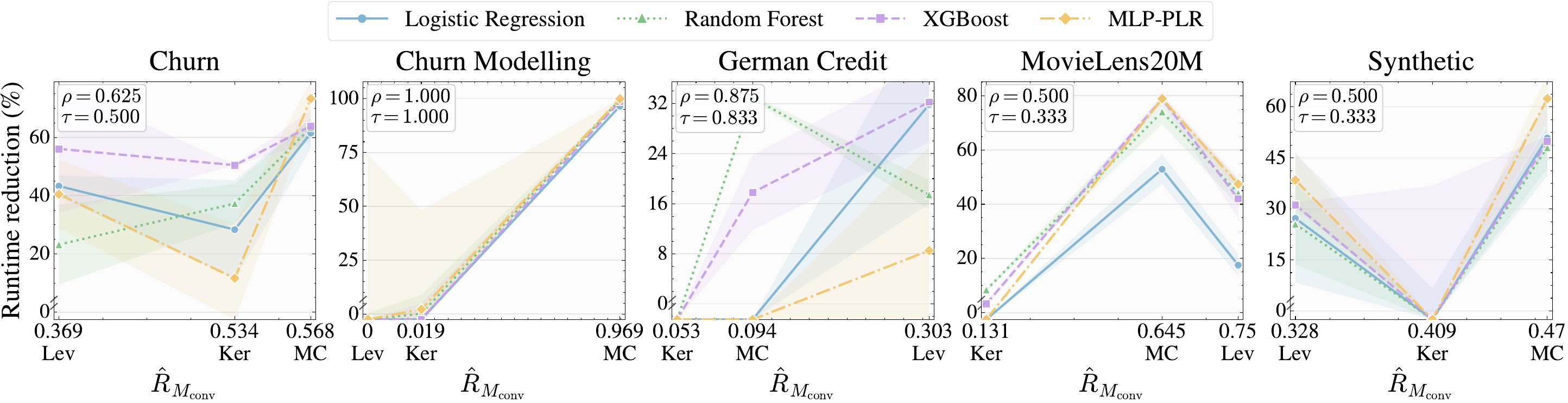}
    \label{fig:runtime_vs_speedup_appendix}
  \end{subfigure}

  \caption{Runtime reduction (\%) versus the empirical runtime speedup factor $\hat{R}_{M_{\mathrm{conv}}}$ across datasets and models (Default vs. RelShap (BG)). Each point corresponds to a configuration at $M_{\mathrm{conv}}$. Shaded regions indicate variability across seeds. Ker, MC, and Lev denote Kernel SHAP, Monte Carlo, and Leverage SHAP.}
  \label{fig:runtime_vs_speedup_combined}
\end{figure*}

\subsubsection{Detailed analysis on provenance-aware mode} \label{app:exp:detailed_provenance}

\textbf{TPC-H.} Schema: As the primary key of the \texttt{Supplier} table, \texttt{suppkey} is referenced by \texttt{Lineitem} and \texttt{Partsupp} tables, and further connected to \texttt{Orders} through join paths, inducing supplier-level FDs across related records. Query: Grouping on \texttt{suppkey} produces one tuple per supplier with aggregate statistics (\eg counts, sums, and averages), making \texttt{suppkey} functionally determine all features. \texttt{suppkey} is removed in the final ML table and utilized in the provenance-aware mode. 
\textbf{UW-CSE.} Schema: The identifier \texttt{p\_id} in the \texttt{Person} table propagates through relational links such as \texttt{AdvisedBy} and \texttt{TaughtBy}, capturing FDs at the student level. Query: Aggregating over \texttt{p\_id} yields one tuple per student with relational statistics (e.g., advisor counts and course-level features), so all features are functionally determined by \texttt{p\_id}. \texttt{p\_id} is finally removed in the ML table and used in the provenance-aware mode.

\subsection{Controlled validation of \RelShap} \label{app:exp_intuition}

\subsubsection{Validation under synthetic settings}

\begin{figure*}[ht]
  \centering
  \includegraphics[width=\linewidth]{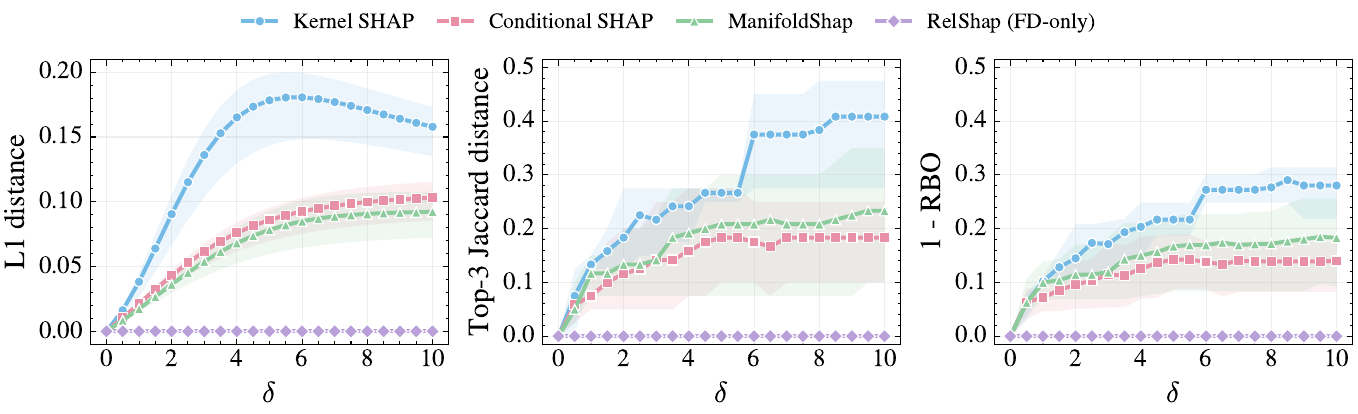}
  \caption{Attribution and ranking changes under the FD
  $\texttt{age}\rightarrow\texttt{life\_stage}$ as the strength $\delta$
  of relationally invalid perturbations increases.}
  \label{fig:synthetic_fd_only}
\end{figure*}

\begin{figure*}[ht]
  \centering
  \includegraphics[width=\linewidth]{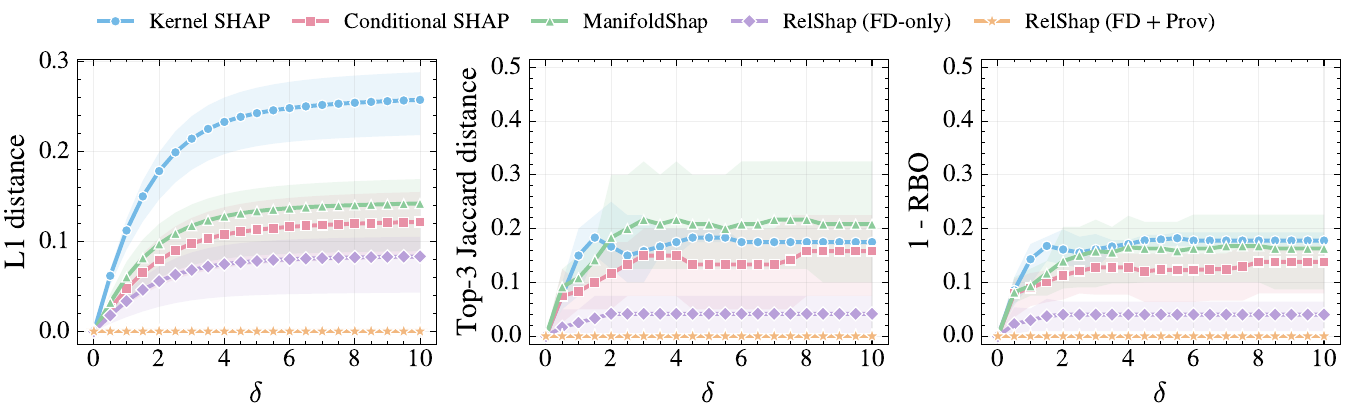}
  \caption{Attribution and ranking changes when incorporating both the FD
  $\texttt{age}\rightarrow\texttt{life\_stage}$ and the provenance
  constraints from Example~\ref{ex:motivating}.}
  \label{fig:synthetic_fd_prov}
\end{figure*}

Expanding on Section~\ref{sec:exp:intuition}, we consider two additional settings over four features---\texttt{age}, \texttt{life\_stage}, \texttt{empl}, and \texttt{total\_amt}. The first enforces only the FD $\texttt{age}\rightarrow\texttt{life\_stage}$, whereas the second additionally incorporates the provenance constraints from Example~\ref{ex:motivating}. Unlike Section~\ref{sec:exp:intuition}, attribution ordering between a particular feature pair is no longer expected to hold in isolation when additional features are introduced, as their Shapley values now also depend on such newly included features.
We therefore examine how the overall attribution vector and feature ranking change as $\delta$ increases ($0$ to $10$ in increments of $0.5$, averaged across three random seeds). Figures~\ref{fig:synthetic_fd_only} and~\ref{fig:synthetic_fd_prov} report the $L_1$ distance and ranking changes as the strength of invalid perturbations increases. Under the FD-only setting, \RelShap remains unchanged across all metrics, while all baselines exhibit increasing attribution and ranking shifts. With provenance constraints, \RelShap using both FDs and provenance again remains invariant, followed by FD-only \RelShap as the most stable alternative. These results show that excluding relationally impossible states corrects the internal Shapley computation and makes \RelShap insensitive to relationally invalid perturbations that are never observed during training or testing.

\subsubsection{Validation with real datasets}

\begin{figure*}[t]
  \centering
  \includegraphics[width=\textwidth]{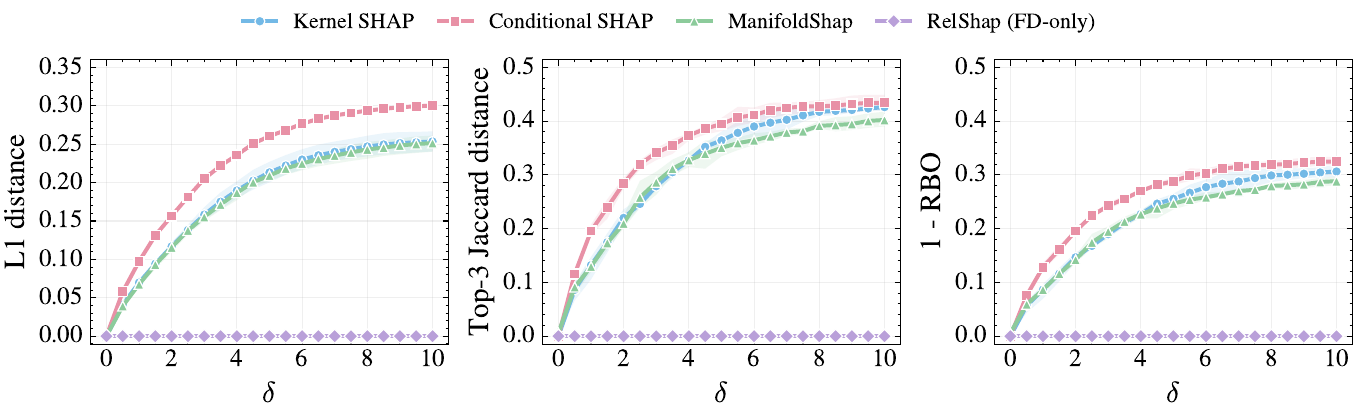}
  \caption{Attribution and ranking changes on German Credit as the influence of $\delta$ increases.}
  \label{fig:invalidity_german}
\end{figure*}

\begin{figure*}[t]
  \centering
  \begin{subfigure}[t]{\textwidth}
    \centering
    \includegraphics[width=\textwidth]{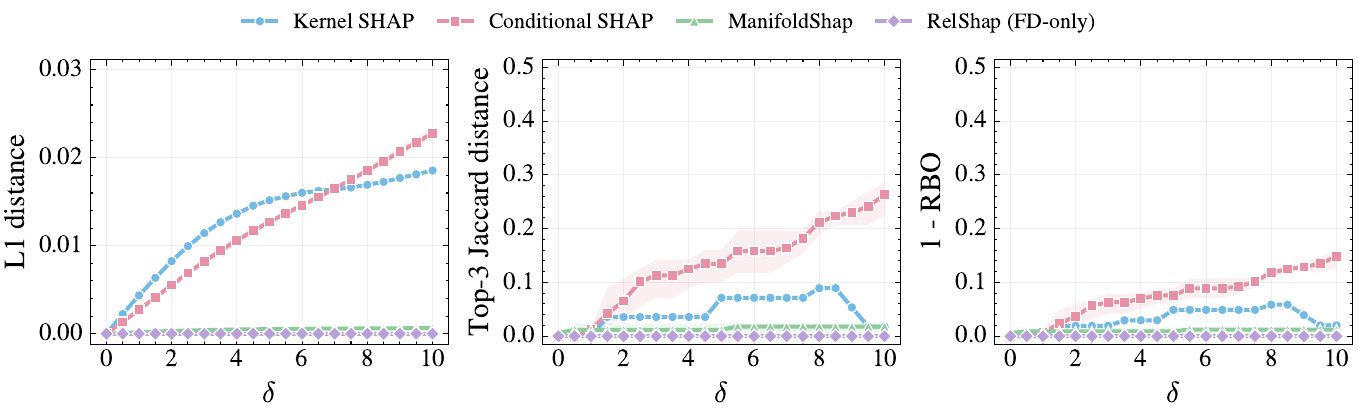}
    \caption{FD-only setting.}
    \label{fig:invalidity_uwcse_fd}
  \end{subfigure}

  \vspace{0.5em}

  \begin{subfigure}[t]{\textwidth}
    \centering
    \includegraphics[width=\textwidth]{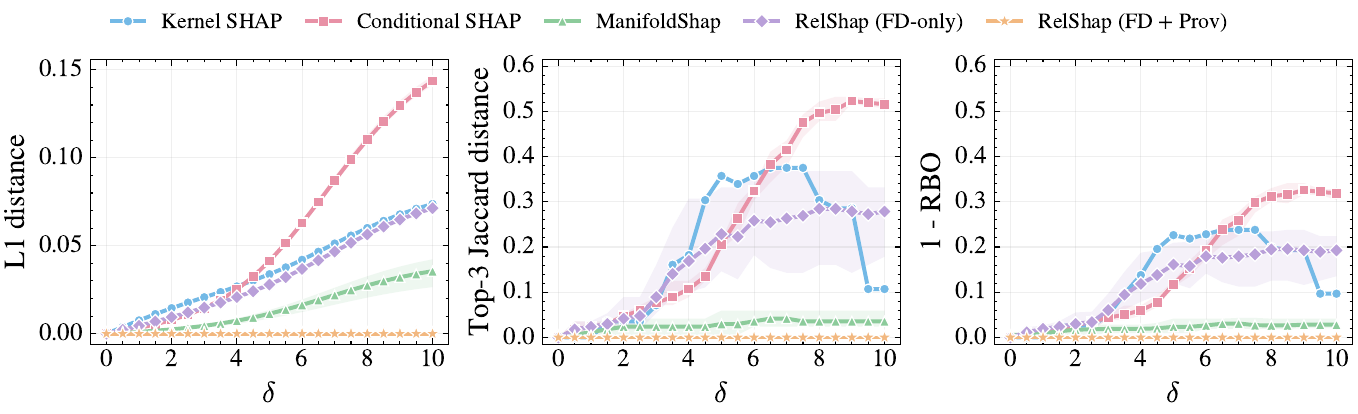}
    \caption{FD~+~Prov setting.}
    \label{fig:invalidity_uwcse_prov}
  \end{subfigure}

  \caption{Attribution and ranking changes on UW-CSE as the influence of $\delta$ increases.}
  \label{fig:invalidity_uwcse}
\end{figure*}

Figures~\ref{fig:invalidity_german} and~\ref{fig:invalidity_uwcse} validate the same behavior on two representative real datasets: German Credit, a standard ML dataset with FDs without provenance information, and UW-CSE, a relational database dataset with both FDs and provenance constraints (Table~\ref{tab:constraints_summary}).
For both datasets, the base predictive model is XGBoost with $0.5 \leq r(\cdot) < 1.0$. On German Credit, all baselines exhibit increasing attribution and ranking changes as $\delta$ grows, whereas FD-only \RelShap remains unchanged. On UW-CSE, in FD-only setting, \RelShap is the only method that remains stable even with increasing $\delta$. In FD~+~Prov setup, only \RelShap incorporating both FDs and provenance constraint remains unchanged. These results show that sensitivity to relationally invalid perturbations is eliminated when the corresponding relational structure is incorporated into Shapley computation.

\subsection{Relational invalidity in Shapley computation} \label{app:violation_details}

\begin{table*}[t]
\centering
\small
\setlength{\tabcolsep}{3pt}
\renewcommand{\arraystretch}{1.15}
\caption{Relational invalidity on datasets without DCs. Each cell reports averaged violation prevalence (top) and violation density (bottom) over three random seeds, in percentage.}
\label{tab:violation_no_dcs}
\begin{tabular}{llccc}
\toprule
Dataset
& Estimator
& BG
& BG~+~Prov, Strict
& BG~+~Prov, Relaxed \\
\midrule

\multirow{3}{*}{Amazon}
& Kernel
& \makecell{$55.956$\\$36.598$}
& \makecell{$77.762$\\$34.501$}
& \makecell{$89.359$\\$34.843$} \\
\cmidrule(lr){2-5}
& MC
& \makecell{$51.350$\\$33.666$}
& \makecell{$70.210$\\$31.938$}
& \makecell{$80.147$\\$32.649$} \\
\cmidrule(lr){2-5}
& Leverage
& \makecell{$55.095$\\$34.508$}
& \makecell{$77.959$\\$33.249$}
& \makecell{$90.317$\\$33.966$} \\
\midrule

\multirow{3}{*}{Churn}
& Kernel
& \makecell{$96.120$\\$23.642$}
& \makecell{$96.122$\\$28.683$}
& \makecell{$96.125$\\$28.627$} \\
\cmidrule(lr){2-5}
& MC
& \makecell{$93.466$\\$28.188$}
& \makecell{$93.467$\\$32.365$}
& \makecell{$93.469$\\$32.320$} \\
\cmidrule(lr){2-5}
& Leverage
& \makecell{$97.721$\\$29.560$}
& \makecell{$97.722$\\$33.377$}
& \makecell{$97.724$\\$33.313$} \\
\midrule

\multirow{3}{*}{Churn Modelling}
& Kernel
& \makecell{$90.240$\\$17.321$}
& \makecell{$90.479$\\$20.537$}
& \makecell{$90.481$\\$20.320$} \\
\cmidrule(lr){2-5}
& MC
& \makecell{$84.992$\\$17.725$}
& \makecell{$85.160$\\$20.963$}
& \makecell{$85.162$\\$20.762$} \\
\cmidrule(lr){2-5}
& Leverage
& \makecell{$95.187$\\$20.313$}
& \makecell{$95.394$\\$23.556$}
& \makecell{$95.395$\\$23.278$} \\
\midrule

\multirow{3}{*}{German Credit}
& Kernel
& \makecell{$18.910$\\$4.395$}
& \makecell{$71.969$\\$16.562$}
& \makecell{$77.132$\\$16.284$} \\
\cmidrule(lr){2-5}
& MC
& \makecell{$20.469$\\$4.787$}
& \makecell{$75.195$\\$18.096$}
& \makecell{$79.719$\\$17.827$} \\
\cmidrule(lr){2-5}
& Leverage
& \makecell{$21.617$\\$4.967$}
& \makecell{$78.244$\\$18.567$}
& \makecell{$83.176$\\$18.266$} \\
\midrule

\multirow{3}{*}{MovieLens 20M}
& Kernel
& \makecell{$87.748$\\$47.969$}
& \makecell{$87.748$\\$47.969$}
& \makecell{$87.748$\\$47.969$} \\
\cmidrule(lr){2-5}
& MC
& \makecell{$71.621$\\$39.704$}
& \makecell{$71.621$\\$39.704$}
& \makecell{$71.621$\\$39.704$} \\
\cmidrule(lr){2-5}
& Leverage
& \makecell{$99.081$\\$47.589$}
& \makecell{$99.081$\\$47.589$}
& \makecell{$99.081$\\$47.589$} \\

\bottomrule
\end{tabular}
\end{table*}

\begin{table*}[t]
\centering
\small
\setlength{\tabcolsep}{3pt}
\renewcommand{\arraystretch}{1.15}
\caption{Relational invalidity on datasets with DCs. Each cell reports averaged violation prevalence (top) and violation density (bottom) over three random seeds, in percentage. All denotes BG~+~DCs~+~Prov, Relaxed.}
\label{tab:violation_with_dcs}
\begin{tabular}{llccccc}
\toprule
Dataset
& Estimator
& BG
& BG~+~DCs
& BG~+~Prov, Strict
& BG~+~Prov, Relaxed
& All \\
\midrule

\multirow{3}{*}{SpeedDating}
& Kernel
& \makecell{$88.509$\\$12.219$}
& \makecell{$88.509$\\$3.164$}
& \makecell{$91.294$\\$17.524$}
& \makecell{$91.695$\\$17.611$}
& \makecell{$91.695$\\$10.530$} \\
\cmidrule(lr){2-7}
& MC
& \makecell{$97.889$\\$20.037$}
& \makecell{$97.889$\\$5.188$}
& \makecell{$98.270$\\$23.743$}
& \makecell{$98.336$\\$23.764$}
& \makecell{$98.336$\\$14.976$} \\
\cmidrule(lr){2-7}
& Leverage
& \makecell{$97.889$\\$20.037$}
& \makecell{$97.889$\\$5.188$}
& \makecell{$98.270$\\$23.743$}
& \makecell{$98.336$\\$23.764$}
& \makecell{$98.336$\\$14.976$} \\
\midrule

\multirow{3}{*}{TPC-H}
& Kernel
& \makecell{$98.704$\\$16.320$}
& \makecell{$98.704$\\$16.183$}
& \makecell{$98.705$\\$17.664$}
& \makecell{$98.706$\\$17.669$}
& \makecell{$98.706$\\$17.523$} \\
\cmidrule(lr){2-7}
& MC
& \makecell{$94.816$\\$19.906$}
& \makecell{$94.816$\\$19.742$}
& \makecell{$94.816$\\$21.100$}
& \makecell{$94.816$\\$21.102$}
& \makecell{$94.816$\\$20.932$} \\
\cmidrule(lr){2-7}
& Leverage
& \makecell{$99.999$\\$19.332$}
& \makecell{$99.999$\\$19.170$}
& \makecell{$99.999$\\$20.682$}
& \makecell{$100.000$\\$20.684$}
& \makecell{$100.000$\\$20.515$} \\
\midrule

\multirow{3}{*}{UW-CSE}
& Kernel
& \makecell{$5.806$\\$1.632$}
& \makecell{$6.675$\\$1.585$}
& \makecell{$7.243$\\$1.829$}
& \makecell{$10.244$\\$1.770$}
& \makecell{$10.819$\\$1.730$} \\
\cmidrule(lr){2-7}
& MC
& \makecell{$4.654$\\$1.287$}
& \makecell{$5.601$\\$1.267$}
& \makecell{$5.824$\\$1.435$}
& \makecell{$8.293$\\$1.434$}
& \makecell{$9.038$\\$1.413$} \\
\cmidrule(lr){2-7}
& Leverage
& \makecell{$5.445$\\$1.535$}
& \makecell{$6.316$\\$1.494$}
& \makecell{$6.844$\\$1.722$}
& \makecell{$9.955$\\$1.684$}
& \makecell{$10.553$\\$1.648$} \\
\midrule

\multirow{3}{*}{Synthetic}
& Kernel
& \makecell{$80.475$\\$14.203$}
& \makecell{$80.475$\\$12.374$}
& \makecell{$91.892$\\$20.661$}
& \makecell{$91.919$\\$20.466$}
& \makecell{$91.919$\\$18.487$} \\
\cmidrule(lr){2-7}
& MC
& \makecell{$75.607$\\$14.660$}
& \makecell{$75.607$\\$12.772$}
& \makecell{$86.302$\\$21.381$}
& \makecell{$86.322$\\$21.189$}
& \makecell{$86.322$\\$19.154$} \\
\cmidrule(lr){2-7}
& Leverage
& \makecell{$77.742$\\$15.004$}
& \makecell{$77.742$\\$12.982$}
& \makecell{$94.370$\\$22.649$}
& \makecell{$94.390$\\$22.362$}
& \makecell{$94.390$\\$20.167$} \\

\bottomrule
\end{tabular}
\end{table*}

\begin{table}[t]
\centering
\small
\setlength{\tabcolsep}{3pt}
\renewcommand{\arraystretch}{1.15}
\caption{Estimator-level relational invalidity averaged across datasets without DCs. Each cell reports averaged violation prevalence (top) and violation density (bottom), in percentage.}
\label{tab:violation_estimator_no_dcs}
\begin{tabular}{lccc}
\toprule
Estimator
& BG
& BG~+~Prov, Strict
& BG~+~Prov, Relaxed \\
\midrule
Kernel
& \makecell{$69.795$\\$25.985$}
& \makecell{$84.816$\\$29.650$}
& \makecell{$88.169$\\$29.609$} \\
\midrule
MC
& \makecell{$64.379$\\$24.814$}
& \makecell{$79.131$\\$28.613$}
& \makecell{$82.024$\\$28.652$} \\
\midrule
Leverage
& \makecell{$73.740$\\$27.388$}
& \makecell{$89.680$\\$31.268$}
& \makecell{$93.139$\\$31.282$} \\
\bottomrule
\end{tabular}
\end{table}

\begin{table*}[t]
\centering
\small
\setlength{\tabcolsep}{3pt}
\renewcommand{\arraystretch}{1.15}
\caption{Estimator-level relational invalidity averaged across datasets with DCs. Each cell reports averaged violation prevalence (top) and violation density (bottom), in percentage. All denotes BG~+~DCs~+~Prov, Relaxed.}
\label{tab:violation_estimator_with_dcs}
\begin{tabular}{lccccc}
\toprule
Estimator
& BG
& BG~+~DCs
& BG~+~Prov, Strict
& BG~+~Prov, Relaxed
& All \\
\midrule
Kernel
& \makecell{$68.373$\\$11.094$}
& \makecell{$68.591$\\$8.327$}
& \makecell{$72.284$\\$14.420$}
& \makecell{$73.141$\\$14.379$}
& \makecell{$73.285$\\$12.068$} \\
\midrule
MC
& \makecell{$68.242$\\$13.972$}
& \makecell{$68.478$\\$9.742$}
& \makecell{$71.303$\\$16.915$}
& \makecell{$71.942$\\$16.872$}
& \makecell{$72.128$\\$14.119$} \\
\midrule
Leverage
& \makecell{$70.269$\\$13.977$}
& \makecell{$70.486$\\$9.709$}
& \makecell{$74.871$\\$17.199$}
& \makecell{$75.670$\\$17.124$}
& \makecell{$75.820$\\$14.326$} \\
\bottomrule
\end{tabular}
\end{table*}

\emph{Violation density} reports the total number of violations normalized by the number of coalition--constraint pairs. Tables~\ref{tab:violation_no_dcs} and~\ref{tab:violation_with_dcs} show that relational invalidity is widespread across datasets and estimators, with provenance constraints substantially increasing violation prevalence while affecting density more moderately; the lower normalized density under BG~+~DCs mainly reflects the larger number of coalition--constraint pairs rather than fewer violations. Leverage SHAP generally yields the highest prevalence and density, whereas MC tends to be the lowest. Averaged across datasets, Tables~\ref{tab:violation_estimator_no_dcs} and~\ref{tab:violation_estimator_with_dcs} confirm that these estimator-level differences are smaller than the increases induced by richer constraints.

\section{Shapley Values in Databases}
\label{app:db_shapley}

Shapley values have been studied in the data management literature to quantify the contribution of \emph{tuples to query answers}~\cite{livshits2021shapley, deutch2022computing, davidson2022shapgraph, alizad2025relation}, and, in databases that violate integrity constraints, to the \emph{extent of that inconsistency}~\cite{bertossi2023shapley}. In this setting, the players of the cooperative game are tuples, the value function is the query output on a subset of the database, and an
endogenous/exogenous partition controls which tuples are subject to attribution. Work in this line has established the computational complexity of exact Shapley computation over conjunctive queries (CQs)~\cite{livshits2021shapley}, practical algorithms based on data provenance and knowledge compilation~\cite{deutch2022computing,davidson2022shapgraph}, and relation-aware sampling for efficient estimation~\cite{alizad2025relation}. Where integrity constraints appear, they are themselves the \emph{object of explanation} (\ie tuples are scored by how much they contribute to inconsistency), rather than constraints on the coalitional game. Our work is orthogonal: we explain the predictions of \emph{black-box ML models} trained on relational data, with \emph{features} as players, and use relational structure as constraints on which coalitions and completions count as valid.

A growing line of database research adapts the Shapley value to attribute the contribution of \emph{tuples} (or facts) to the \emph{results of queries}, and, in a related sub-line, to the \emph{extent of a database's inconsistency} with respect to integrity constraints. \citet{livshits2021shapley} initiated the formal study:
they model query evaluation as a cooperative game in which endogenous tuples act as players and the query's truth value (for Boolean CQs) or numerical output (for aggregate queries) acts as the wealth function. They establish a fundamental complexity dichotomy---Shapley computation is in polynomial time for hierarchical CQs and $\mathrm{FP}^{\#\mathrm{P}}$-complete otherwise---and give an a fully polynomial-time randomized approximation scheme (FPRAS) for summation queries via Monte Carlo sampling.

Subsequent work has focused on making these semantics practical. \citet{deutch2022computing} reduce Shapley computation to probabilistic query evaluation and develop an exact algorithm based on knowledge compilation of Boolean provenance into deterministic and decomposable (d-D) circuits, together with a faster CNF-Proxy heuristic for ranking; their experiments on TPC-H and IMDB demonstrate that provenance-based Shapley computation is feasible for a large fraction of realistic queries. \citet{davidson2022shapgraph} build ShapGraph that pairs the global provenance-graph view of query derivations with local Shapley heatmaps over input tuples, letting analysts navigate between a \emph{structural} (how) and a
\emph{quantitative} (how much) account of a given output tuple. \citet{alizad2025relation} introduces Relation-Stratified Sampling, a sampling strategy that partitions the Shapley subset space by the number of tuples drawn from each relation and allocates samples adaptively, yielding lower-variance estimates than generic Monte Carlo
and classical stratified sampling on TPC-H workloads. \citet{bertossi2023shapley} study both the query-answer and the database-inconsistency applications; in the latter, integrity constraints (ICs) are themselves the object being explained, and tuples are scored by how much their presence contributes to violations of the ICs.

Table~\ref{tab:db_shapley} summarizes these efforts along four axes: what the Shapley game \emph{explains}, who the \emph{players} are, what \emph{technical contribution} is made, and how \emph{integrity constraints} are treated. \RelShap departs from all of them along the first two axes: the object of explanation is a black-box ML model's prediction on a flattened relational instance, and the players are \emph{features} of that instance rather than tuples of the underlying database. This shift is not cosmetic: as Example \ref{ex:motivating} illustrates and Section \ref{sec:experiments} quantifies, applying Shapley to a flattened instance without accounting for the relational structure may produce different explanations.

Consequently, we use relational structure---schema-level FDs, query-induced FDs, discovered FDs, and provenance---as \emph{structural constraints on the Shapley game itself}, rather than
as an object to be explained or as a purely computational device for tractability.

\begin{table*}[t]
\centering
\small
\caption{Shapley-value approaches in database (DB) research. All prior work uses DB tuples/facts as players and explains a property of the DB (query answer or inconsistency w.r.t.\ integrity constraints). \RelShap instead explains ML model predictions on flattened relational data, with features as players, and uses relational structure as a constraint on the coalitional game.}
\label{tab:db_shapley}
\begin{tabular}{@{}p{2.4cm}p{2.5cm}p{2.2cm}p{4.8cm}p{3.0cm}@{}}
\toprule
\textbf{Work} & \textbf{Object explained} & \textbf{Players} &
\textbf{Key technical contribution} &
\textbf{Role of ICs} \\
\midrule
Livshits et al.~\cite{livshits2021shapley} &
Query answer (Boolean / aggregate CQ) &
Endogenous tuples &
Complexity dichotomy (hierarchical CQs in polynomial time, else
$\mathrm{FP}^{\#\mathrm{P}}$-complete); FPRAS for summation via MC &
Not part of the setup \\
\addlinespace[2pt]
Deutch et al.~\cite{deutch2022computing} &
Query answer &
Endogenous tuples &
Reduction to probabilistic QE; knowledge compilation of provenance
into d-D circuits; CNF-Proxy heuristic &
Not part of the setup \\
\addlinespace[2pt]
ShapGraph~\cite{davidson2022shapgraph} &
Query answer (per output tuple) &
Input tuples &
Interactive system pairing provenance graphs with Shapley heatmaps
(built on ProvSQL) &
Not part of the setup \\
\addlinespace[2pt]
Alizad~\cite{alizad2025relation} &
Query answer &
Endogenous tuples &
Relation-Stratified Sampling (RSS) and adaptive variant (ARSS) with
lower-variance estimates on TPC-H &
Not part of the setup \\
\addlinespace[2pt]
Bertossi et al.~\cite{bertossi2023shapley} &
Query answer \emph{and} DB inconsistency &
Endogenous tuples &
Survey; covers tuple contribution to inconsistency w.r.t.\ ICs as a
second Shapley application in data management &
\textbf{Object of explanation} (tuples scored by contribution to
IC violations) \\
\midrule
\textbf{\RelShap{} (ours)} &
\textbf{ML-model prediction on a flattened relational instance} &
\textbf{Features} &
\textbf{Relationally consistent background and coalition selection,
grounded in schema/query/discovered FDs and provenance; Shapley
invariance and runtime-reduction results} &
\textbf{Structural constraint on the coalitional game} \\
\bottomrule
\end{tabular}
\end{table*}

\section{Limitations}
\label{app:limitations}

\textbf{Human evaluation.} Our controlled validation (Section~4.1) establishes that \RelShap corrects attributions with respect to relational validity under known ground truth, and Section~4.2 shows these corrections are large and systematic on real data. However, we do not evaluate whether the corrected explanations are more useful, trustworthy, or actionable for human decision-makers. A user study with practitioners who consume Shapley explanations (\eg in credit or hiring contexts) is important future work.

\textbf{Constraint quality.} \RelShap treats the selected constraint set $\Sigma$ as correct. Data-driven FDs are exact dependencies on finite data, so small or unrepresentative samples can yield spurious FDs; conversely, meaningful dependencies that hold only approximately (\eg with rare exceptions or noise) are not discovered by exact FD mining. The user inspection step (Section~3.1) mitigates but does not eliminate this risk, and extending \RelShap to approximate or probabilistic constraints is a natural direction.

\end{document}